\documentclass[lettersize,journal]{IEEEtran}
\usepackage{array}
\usepackage[caption=false,font=normalsize,labelfont=sf,textfont=sf]{subfig}
\usepackage{textcomp}
\usepackage{stfloats}
\usepackage{url}
\usepackage{verbatim}
\usepackage{graphicx}
\usepackage{cite}
\usepackage{algorithmicx}
\usepackage{algorithm}
\usepackage{algpseudocode}
\usepackage{bm}
\usepackage{multirow}
\usepackage{amsmath,amssymb,amsfonts}
\usepackage{amsthm}
\newtheorem{theorem}{\bf \emph{Theorem}}
\newtheorem{lemma}{\bf \emph{Lemma}}
\usepackage{color}
\usepackage{textgreek}
\usepackage{hyperref}

\makeatletter
\newenvironment{breakablealgorithm}
  {% \begin{breakablealgorithm}
   \begin{center}
     \refstepcounter{algorithm}% New algorithm
     \hrule height.8pt depth0pt \kern2pt% \@fs@pre for \@fs@ruled 画线
     \renewcommand{\caption}[2][\relax]{% Make a new \caption
       {\raggedright\textbf{\ALG@name~\thealgorithm} ##2\par}%
       \ifx\relax##1\relax % #1 is \relax
         \addcontentsline{loa}{algorithm}{\protect\numberline{\thealgorithm}##2}%
       \else % #1 is not \relax
         \addcontentsline{loa}{algorithm}{\protect\numberline{\thealgorithm}##1}%
       \fi
       \kern2pt\hrule\kern2pt
     }
  }{% \end{breakablealgorithm}
     \kern2pt\hrule\relax% \@fs@post for \@fs@ruled 画线
   \end{center}
  }
\makeatother

\begin{document}

\title{Detecting outliers by clustering algorithms}

\author{Qi~Li,~
        Shuliang~Wang
\IEEEcompsocitemizethanks{\IEEEcompsocthanksitem Q. Li is with School of Information Science and Technology, Beijing Forestry University, Beijing, 100083, China.\protect\\
E-mail: liqi2024@bjfu.edu.cn
\IEEEcompsocthanksitem  S. Wang is with School of Computer Science \& Technology, Beijing Institute of Technology, Beijing, 100081, China.\protect\\}
}

\maketitle

\begin{abstract}
Clustering and outlier detection are two important tasks in data mining. Outliers frequently interfere with clustering algorithms to determine the similarity between objects, resulting in unreliable clustering results. Currently, only a few clustering algorithms (\emph{e.g.}, DBSCAN) have the ability to detect outliers to eliminate interference. For other clustering algorithms, it is tedious to introduce another outlier detection task to eliminate outliers before each clustering process. Obviously, how to equip more clustering algorithms with outlier detection ability is very meaningful. Although a common strategy allows clustering algorithms to detect outliers based on the distance between objects and clusters, it is contradictory to improving the performance of clustering algorithms on the datasets with outliers. In this paper, we propose a novel outlier detection approach, called ODAR, for clustering. ODAR maps outliers and normal objects into two separated clusters by feature transformation. As a result, any clustering algorithm can detect outliers by identifying clusters. Experiments show that ODAR is robust to diverse datasets. Compared with baseline methods, the clustering algorithms achieve the best on 7 out of 10 datasets with the help of ODAR, with at least 5\% improvement in accuracy. 
\end{abstract}


\section{Introduction}
\label{sec:introduction}
Outlier detection aims to discover objects that do not conform to the global pattern of the dataset. It is commonly used to identify credit card fraud, analyze health condition, detect device failure, maintain network security, and explore new stars in astronomical images \cite{panjei2022survey}. Clustering aims to divide objects into different categories solely based on similarity. It is widely used in many fields, such as pattern recognition, data compression, image segmentation, time series analysis, information retrieval, spatial data analysis, biomedical research, and so on \cite{chao2021survey}.

Clustering and outlier detection are two diametrically opposed tasks because clustering focuses on objects that are dense in the feature space, while outliers are usually sparsely distributed in the feature space. However, it seems that clustering and outlier detection are again two inseparable tasks because outliers can greatly affect the clustering algorithm's ability. As shown in Figure \ref{fig:example-clustering}, there are several outliers between the two clusters. Since similarity can be propagated by neighbors \cite{wang2018delta}, these outliers will act as similarity propagation chains to induce clustering algorithms to identify the two clusters as one category. 

\begin{figure}[h]
  \centering
  \includegraphics[width=3in]{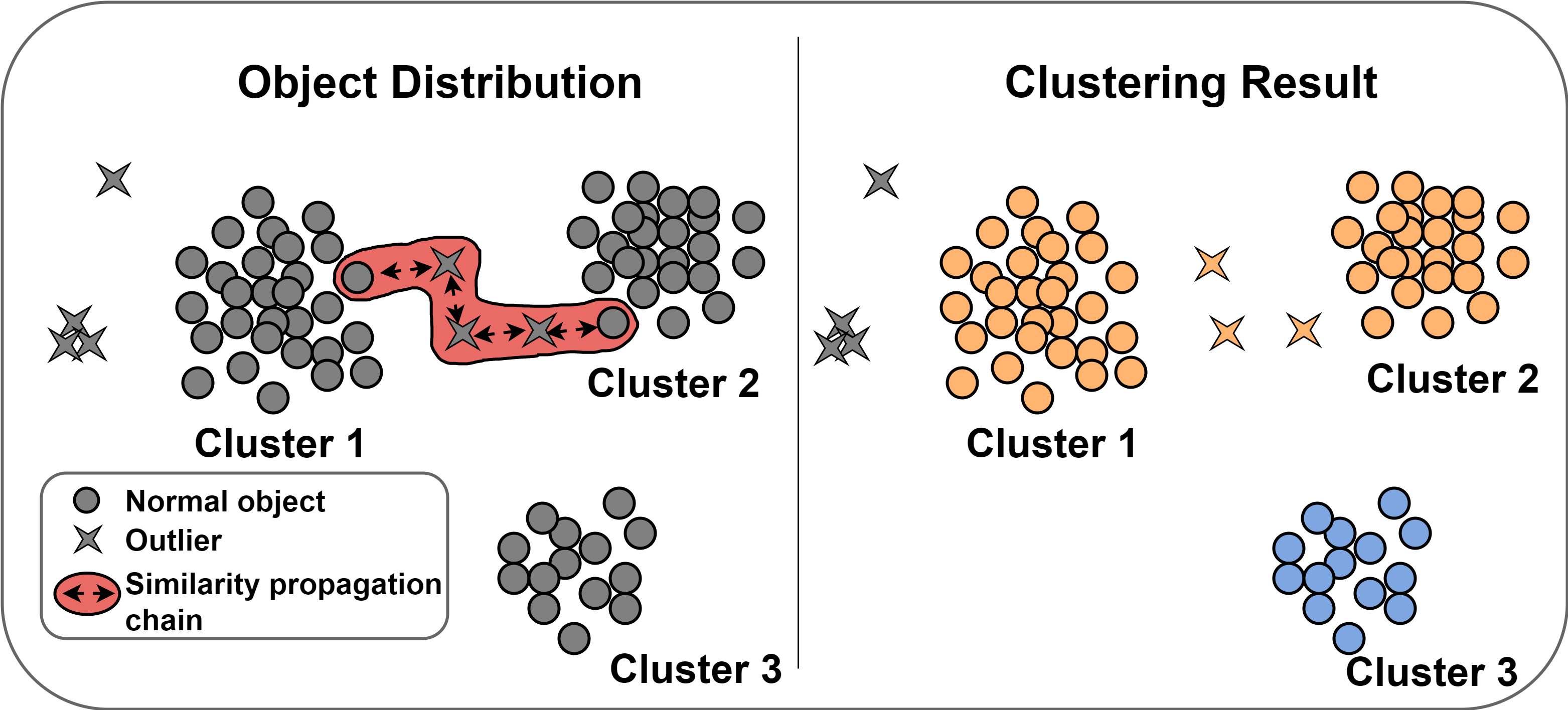}
  \caption{An example of outlier misleading clustering. The outliers that are between clusters act as similarity propagation chains, increasing the similarity between clusters and thus misleading clustering algorithms to identify different clusters as one category.}
  \label{fig:example-clustering}
\end{figure}

In order to exclude outlier interference, some clustering algorithms are equipped with outlier detection approaches, but these approaches are not universal and only match specific clustering principles. For example, DBSCAN \cite{ester1996density} uses non-density reachable relations to detect outliers, provided that all density reachable relations have been determined by its clustering principle; DPC \cite{rodriguez2014clustering} strips outliers from the edge of the cluster, provided that the cluster and surrounding sparse objects have been identified as a class by its clustering principle. So far, a large number of existing clustering algorithms are not immune to outliers. Currently, a common strategy is to calculate the distance from the object to clusters (or clustering centers), the farther the distance, the more likely the object is to be an outlier \cite{panjei2022survey}. However, this strategy is contradictory to improving the performance of clustering algorithms on datasets with outliers, because it is premised on the need to accurately identify clusters, but outliers instead interfere with the identification of clusters. Therefore, it is a valuable study how to establish an  association between outlier detection task and clustering task so that \textbf{\emph{any}} existing clustering algorithm gains effective outlier detection ability.

In this paper, we propose a universal and effective outlier detection approach called ODAR (\underline{O}utlier \underline{D}etection \underline{A}pproach via cluster \underline{R}ecognition) that meets the above requirements. We summarize the main contributions of this work as follows:

\textbf{1) We propose a daring idea of turning the outlier detection process into a clustering process, so that \emph{any} existing clustering algorithm
gains outlier detection ability.} Specifically, ODAR maps a dataset of arbitrary dimensionality into a two-dimensional feature space (we call it ODAR space), where normal objects and outliers are divided into two clusters. In this way, \emph{\textbf{any}} clustering algorithm can detect outliers by identifying the clusters in ODAR space with its own cluster identification ability. 

\begin{figure}[h]
  \centering
  \includegraphics[width=3in]{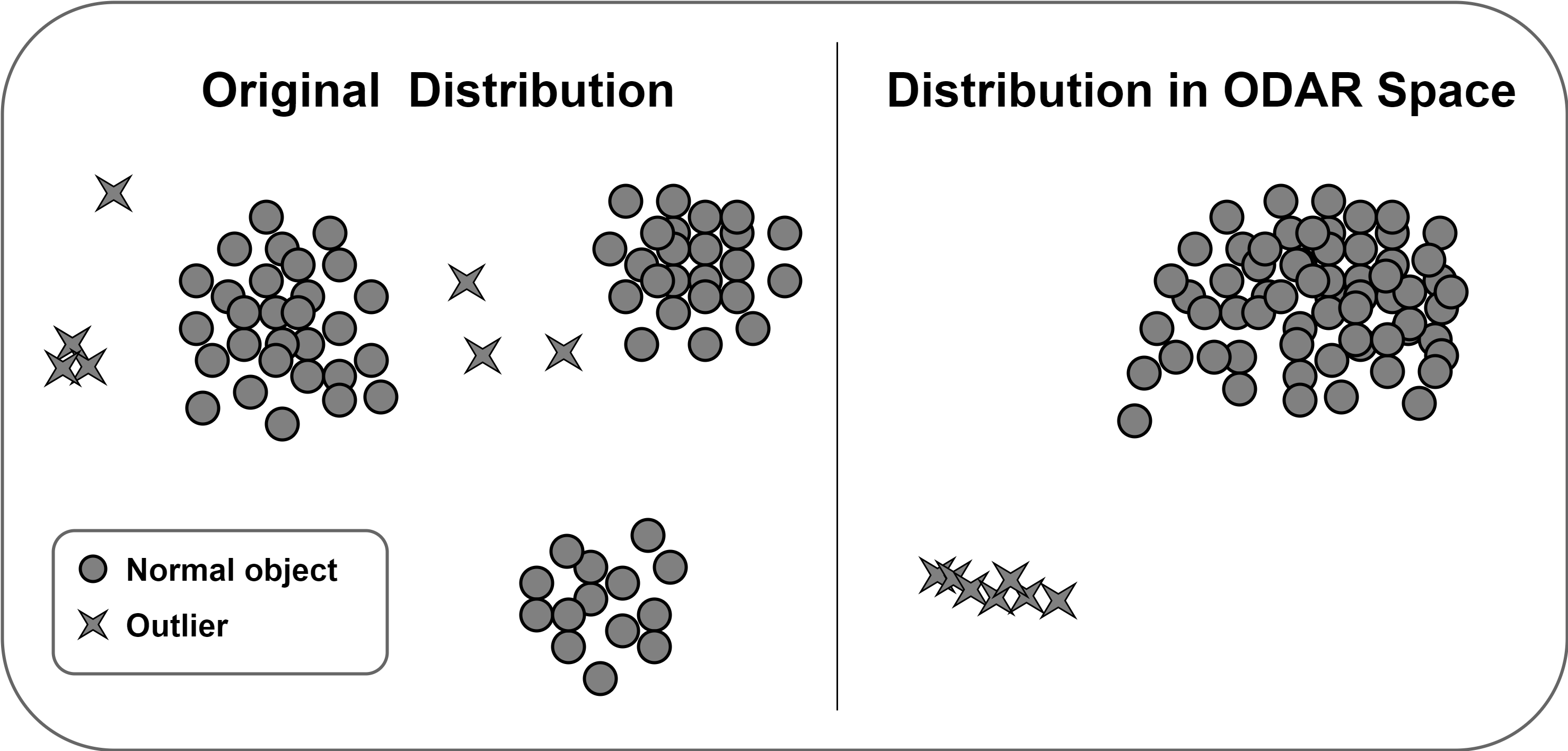}
  \caption{ODAR maps the dataset into the ODAR space, where high-order density can ensure normal objects and outliers are divided into clearly separable clusters.}
  \label{fig:example-ODAR}
\end{figure}

\textbf{2) We propose a novel concept of high-order density, \emph{i.e.} the density of local density values, to ensure the \emph{effectiveness} of ODAR.} Obviously, in ODAR space, only if the cluster in which outliers are located and the cluster in which normal objects are located are clearly separable, then ODAR is effective. Traditional outlier detection approaches \cite{rodriguez2014clustering, ester1996density, wang2020extreme, wang2018delta} only focus on the local density of objects. However, it is difficult to find a clear gap in local density values between outliers and normal objects because there are always outliers close to each other whose local densities are close to those of normal objects, as shown in Figure \ref{fig:example-clustering}. Fortunately, regardless of the proximity of outliers to each other, the distribution of their local density values must be sparse because their number is much less than that of normal objects. As a result, the high-order density values of outliers are not only small but also significantly smaller than those of normal objects. That is, the high-order density widens the difference between outliers and normal objects but narrows the difference between outliers. Therefore, in addition to the local density, we use the high-order density as another feature of ODAR space. The proposed high-order density ensures that outliers and normal objects can be mapped into two \emph{clearly separable} clusters in ODAR space, as shown in Figure \ref{fig:example-ODAR}, thus ensuring the effectiveness of ODAR.

\textbf{3) We design a component clustering strategy to ensure the \emph{universality} of ODAR.} In some specific datasets, there are some normal objects with extremely large local densities (hereafter referred to as large-density objects). Due to the small number of large-density objects, their high-order densities may be as small as the high-order densities of outliers, leading clustering algorithms with weak cluster identification abilities to treat outliers and large-density objects in a cluster in ODAR space. The component clustering strategy does not directly cluster the dataset but first clusters the local density dimension to exclude large-density objects, and then second-clusters the previous result through the high-order density dimension to exclude common normal objects. As a result, clustering algorithms with weak cluster identification abilities can accurately identify the outlier cluster in ODAR space. Therefore, the component clustering strategy improves the universality of ODAR, so that clustering algorithms with different cluster identification abilities accurately detect outliers.

\textbf{4) Numerous experiments confirm the advantages of ODAR.} These experiments show that: 1) ODAR has strong robustness, and it is not limited by the distribution and number of outliers, and unbalanced density. 2) Regardless of clustering algorithms, ODAR enables them to obtain excellent outlier detection with an average accuracy of 0.84 (1 is completely accurate). 3) Compared with 17 baseline methods, the clustering algorithms with the help of ODAR perform best on 7 out of 10 datasets. In terms of average accuracy, they outperform baseline methods by at least 5\%. 4) ODAR is not sensitive to input parameter. Under different parameter values, the clustering algorithms with the help of ODAR are consistently superior to baseline methods. 5) ODAR has a shorter runtime than most baseline methods.

The remainder of the paper is organized as follows. The next section is about related works. Section \ref{sec:ODAR} introduces the theory of ODAR. Section \ref{sec:experiments} verifies the robustness and effectiveness of ODAR. The conclusion and future works are shown in Section \ref{sec:conlusion}.

\section{Related works}
\textbf{Outlier-clustering algorithms (Clustering algorithms that can directly detect outliers):} \emph{DBSCAN} is the most famous clustering algorithm with the ability of outlier detection \cite{wang2020extreme}. Based on \emph{epsilon} and \emph{minpts}, it selects all with density reachable relation and then treats the remaining sparse objects, those without the relation, as outliers. Outlier-clustering algorithms usually have two approaches to detect outliers. One approach \cite{abbas2021denmune, rodriguez2014clustering, d2021automatic} peels off outliers from identified clusters. Another approach \cite{wang2020extreme, wang2018delta, he2003discovering} assumes each object should belong to a certain cluster, so some abnormal clusters exist, all objects inside of which are outliers. \emph{DPC} \cite{rodriguez2014clustering} belongs to the first approach. It first assigns each object into a subject cluster, then peels off those low-density boundary objects from the subject clusters as outliers. \emph{Delta clustering} \cite{wang2018delta} and \emph{CBLOF} \cite{he2003discovering} belong to the second approach. \emph{Delta clustering} separates the whole dataset into different clusters and counts the number of objects in each cluster. If the number in certain clusters are far less than others, then those clusters are set to be abnormal clusters. All objects in abnormal clusters are treated as outliers. After clustering, \emph{CBLOF} classifies identified clusters into \emph{large clusters} and \emph{small clusters}. Based on where an object comes from and the clusters' size, \emph{CBLOF} can identify whether the object is an outlier. If an object is identified as an outlier, then the cluster in which the object is found is an abnormal cluster. \emph{Compared with the above algorithms, whose outlier detection approach can only match specific clustering principles, we propose a universal approach for clustering principles.}

\textbf{General outlier detection algorithms:} General outlier detection algorithms have a huge family, and can be roughly subdivided into neighbor-based \cite{mohotti2020efficient, wang2023self, breunig2000lof}, distance-based \cite{zhu2022high, tran2020real,almardeny2022novel,huang2023novel}, probabilistic-based \cite{2020COPOD, li2022ecod, zheng2021probabilistic}, network-based \cite{xu2023deep, goodge2022lunar, liu2020generative, sarvari2021unsupervised}, and ensemble-based \cite{zhao2021suod, zhao2019lscp} algorithms. \emph{LOF} \cite{breunig2000lof}, \emph{Isforest} \cite{liu2012isolation}, \emph{LSCP} \cite{zhao2019lscp} are the most representative and widely used methods. \emph{LOF} determines whether objects are outliers by calculating the similarity between neighbors. \emph{Isforest} is a tree classifier, and it detects outliers based on the locations of objects in the tree. \emph{LSCP} is a parallel outlier detection ensemble which selects competent detectors in the local region of a test instance. \emph{SO-GAAL} \cite{liu2020generative}, \emph{MO-GAAL} \cite{liu2020generative}, \emph{BAE} \cite{sarvari2021unsupervised}, \emph{RCA} \cite{liu2021rca}, \emph{DeepSVDD} \cite{ruff2018deep}, \emph{DIF} \cite{xu2023deep} and \emph{LUNAR} \cite{goodge2022lunar} are network-based outlier detection methods that have received a lot of attention recently. \emph{BAE} builds an adaptive cascade of autoencoders to detect outliers. \emph{SO-GAAL} applies a generative adversarial network to generate informative potential outliers to assist the classifier in describing the boundary. \emph{MO-GAAL} integrates multiple generators to generate informative potential outliers. \emph{RCA} trains collaborative auto-encoders to obtain the reconstruction errors for outlier detection. \emph{DeepSVDD} trains a neural network to map potential representations of most objects into a hypersphere, and then detects outliers based on the distance from the object to the hypersphere. \emph{LUNAR} implements local outlier methods (such as \emph{LOF}) in graph neural networks, and it uses information from the nearest neighbours in a trainable way to detect outliers.  \emph{DIF} is an \emph{Isforest} extension algorithm based on deep learning. \emph{Compared with the above outlier detection algorithms, which cannot be used directly for clustering task, we establish an association between clustering task and outlier detection task, so that any clustering algorithm can directly gain outlier detection ability to eliminate outliers before clustering.}

\begin{figure}[h]
  \centering
  \includegraphics[width=3in]{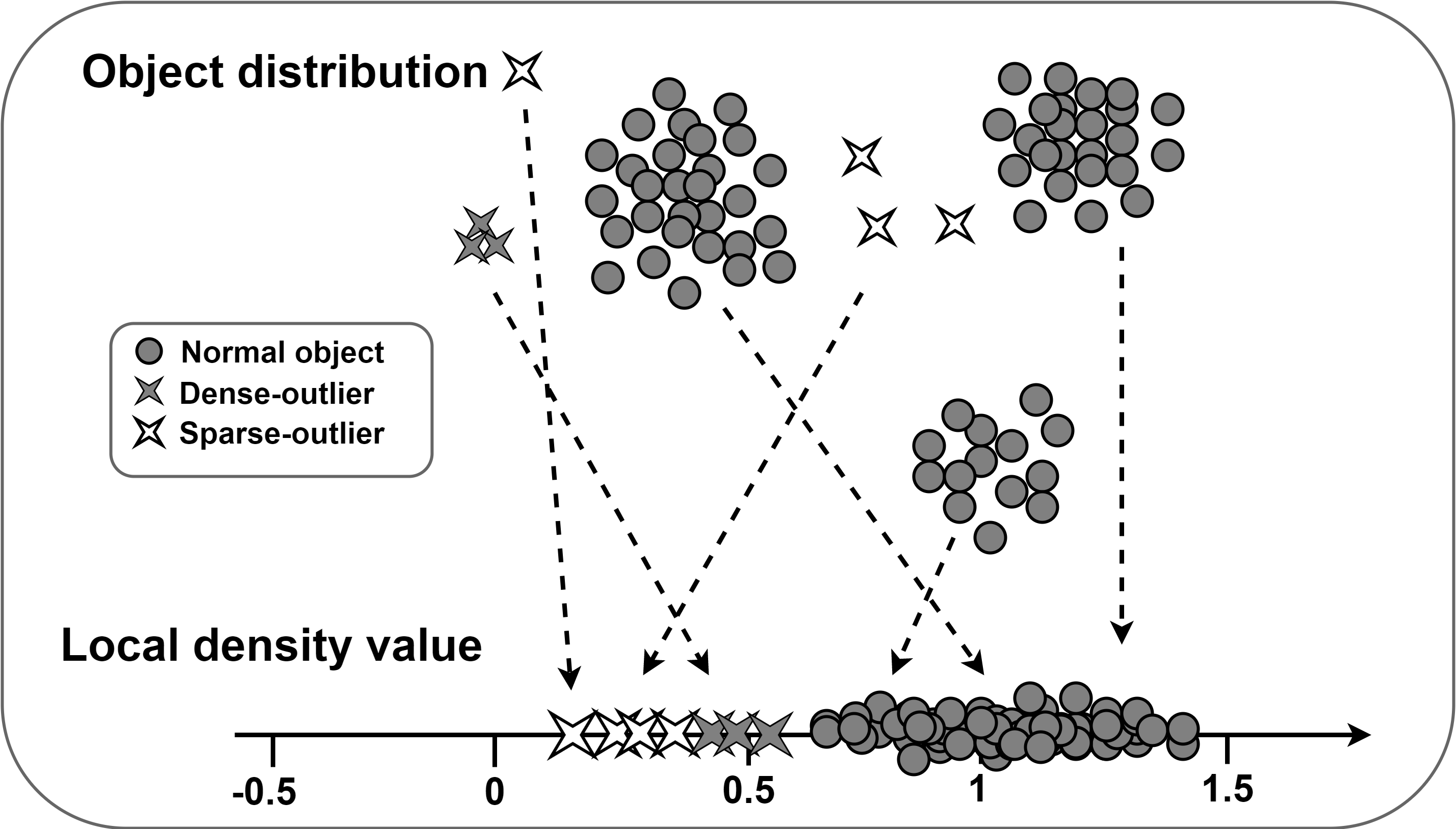}
  \caption{Local density does not clearly distinguish between normal objects and outliers.}
  \label{fig:example-local density}
\end{figure}

\section{Method}
\label{sec:ODAR}

Generally, there are two types of outliers: sparse-outliers and dense-outliers. Sparse-outliers are sparsely distributed objects in the data space, and they are away from other objects, such as the hollow star-shaped objects in Figure \ref{fig:example-local density}. Dense-outliers are several objects that are gathered in one place far from subject objects, such as the solid star-shaped objects in Figure \ref{fig:example-local density}. For example, the heights of pygmy horses are dense-outliers relative to normal horses. Sparse-outliers can be easily detected because they have an obvious characteristic that their local densities are much smaller than those of normal objects. \emph{DBSCAN}, \emph{DPC}, \emph{LOF}, and \emph{KNN} actually make use of this characteristic as the basis for judging outliers. However, such an obvious characteristic does not occur in dense-outliers. Gathering dense-outliers have greater local densities, even close to normal objects. Figure \ref{fig:example-local density} shows the local density values corresponding to the different objects in the dataset of Figure \ref{fig:example-clustering}. It can be observed that the gap between the local density values of sparse-outliers and normal objects is large, but the gap between the local density values of dense-outliers and normal objects is small.

Therefore, it will be significant to minimize the differences between outliers so that they are significantly different from normal objects. This paper consists of two parts: feature transformation and clustering strategy. The feature transformation constructs a feature space that allows sparse-outliers and dense-outliers to be mapped into a cluster away from normal objects. Then, the clustering strategy assists clustering algorithms in detecting outliers by identifying clusters.

\subsection{Feature Transformation}
Let $O$ be a set of $N$ objects in $d$-dimensional space $R^d$. First, the local density value is calculated for each object in $O$. We select $-e^x$ as the core of local density because its wide gradient range can make an obvious difference between normal objects and outliers. Let $o_i$ be the $i$-th object in $O$, its local density is defined as follows,
\begin{equation}
\rho_i=-f(\sum_{j=1}^ke^{g(\|o_{ij}-o_i\|_2)}),
\label{eq1}
\end{equation}
where $k$ is the unique input parameter of ODAR, and $o_{ij}$ is the $j$-nearest object of $o_i$. $g(\cdot)$ and $f(\cdot)$ are used to scale density values to avoid large differences between sparse-outliers and dense-outliers. Specifically, $g(\cdot)$ is a common normalization function, which normalizes the distance between objects rather than the objects' components.
\begin{equation}
g(\|o_{ij}-o_i\|_2)=\frac{\|o_{ij}-o_i\|_2-\min\limits_{s\in [1, N]}(\|o_{sj}-o_s\|_2)}{\max\limits_{s\in [1, N]}(\|o_{sj}-o_s\|_2)-\min\limits_{s\in [1, N]}(\|o_{sj}-o_s\|_2)},
\label{eq2}
\end{equation}
where $N$ is the number of objects in dataset $O$. $f(\cdot)$ is a function to adjust the center of density values' distribution to 1 by dividing the mean of all local densities. Consequently,
\begin{equation}
\rho_i=\frac{-N\sum_{j=1}^k e^{\frac{\|o_{ij}-o_i\|_2-\min\limits_{s\in [1, N]}(\|o_{sj}-o_s\|_2)}{\max\limits_{s\in [1, N]}(\|o_{sj}-o_s\|_2)-\min\limits_{s\in [1, N]}(\|o_{sj}-o_s\|_2)}}}{\sum_{v=1}^N\sum_{j=1}^k e^{\frac{\|o_{vj}-o_v\|_2-\min\limits_{s\in [1, N]}(\|o_{sj}-o_s\|_2)}{\max\limits_{s\in [1, N]}(\|o_{sj}-o_s\|_2)-\min\limits_{s\in [1, N]}(\|o_{sj}-o_s\|_2)}}}.
\label{eq3}
\end{equation}

\textbf{Main Idea.} As we described above, dense-outliers' density values may be close to those of normal objects. However, based on two important facts (\emph{i.e., \textbf{1)} dense-outliers' density values are still smaller than those of normal objects; \textbf{2)} the number of outliers is much less than the number of normal objects}), we can draw the conclusion that dense-outliers' density values are sparse (the detailed proof is given in Theorem \ref{theorem}), which means that the density of the density value (hereinafter referred to as \emph{\textbf{high-order density}}) of a dense-outlier must be small. To prove Theorem \ref{theorem}, we first give Lemma \ref{lemma}.

\begin{lemma}
\label{lemma}
If $a<b<c<d$ and $b-a<d-c$, then $e^b-e^a<e^d-e^c$.
\end{lemma}
\begin{proof}
\begin{align*}
&\because 0<b-a<d-c\\
&\therefore e^{b-a}<e^{d-c}\\
&\therefore\frac{e^b}{e^a}<\frac{e^d}{e^c}\\
&\therefore\frac{e^b-e^a+e^a}{e^a}<\frac{e^d-e^c+e^c}{e^c}\\
&\therefore\frac{e^b-e^a}{e^a}+1<\frac{e^d-e^c}{e^c}+1\\
&\therefore\frac{e^b-e^a}{e^a}<\frac{e^d-e^c}{e^c}\\
&\because e^a<e^c\\
&\therefore\frac{e^b-e^a}{e^a}<\frac{e^d-e^c}{e^a}\\
&\therefore e^b-e^a<e^d-e^c
\end{align*}
\end{proof}

\begin{theorem}
\label{theorem}
The density values of outliers are sparser than the density values of normal objects. 
\end{theorem}
\begin{proof}
The sparsity can be reflected by distance. The larger the distance between adjacent instances, the sparser  instances. Therefore, the theorem can be proved by the distance between adjacent density values. Let $a,\overline{a},\widetilde{a}\in outlier$, $b,\overline{b},\widetilde{b}\in normal$-$object$. $\overline{a}$ is the nearest object of $a$, $\widetilde{a}$ is the nearest object of $\overline{a}$. $\overline{b}$ is the nearest object of $b$, $\widetilde{b}$ is the nearest object of $\overline{b}$. Therefore, $a$ and $\overline{a}$ have many neighbor objects in common. $b$ and $\overline{b}$ also have many neighbor objects in common. Based on the Equation (\ref{eq1}), the density values of $a$ and $\overline{a}$ are adjacent, the density values of $b$ and $\overline{b}$ are also adjacent. Their density is as follows (here, we let $k=1$, because the proof is similar when $k>1$; since $f(\cdot)$ and $g(\cdot)$ are equal proportional reduction methods and do not affect sparsity, $f(\cdot)$ and $g(\cdot)$ are omitted)
\begin{align*}
\rho(a)=-e^{\|\overline{a}-a\|_2}\\
\rho(\overline{a})=-e^{\|\widetilde{a}-\overline{a}\|_2}\\
\rho(b)=-e^{\|\overline{b}-b\|_2}\\
\rho(\overline{b})=-e^{\|\widetilde{b}-\overline{b}\|_2}\\
\end{align*}
Let $dis_{i}(x)$ be the distance between object $x$ and its $i$-th nearest object. Assuming $\|\overline{a}-a\|_2<\|\widetilde{a}-\overline{a}\|_2$, $\|\overline{b}-b\|_2<\|\widetilde{b}-\overline{b}\|_2$, then
\begin{align*}
&|\rho(\overline{a})-\rho(a)|=e^{\|\widetilde{a}-\overline{a}\|_2}-e^{\|\overline{a}-a\|_2}\\
&|\rho(\overline{b})-\rho(b)|=e^{\|\widetilde{b}-\overline{b}\|_2}-e^{\|\overline{b}-b\|_2}\\
&\because \overline{a} \in outlier, \overline{b} \in normal{\rm-}object\\
&\therefore {\rm Neighbors \: of} \: \overline{a} \: {\rm are \: sparser \: than \: those \: of} \: \overline{b}. \\
&\therefore {\rm For} \: i,j, dis_{i}(\overline{b}) - dis_{i-j}(\overline{b}) < dis_{i}(\overline{a}) - dis_{i-j}(\overline{a})\\
&\therefore \|\widetilde{b}-\overline{b}\|_2 - \|\overline{b}-b\|_2 < \|\widetilde{a}-\overline{a}\|_2 - \|\overline{a}-a\|_2\\
&\because {\rm Lemma \: 1}\\
&\therefore e^{\|\widetilde{b}-\overline{b}\|_2}-e^{\|\overline{b}-b\|_2}<e^{\|\widetilde{a}-\overline{a}\|_2}-e^{\|\overline{a}-a\|_2}\\
&\therefore |\rho(\overline{b})-\rho(b)|<|\rho(\overline{a})-\rho(a)|\\
\end{align*}
Therefore, the density values of outliers are sparser than the density values of normal objects.
\end{proof}

We define the high-order density of $o_i$ as follows,
\begin{equation}
h\rho_i=\sum_{j=1}^N\chi(\frac{|\rho_i-\rho_j|}{\sigma})e^{\frac{-(\rho_i-\rho_j)^2}{\sigma^2}},
\label{eq4}
\end{equation}
where $\chi(x)=1$ if $x\leq 1$ and $\chi(x)=0$ otherwise. $\sigma=10\cdot(\max(\rho)-\min(\rho))/N$.  Obviously, the high-order density $h\rho_i$ of $o_i$ is not only affected by the \textbf{number} of $\rho_j$ in the neighborhood of $\rho_i$ with the radius $\sigma$, but also by the \textbf{compactness} between these $\rho_j$ and $\rho_i$. 

\begin{figure}[h]
  \centering
  \includegraphics[width=3in]{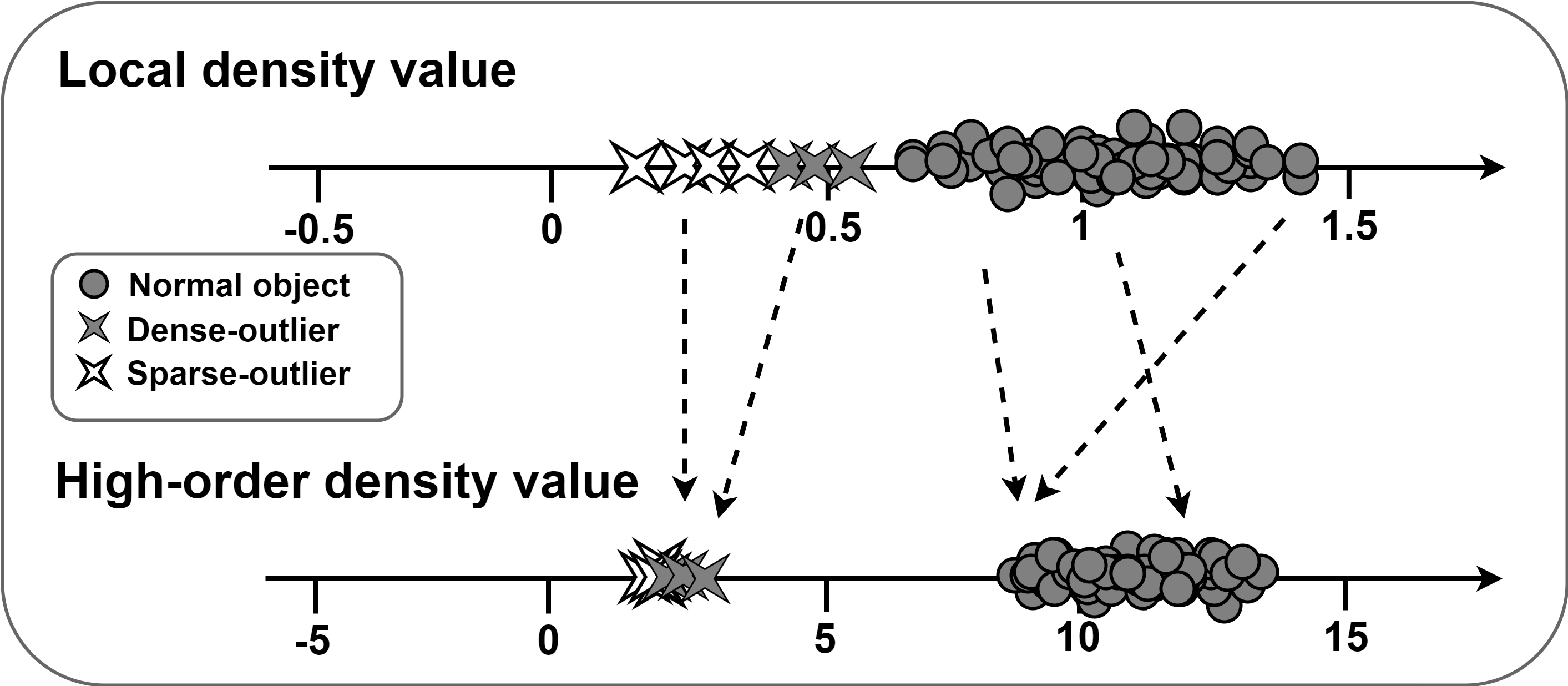}
  \caption{High-order density can clearly distinguish between normal objects and outliers.}
  \label{fig:example-high-order density}
\end{figure}

\textbf{Advantages.} Defined by Equation (\ref{eq4}), high-order density has two advantages. \textbf{The first advantage is the magnification of the difference}. As long as one factor (number or compactness) is small, high-order density will become small, and vice versa. Hence, even if the local densities of dense-outliers are close to those of normal objects, the high-order densities of dense-outliers will be significantly smaller than those of normal objects because the local density values of normal objects are denser according to Theorem \ref{theorem} (\emph{i.e.}, the number and compactness of $\rho_j$ around the local density values of normal objects are greater). \textbf{The second advantage is the formation of a cluster}. Regardless of sparse-outliers or dense-outliers, their high-order densities are small. Relative to normal objects, all outliers can form a cluster. Figure \ref{fig:example-high-order density} shows the high-order density values corresponding to the local density values in Figure \ref{fig:example-local density}. It can be observed that high-order density can clearly distinguish between normal objects and outliers compared to local density.

By combining local density and high-order density, we construct a novel feature space called ODAR space:
\begin{equation}
ODAR=\rho\otimes h\rho.
\label{eq5}
\end{equation}
As shown in Figure \ref{fig:example-ODAR}, in ODAR space, all outliers are gathered in one cluster. We call the cluster \emph{\textbf{the outlier-cluster}}. Normal objects are gathered in another cluster. Therefore, any clustering algorithm can detect outliers by identifying the outlier-cluster in ODAR space. Algorithm \ref{ag1} describes the detailed implementation of constructing ODAR space.

\algnewcommand\algorithmicinput{\textbf{Inputs:}}
\algnewcommand\Inputs{\item[\algorithmicinput]}
\algnewcommand\algorithmicoutput{\textbf{Output:}}
\algnewcommand\Output{\item[\algorithmicoutput]}
%%\lipsum[1]
\begin{breakablealgorithm}
\label{ag1}
%\begin{algorithm}[H]
\caption{\emph{Constructing ODAR space}.}
\algblock[Name]{Begin}{End}
\begin{algorithmic}[1]
\Inputs{ $dataset$, $k$}
\Output{$ODAR$}
\State \textbf{Step 1: Calculating local density.}
\State $dis\_k \leftarrow KDtree(dataset,k)$; \# $dis\_k$ stores the distance between all objects and their nearest $k$ objects.
\State $max \leftarrow max(dis\_k)$; \# Finding the maximum distance.
\State $min \leftarrow min(dis\_k)$; \# Finding the minimum distance.
\For{ $i$ $in$ $range(N)$}  \# Lines 5 to 9 corresponds to the $g(\cdot)$ function.
\For{ $j$ $in$ $range(k)$}
\State $dis\_k[i,j] \leftarrow (dis\_k[i,j]-min)/(max-min)$; \# $dis\_k[i,j]$ stores the distance between object $i$ and object $j$.
\EndFor
\EndFor
\State $dis\_k \leftarrow e^{dis\_k}$;
\State $\rho \leftarrow sum(dis\_k, axis=1)$;
\State $\rho \leftarrow -\rho/mean(\rho)$; \# Corresponding to the $f(\cdot)$ function.
\State \textbf{Step 2: Calculating high-order density.}
\State $sorted_\rho \leftarrow  sort(\rho, axis=0)$; \# Sorting the local density.
\State $index_\rho \leftarrow argsort(\rho, axis=0)$;
\State $\sigma \leftarrow 10\cdot(\max(\rho)-\min(\rho))/N$;
\For{ $i$ $in$ $range(N)$}
\State $j \leftarrow i$;
\State $h\rho[index_\rho[i]] \leftarrow 0$; \# $index_\rho[i]$ stores the subscript of the object with the $i$-th smallest local density.
\While {$sorted_\rho[j]<sorted_\rho[i]$+$\sigma$ $And$ $j<N-1$}
\State $h\rho[index_\rho[i]] \leftarrow h\rho[index_\rho[i]]+exp(-(sorted_\rho[j]-sorted_\rho[i])^2/\sigma^2)$;
\State $j \leftarrow j+1$;
\EndWhile
\EndFor
\State \textbf{Step 3: Constructing ODAR space.}
\State $ODAR$ $\leftarrow [\rho, h\rho]$;
\State \emph{Return ODAR.}
\end{algorithmic}
\end{breakablealgorithm}

\normalsize

Many studies have shown that compacting the clusters before clustering will greatly reduce the difficulty of clustering \cite{li2020hibog, li2023improve}. Here, we design a simple shrinking method, which forces each object to move to the center of nearby objects, to compact the clusters in ODAR space. Such movement is as follows,
\begin{equation}
o_i=o_i+(\frac{1}{\beta}\sum_{j=1}^\beta o_{ij}-o_i),
\label{eq6}
\end{equation}
where $o_i\in$ ODAR is an object, $o_{ij}$ is the $j$-th nearest neighbor of $o_i$. We set $\beta=N/10$. Since $\beta=N/10$ is generally much smaller than the size of cluster, the movement of objects only occurs inside the cluster.

\begin{figure}[h]
  \centering
  \includegraphics[width=3.2in]{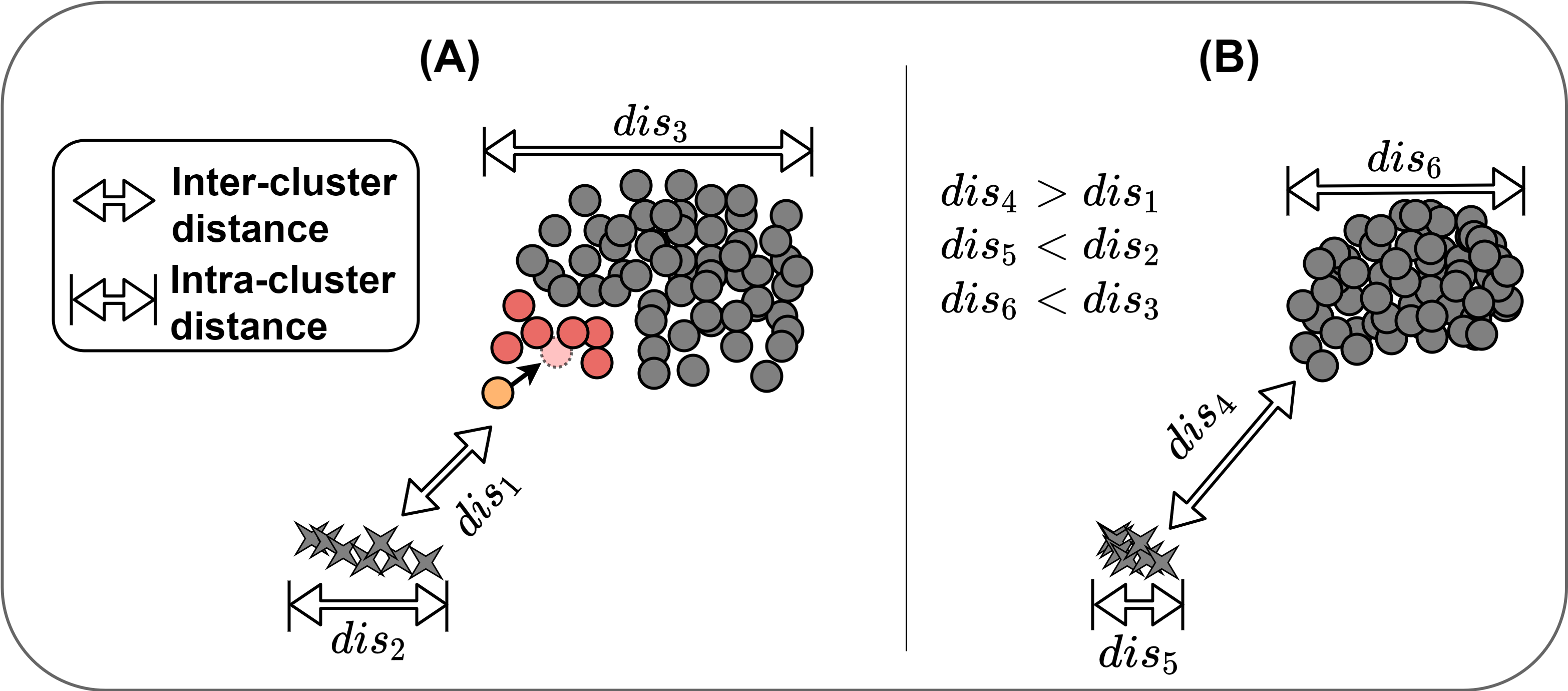}
  \caption{An example of the shrinking method. After shrinking, the clusters in ODAR space become compact.}
  \label{fig:example-shrink}
\end{figure}

As shown in Figure \ref{fig:example-shrink}(A), the red objects are $\beta$ neighbors of the orange object, and the light red object is the center of these red objects. According to Equation (\ref{eq6}), the orange object will move towards the light red object. When all objects have been moved, the inter-cluster distance becomes larger and the intra-cluster distance becomes smaller, as shown in Figure \ref{fig:example-shrink}(B). In other words, the clusters in ODAR space become more compact, which is very favorable for identifying the outlier-cluster in ODAR space. We will experimentally verify the necessity of the shrinking method in Section \ref{sec:the shrink method}.

\subsection{Clustering strategy}
\textbf{Motivation.} There are normal objects with extremely large local densities (hereafter referred to as \emph{\textbf{large-density objects}}) in some datasets. Since the number of large-density objects is small, their high-order densities are as small as outliers. In such datasets, objects can be subdivided into three types, \emph{i.e.}, common normal objects (large local density, large high-order density), large-density objects (large local density, small high-order density), and outliers (small local density, small high-order density). Obviously, large-density objects are similar to common normal objects as well as outliers. If clustering directly in the ODAR space, some clustering algorithms with weak abilities may misidentify large-density objects and outliers as a cluster. In order to improve the universality allowing clustering algorithms with different abilities to accurately identify the outlier-cluster, we design a new clustering strategy: component clustering. The component clustering strategy does not directly cluster the dataset but first clusters the local density dimension to exclude large-density objects, and then second-clusters the previous result through the high-order density dimension to exclude common normal objects.

For clustering on the local density dimension, since the local densities of outliers must be small, we adopt $M(\rho)$, which represents the median of $\rho$, as a simple clustering division to get $outlier_{\rho}$,
\begin{equation}
outlier_{\rho}=\{o_i|\rho_i<M(\rho)\}.
\label{eq7}
\end{equation}
$outlier_{\rho}$ excludes large-density objects. Next, any clustering algorithm can be selected to cluster the high-order density dimension of $outlier_{\rho}$,
\begin{equation}
label=clustering(h\rho^{\dagger}),
\label{eq8}
\end{equation}
in which $clustering$ is a clustering algorithm, $h\rho^{\dagger}=\{h\rho_j|o_j\in outlier_{\rho}\}$. As mentioned above, the high-order densities of outliers must also be small. Therefore, the outlier-cluster is
\begin{equation}
outlier=\{o_i|o_i\in outlier_{\rho}, label(h\rho_i)==label(h\rho_\Theta)\},
\label{eq9}
\end{equation}
in which $h\rho_\Theta$ is the smallest high-order density in $h\rho^{\dagger}$. Finally, the objects in $outlier$ are the detected outliers. Algorithm \ref{ag2} describes the detailed implementation of the component clustering strategy. We will experimentally verify the necessity of the component clustering strategy in Section \ref{sec:component clystering strategy}.

\begin{breakablealgorithm}
\label{ag2}
%\begin{algorithm}[H]
\caption{\emph{The component clustering strategy}.}
\algblock[Name]{Begin}{End}
\begin{algorithmic}[1]
\Inputs{$ODAR$}
\Output{$outlier$} \# it contains all detected outliers.
\State \textbf{Step 1: Clustering on the local density dimension.}
\For{ $i$ $in$ $range(N)$}
\If{ ODAR$[i,0]<M(\rho)$} \# $M(\rho)$ is the median of $\rho$
\State $outlier_{\rho}.append(i)$;
\EndIf
\EndFor
\State \textbf{Step 2: Clustering on the high-order density dimension.}
\State $min\_index\leftarrow argmin(outlier_{\rho}[:,1])$; \# $min\_index$ stores the subscript of the object with the smallest high-order density.
\State $label\leftarrow clustering(outlier_{\rho}[:,1])$; \# $label$ stores the clustering result on the high-order density dimension of $outlier_{\rho}$.
\State $min\_label\leftarrow label[min\_index]$; \# $min\_label$ stores the category of the object with the smallest high-order density.
\State $outlier$ $\leftarrow where(label==min\_label)$; \# Searching for objects with category $min\_label$.
\State \emph{Return $outlier$.}
\end{algorithmic}
\end{breakablealgorithm}
\normalsize

\subsection{Time complexity analysis}
ODAR actually constructs a feature space for detecting outliers, its time complexity is discussed as follows. When calculating local density, it needs to iterate through $k$ neighbors of each object, so the time complexity is $O(kN)$, in which $k\ll N$. When calculating high-order density, it needs to first sort local density values and then iterate through the neighbors in the neighborhood of each object with radius $\sigma$, so the time complexity is $O(sN)$, $s$ is the average number of neighbors per object and $s\ll N$. When compacting the clusters in ODAR space, it needs to iterate through $\beta$ neighbors of each object, so the time complexity is $O(\beta N)$, in which $\beta\ll N$. In summary, the total time complexity of ODAR is $(k+s+\beta)N$. We will compare the running times of ODAR and other methods in Section \ref{sec:Runtime}.

\section{EXPERIMENTS}
\label{sec:experiments}

\subsection{Experimental Settings}
Following, we describe the datasets, baseline methods and their parameters, the evaluation indicator that we use, and the clustering algorithms being tested.

\subsubsection{Datasets}
We select some common synthetic datasets and real-world datasets to test ODAR. Synthetic datasets (from clustering basic benchmark \cite{ClusteringDatasets}) are described in the first part of Table \ref{tab:Datasets}. Real-world datasets (from ODDS Library \cite{Rayana2016}) are described in the second part of Table \ref{tab:Datasets}, and their more detailed information is described as follows.
\begin{itemize}
\item \emph{Annthyroid} consists of many health indicators, which determine whether a patient is hypothyroid (\emph{i.e.}, outlier). 
\item In the \emph{arrhythmia} dataset, the hyperfunction and subnormal instances are treated as outliers. 
\item \emph{Glass} is a dataset about glass types. Category 6 is a clear minority class, so the objects of category 6 are marked as outliers.
\item \emph{Ionosphere} is a binary dataset, in which outliers are the objects in the \emph{bad} class. 
\item \emph{Letter} records black-and-white rectangular pixel of capital letters in the English alphabet. Outliers are a few instances of letters that are not in normal letters. 
\item In the \emph{lympho} dataset, outliers are the rare lymphatics instances. 
\item \emph{Mnist} is a common handwritten dataset, and it is converted for outlier detection. Digit-$0$ images are considered as normal objects, and outliers are some images sampled from digit-$6$ because the handwritten $6$ and $0$ are easily confused. 
\item \emph{Thyroid} is a dataset of thyroid disease, in which outliers are the hyperfunction instances. 
\item \emph{Vowels} records the speech-time-series data of multiple male speakers, and outliers are from a certain speaker. 
\item \emph{Wbc} stores the measurements for breast cancer instances. Outliers are malignant instances.
\end{itemize}

\begin{table}
  \caption{The Details of Datasets.}
  \label{tab:Datasets}
  \centering
  \setlength{\tabcolsep}{1mm}
  \begin{tabular}{c|ccccccccccc}
    \hline
    \multicolumn{1}{c}{ }&\multicolumn{1}{c}{Name}&\multicolumn{1}{c}{Number}&\multicolumn{1}{c}{Dimension}&\multicolumn{1}{c}{Outliers(\%)}\\
    \hline
    \multirow{11}{*}{Synthetic datasets}
    & \emph{T4.8k} & 8000 & 2 & 13.61\%\\
    & \emph{T7.10k} & 10000 & 2 & 11.02\%\\
    & \emph{Worm-num-least} & 18799 & 2 & 8.69\%\\
    & \emph{Worm-num-medium} & 22503 & 2 & 22.01\%\\
    & \emph{Worm-num-most} & 29932 & 2 & 40.92\%\\
    & \emph{Unbalanced-2} & 1454 & 2 & 0.55\%\\
    & \emph{Unbalanced-10} & 1070 & 2 & 0.84\%\\
    & \emph{Unbalanced-15} & 1553 & 2 & 0.58\%\\
    \hline
    \multirow{10}{*}{Real-world datasets}
    & \emph{Annthyroid} & 7200 & 6 & 7.42\%\\
    & \emph{Arrhythmia} & 452 & 274 & 15\%\\
    & \emph{Glass} & 214 & 9 & 3.2\%\\
    & \emph{Ionosphere} & 351 & 33 & 36\%\\
    & \emph{Letter} & 1600 & 32 & 6.25\%\\
    & \emph{Lympho} & 148 & 18 & 4.1\%\\
    & \emph{Mnist} & 7603 & 100 & 9.2\%\\
    & \emph{Thyroid} & 3772 & 6 & 2.5\%\\
    & \emph{Vowels} & 1456 & 12 & 3.4\%\\
    & \emph{Wbc} & 278 & 30 & 5.6\%\\
    \hline

  \end{tabular}
\end{table}

\begin{table*}
  \caption{Input Parameters of ODAR on each Dataset.}
  \label{tab:table2}
  \centering
  \begin{tabular}{cccccccccccc}
    \hline
    \multicolumn{1}{c}{ }&\multicolumn{1}{c}{\emph{Annthyroid}}&\multicolumn{1}{c}{\emph{Arrhythmia}}&\multicolumn{1}{c}{\emph{Glass}}&\multicolumn{1}{c}{\emph{Ionosphere}}&\multicolumn{1}{c}{\emph{Letter}}&\multicolumn{1}{c}{\emph{Lympho}}&\multicolumn{1}{c}{\emph{Mnist}}&\multicolumn{1}{c}{\emph{Thyroid}}&\multicolumn{1}{c}{\emph{Vowels}}&\multicolumn{1}{c}{\emph{Wbc}}\\
    
    \hline
    ODAR (\emph{delta clustering}) & 3 & 17 & 11 & 5 & 6 & 3 & 7 & 4 & 9 & 7 \\
    ODAR (\emph{kmeans}) & 2 & 4 & 13 & 18 & 4 & 7 & 19 & 9 & 12 & 4 \\
    ODAR (\emph{DPC}) & 3 & 20 & 20 & 18 & 2 & 18 & 19 & 19 & 8 & 5\\
    \hline

  \end{tabular}
\end{table*}

\subsubsection{Baseline Methods and Parameters}
\label{sectionparameter}
ODAR is an outlier detection approach for clustering algorithms, so we select outlier-clustering algorithms (that is, clustering algorithms with outlier detection ability) for comparison: \emph{extreme clustering} \cite{wang2020extreme}, \emph{delta clustering} \cite{wang2018delta}, \emph{DPC} \cite{rodriguez2014clustering}, and \emph{DBSCAN} \cite{ester1996density}. Besides, we also select some popular or recently proposed general outlier detection methods. Specifically, six traditional methods: \emph{LOF}  \cite{breunig2000lof}, \emph{Isforest} \cite{liu2012isolation}, \emph{ECOD} \cite{li2022ecod}, \emph{COPOD} \cite{2020COPOD}, \emph{LSCP} \cite{zhao2019lscp}, \emph{OTF} \cite{huang2023novel}, and seven neural network-based methods: \emph{DIF} \cite{xu2023deep}, \emph{LUNAR} \cite{goodge2022lunar}, \emph{BAE} \cite{sarvari2021unsupervised}, \emph{RCA} \cite{liu2021rca}, \emph{MO-GAAL} \cite{liu2020generative}, \emph{SO-GALL} \cite{liu2020generative}, \emph{DeepSVDD} \cite{ruff2018deep}. Since some methods' optimal parameters are difficult to determine directly, we execute them with successive parameter intervals and select the values with the best performance as their input parameters. 
\begin{itemize}
\item For \emph{DBSCAN}, the interval of $minpts$ is from 1 to 20, and the interval of $eps$ is $dis*s$, in which $s$ is from 1 to 20, and $dis$ is the mean of the distance between objects and their nearest objects. 
%\item The interval of $min\_cluster\_size$ in \emph{HDBSCAN} is from 1 to 20. 
\item The interval of $\delta$ in \emph{extreme clustering} is $dis*s$, in which $s$ is from 1 to 20.
\item The interval of $radius$ in \emph{delta clustering} is $dis*s$, in which $s$ is from 1 to 20. 
%\item The interval of $alpha$ in \emph{CBLOF} is $i*0.05$, in which $i$ is from 1 to 20. 
\item The interval of $n\_neighbors$ in \emph{LOF} is $i$, in which $i$ is from 1 to 20. 
%\item The interval of $n\_neighbors$ in \emph{KNN} is the same as \emph{LOF}. 
\item The interval of $k$ in \emph{MO-GAAL} is from 1 to 20. 
\item The interval of $n\_neighbours$ in \emph{LUNAR} is from 1 to 20. 
\item The interval of $n\_estimators$ in \emph{DIF} is from 1 to 20.
\item As for \emph{LSCP} and \emph{Isforest}, their results are unstable, so we repeat them 30 times and select the best values as final accuracies. 
\item The input parameters of ODAR with three clustering algorithms are recorded in Table \ref{tab:table2}. 
\end{itemize}

\subsubsection{clustering algorithms}
We select the three clustering algorithms, \emph{delta clustering}, \emph{DPC} and \emph{kmeans}, to verify the robustness and effectiveness of ODAR.

\textbf{Why do we only select the three algorithms in this paper?} Our choice does not mean that ODAR can only match the three clustering algorithms. On the contrary, any clustering algorithm can detect outliers with the help of ODAR. Here, we select them for the following reasons: 

\textbf{$\bullet$ Verifying that ODAR is universal.} As we know, different clustering algorithms have different cluster identification abilities. Some algorithms can identify very complex clusters, but some algorithms can only identify simple clusters. If ODAR is a universal outlier detection approach, then ODAR must be friendly to all clustering algorithms. \emph{Delta clustering}, \emph{DPC}, and \emph{kmeans} are three representative clustering algorithms. \emph{Kmeans} can only identify spherical clusters. \emph{DPC} is better than \emph{kmeans}, it can identify simple shape clusters. However, once shape clusters are complex, \emph{DPC} will fail. \emph{Delta clustering} has high accuracy, regardless of the shape of clusters, it can accurately identify them. Therefore, if the three clustering algorithms can accurately detect outliers with the help of ODAR, it is sufficient to show that ODAR is also friendly to other clustering algorithms, \emph{i.e.}, ODAR is an excellent universal outlier detection approach for clustering algorithms. 

\textbf{$\bullet$ Explaining the value of ODAR.} \emph{Delta clustering} and \emph{DPC} have the ability to detect outliers. If their detected results with the help of ODAR are better than their original detected results, it is sufficient to show that ODAR is very valuable.

\subsubsection{Evaluation Indicator}
Two approaches are used to measure outlier detection accuracy:

\textbf{$\bullet$ Numerical approach.} Outlier-clustering algorithms (including the clustering algorithms with the help of ODAR) directly output outliers as results, but general outlier detection methods usually only output the score that an object is an outlier. Traditional evaluation indicator, ROC, is unable to assess the performance between outlier-clustering algorithms and general outlier detection methods. Therefore, we treat the top $x$\% of objects with the highest score ($x$\% is the ground-truth outlier rate recorded in Table \ref{tab:Datasets}) as outliers of general outlier detection methods. Next, we employ
\begin{equation}
accuracy=\frac{1}{2}(\frac{TP}{TP+FN}+\frac{TN}{FP+TN})
\end{equation}
as the evaluation indicator to compare the detected results and ground-truth labels. \emph{TP} is the number of ground-truth outliers detected as outliers. \emph{FN} is the number of ground-truth outliers detected as normal objects. \emph{TN} is the number of ground-truth normal objects detected as normal objects. \emph{FP} is the number of ground-truth normal objects detected as outliers. The closer the indicator value is to 1, the more accurate the outlier detection result.

\textbf{$\bullet$ Visualization approach.} For the 2-dimensional dataset, we visualize the dataset and mark the detected outliers in orange. Since the naked eye can distinguish which objects are ground-truth outliers, the higher the number of marked ground-truth outliers and the lower the number of marked ground-truth normal objects, the more accurate the outlier detection result.

\begin{table*}
  \caption{Comparison of Accuracy between ODAR and Baseline Methods.}
  \label{tab:table3}
  \setlength{\tabcolsep}{0.5mm}
  \centering
  \begin{tabular}{ccccccccccc|c}
    \hline
    \multicolumn{1}{c}{ }&\multicolumn{1}{c}{\emph{Annthyroid}}&\multicolumn{1}{c}{\emph{Arrhythmia}}&\multicolumn{1}{c}{\emph{Glass}}&\multicolumn{1}{c}{\emph{Ionosphere}}&\multicolumn{1}{c}{\emph{Letter}}&\multicolumn{1}{c}{\emph{Lympho}}&\multicolumn{1}{c}{\emph{Mnist}}&\multicolumn{1}{c}{\emph{Thyroid}}&\multicolumn{1}{c}{\emph{Vowels}}&\multicolumn{1}{c}{\emph{Wbc}}&\multicolumn{1}{|c}{Average}\\
    
    \hline
    %\emph{HDBSCAN} & 0.63 & 0.69 & 0.83 & 0.82 & 0.81 & 0.87 & 0.72 & 0.56 & 0.92 & 0.72\\
    \emph{Extreme clustering} & 0.56 &	0.56 & 0.55 & 0.5 & 0.58 & 0.58 & 0.5 & 0.64 & 0.75 & 0.52 & 0.574\\
    \emph{DBSCAN} & 0.69 & 0.72 & \textbf{0.91} & \textbf{0.90} & 0.78 & 0.70 & 0.66 & 0.90 & 0.95 & 0.82 & 0.803\\
    %\emph{CBLOF} & 0.53 & 0.73 & 0.54 & 0.66 & 0.55 & 0.82 & 0.70 & 0.57 & 0.52 & 0.80\\
    \emph{Delta clustering} & 0.65 & 0.71 & 0.86 & 0.89 & 0.58 & 0.64 & 0.67 & 0.88 & 0.95 & 0.82 & 0.765\\
    \emph{DPC} & 0.52 & 0.50 & 0.49 & 0.50 & 0.61 & 0.55 & 0.50 & 0.51 & 0.51 & 0.50 & 0.519\\
    \emph{LOF} & 0.64 &0.72 &0.59 & 0.88 & 0.76 & 0.91 & 0.65 & 0.69 & 0.69 & 0.77 & 0.730\\
    \emph{Isforest} & 0.64 & 0.73 & 0.54 & 0.76 & 0.53 & \textbf{1.00} & 0.66 & 0.81 & 0.61 & 0.82 & 0.710\\
    %\emph{KNN} & 0.62 & 0.72 & 0.54 & \textbf{0.90} & 0.74 & 0.83 & 0.69 & 0.66 & 0.82 & 0.80\\
    \emph{ECOD} & 0.61 & 0.64 &	0.54 & 0.54 & 0.55 & 0.99 & 0.51 & \textbf{0.93} & 0.57 & 0.66 & 0.654\\
    \emph{COPOD} & 0.56 & 0.62 & 0.48 & 0.55 & 0.52 & 0.99 & 0.54 & 0.85 & 0.49 & 0.76 & 0.636\\
    \emph{LSCP} & 0.65 & 0.72 & 0.59 & 0.74 & 0.67 & 0.90 & 0.68 & 0.72 & 0.70 & 0.80 & 0.717\\
    \emph{SO-GAAL} & 0.47 & 0.50 & 0.48 & 0.37 & 0.48 & 0.48 & 0.50 & 0.49 & 0.48 & 0.47 & 0.472\\
    \emph{MO-GAAL} & 0.55 & 0.50 & 0.65 & 0.47 & 0.50 & 0.48 & 0.50 & 0.54 & 0.48 & 0.52 & 0.519\\
    \emph{BAE} & 0.58  & 0.70 & 0.54 & 0.67 & 0.50 & 0.83 & 0.66 & 0.71 & 0.53 & 0.82 & 0.654\\
    %\emph{DAGMM} & 0.63  & 0.66 & 0.54 & 0.84 & 0.64 & 0.74 & 0.73 & 0.72 & 0.62 & 0.75\\
    \emph{RCA} & 0.61 & 0.66 & 0.55 & 0.75 & 0.61 & 0.75 & 0.71 & 0.68 & 0.66 & 0.76 & 0.674\\
    \emph{DeepSVDD} & 0.61 & 0.60 & 0.49 & 0.56 & 0.51 & 0.74 & 0.58 & 0.49 & 0.48 & 0.52 & 0.558\\
    \emph{LUNAR} & 0.62  & 0.68 & 0.68 & 0.64 & 0.77 & 0.97 & 0.68 & 0.88 & 0.88 & 0.82 & 0.762\\
    \emph{DIF} & 0.58 & 0.66 & 0.48 & 0.76 & 0.54 & 0.74 & 0.69 & 0.76 & 0.55 & 0.65 & 0.641\\
    \emph{OTF} & 0.63 & 0.66 & 0.54 & 0.75 & 0.81 & 0.82 & 0.72 & 0.84 & 0.79 & 0.85 & 0.741\\
    \hline
    ODAR+\emph{kmeans} & \textbf{0.70} & 0.75 & 0.79 & 0.89 & \textbf{0.86} & 0.90 & \textbf{0.74} & 0.84 & 0.93 & 0.85 & 0.825\\
    \multirow{2}{*}{ODAR+\emph{delta}} & 0.69 & \textbf{0.77} & 0.90 & \textbf{0.90} & 0.84 & 0.96 & 0.71 & 0.91 & 0.95 & \textbf{0.92} & \textbf{0.855}\\
    &(\textcolor{black}{+6.2\%})&(\textcolor{black}{+8.5\%})&(\textcolor{black}{+4.7\%})&(\textcolor{black}{+1.1\%})&(\textcolor{black}{+44.8\%})&(\textcolor{black}{+50\%})&(\textcolor{black}{+6.0\%})&(\textcolor{black}{+3.4\%})&(+0\%)&(\textcolor{black}{+12.2\%})&(\textcolor{black}{+11.8\%})\\
    \multirow{2}{*}{ODAR+\emph{DPC}} & \textbf{0.70} & 0.76 & 0.89 & 0.88 & \textbf{0.86} & 0.96 & \textbf{0.74} & 0.88 & \textbf{0.96} & 0.85 & 0.848\\
    &(\textcolor{black}{+34.6\%})&(\textcolor{black}{+52\%})&(\textcolor{black}{+81.6\%})&(\textcolor{black}{+76\%})&(\textcolor{black}{+41.0\%})&(\textcolor{black}{+74.5\%})&(\textcolor{black}{+48\%})&(\textcolor{black}{+72.5\%})&(\textcolor{black}{+88.2\%})&(\textcolor{black}{+70\%})&(\textcolor{black}{+63.4\%})\\
    \hline

  \end{tabular}
\end{table*}

\subsection{Comparison Experiments}
To verify the effectiveness of ODAR, we compared it with other methods on real-world datasets, and tested their performance under different parameters.

\subsubsection{Accuracy Analysis}
The accuracy of all methods is recorded in Table \ref{tab:table3}, in which the best results are bolded. Since normal objects and outliers in ODAR space belong to two clusters, $2$ is used as the parameter of \emph{kmeans}, and two clustering centers are selected from the \emph{DPC} decision graph. Based on Table \ref{tab:table3}, it can be concluded that:

\textbf{$\bullet$ ODAR can further improve the detection effect of outlier-clustering algorithms.} \emph{Delta clustering} and \emph{DPC} are outlier-clustering algorithms (\emph{i.e.}, clustering algorithms that can directly detect outliers). However, the effect of \emph{DPC} is very poor, its average accuracy is only 0.519. The average accuracy of \emph{delta clustering} is 0.765. With the help of ODAR, the accuracies of \emph{DPC} and \emph{delta clustering} have been improved, and the improvement rates are recorded in parentheses in Table \ref{tab:table3}. Specifically, the average accuracy of \emph{delta clustering} is improved to 0.855 and the average accuracy of \emph{DPC} is improved to 0.848. The average improvement rate of \emph{DPC} is up to 63.4\%. In addition, ODAR enables \emph{kmeans}, a clustering algorithm without outlier detection ability, to accurately detect outliers.

\textbf{$\bullet$ ODAR is an excellent outlier detection approach for clustering algorithms.} ODAR (\emph{i.e.}, three clustering algorithms with the help of ODAR) performs the best on 7 out of 10 datasets, such as \emph{annthyroid}, \emph{arrhythmia}, \emph{ionosphere}, \emph{letter}, \emph{mnist}, \emph{vowels}, and \emph{wbc}. ODAR does not achieve the best results on \emph{glass}, \emph{lympho}, and \emph{thyroid} datasets, but it achieves second-best results. In terms of average accuracy, ODAR is as high as 0.84, outperforming baseline methods by at least 5\%.

\begin{table*}
  \caption{The Performance of \emph{Delta Clustering} with ODAR under Different Parameter Values.}
  \label{tab:table4}
  \setlength{\tabcolsep}{3.5mm}
  \centering
  \begin{tabular}{cccccccccccc}
    \hline
    \multicolumn{1}{c}{ }&\multicolumn{1}{c}{\emph{k}=2}&\multicolumn{1}{c}{\emph{k}=4}&\multicolumn{1}{c}{\emph{k}=6}&\multicolumn{1}{c}{\emph{k}=8}&\multicolumn{1}{c}{\emph{k}=10}&\multicolumn{1}{c}{\emph{k}=12}&\multicolumn{1}{c}{\emph{k}=14}&\multicolumn{1}{c}{\emph{k}=16}&\multicolumn{1}{c}{\emph{k}=18}&\multicolumn{1}{c}{\emph{k}=20}\\
    
    \hline
    \emph{Annthyroid dataset} & 0.68 & 0.69 & 0.68 & 0.66 & 0.68 & 0.68 & 0.67 & 0.66 & 0.66 & 0.66\\
    \emph{Arrhythmia dataset} & 0.68 & 0.68 & 0.69 & 0.69 & 0.75 & 0.68 & 0.76 & 0.70 & 0.76 & 0.76\\
    \emph{Glass dataset} & 0.71 & 0.87 & 0.88 & 0.87 & 0.88 & 0.88 & 0.89 & 0.89 & 0.89 & 0.85\\
    \emph{Ionosphere dataset} & 0.89 & 0.88 & 0.87 & 0.89 & 0.89 & 0.88 & 0.88 & 0.87 & 0.87 & 0.87\\
    \emph{Letter dataset} & 0.68 & 0.79 & 0.84 & 0.83 & 0.68 & 0.68 & 0.83 & 0.69 & 0.80 & 0.72\\
    \emph{Lympho dataset} & 0.88 & 0.95 & 0.91 & 0.93 & 0.94 & 0.96 & 0.93 & 0.95 & 0.62 & 0.92\\
    \emph{Mnist dataset} & 0.63 & 0.65 & 0.69 & 0.68 & 0.71 & 0.70 & 0.66 & 0.67 & 0.66 & 0.71\\
    \emph{Thyroid dataset} & 0.88 & 0.91 & 0.90 & 0.90 & 0.90 & 0.90 & 0.89 & 0.90 & 0.90 & 0.89\\
    \emph{Vowels dataset} & 0.84 & 0.91 & 0.80 & 0.79 & 0.78 & 0.90 & 0.83 & 0.77 & 0.77 & 0.92\\
    \emph{Wbc dataset} & 0.81 & 0.82 & 0.82 & 0.82 & 0.82 & 0.85 & 0.86 & 0.90 & 0.84 & 0.90\\
    \hline
    \emph{Average accuracy} & 0.768 & 0.815 & 0.808 & 0.806 & 0.803 & 0.811 & 0.820 & 0.800 & 0.777 & 0.820\\
    \hline

  \end{tabular}
\end{table*}

\begin{table*}
  \caption{The Performance of \emph{Kmeans} with ODAR under Different Parameter Values.}
  \label{tab:table5}
  \setlength{\tabcolsep}{3.5mm}
  \centering
  \begin{tabular}{cccccccccccc}
    \hline
    \multicolumn{1}{c}{ }&\multicolumn{1}{c}{\emph{k}=2}&\multicolumn{1}{c}{\emph{k}=4}&\multicolumn{1}{c}{\emph{k}=6}&\multicolumn{1}{c}{\emph{k}=8}&\multicolumn{1}{c}{\emph{k}=10}&\multicolumn{1}{c}{\emph{k}=12}&\multicolumn{1}{c}{\emph{k}=14}&\multicolumn{1}{c}{\emph{k}=16}&\multicolumn{1}{c}{\emph{k}=18}&\multicolumn{1}{c}{\emph{k}=20}\\
    
    \hline
    \emph{Annthyroid dataset} & 0.70&	0.70&	0.69&	0.68&	0.67&	0.67&	0.67&	0.66&	0.66&	0.66\\
    \emph{Arrhythmia dataset} & 0.74&	0.75&	0.68&	0.70&	0.70&	0.70&	0.70&	0.70&	0.70&	0.70\\
    \emph{Glass dataset} & 0.76&	0.74&	0.75&	0.76&	0.78&	0.79&	0.79&	0.79&	0.78&	0.73\\
    \emph{Ionosphere dataset} & 0.77&	0.77&	0.78&	0.78&	0.82&	0.87&	0.86&	0.87&	0.89&	0.88\\
    \emph{Letter dataset} & 0.85&	0.86&	0.84&	0.83&	0.84&	0.83&	0.83&	0.83&	0.82&	0.83\\
    \emph{Lympho dataset} & 0.87&	0.89&	0.89&	0.88&	0.73&	0.82&	0.73&	0.74&	0.76&	0.74\\
    \emph{Mnist dataset} & 0.70&	0.71&	0.72&	0.73&	0.72&	0.73&	0.73&	0.72&	0.71&	0.73\\
    \emph{Thyroid dataset} & 0.81&	0.83&	0.83&	0.84&	0.84&	0.83&	0.82&	0.82&	0.82&	0.82\\
    \emph{Vowels dataset} & 0.83&	0.90&	0.89&	0.87&	0.81&	0.93&	0.91&	0.88&	0.89&	0.81\\
    \emph{Wbc dataset} & 0.78&	0.85&	0.76&	0.76&	0.76&	0.80&	0.82&	0.82&	0.82&	0.82\\
    \hline
    \emph{Average accuracy} & 0.781&	0.800&	0.783&	0.783&	0.767&	0.797&	0.786&	0.783&	0.785&	0.772\\
    \hline

  \end{tabular}
\end{table*}

\begin{table*}
  \caption{The Performance of \emph{DPC} with ODAR under Different Parameter Values.}
  \label{tab:table6}
  \setlength{\tabcolsep}{3.5mm}
  \centering
  \begin{tabular}{cccccccccccc}
    \hline
    \multicolumn{1}{c}{ }&\multicolumn{1}{c}{\emph{k}=2}&\multicolumn{1}{c}{\emph{k}=4}&\multicolumn{1}{c}{\emph{k}=6}&\multicolumn{1}{c}{\emph{k}=8}&\multicolumn{1}{c}{\emph{k}=10}&\multicolumn{1}{c}{\emph{k}=12}&\multicolumn{1}{c}{\emph{k}=14}&\multicolumn{1}{c}{\emph{k}=16}&\multicolumn{1}{c}{\emph{k}=18}&\multicolumn{1}{c}{\emph{k}=20}\\
    
    \hline
    \emph{Annthyroid dataset} & 0.65&	0.68&	0.69&	0.66&	0.67&	0.68&	0.67&	0.65&	0.66&	0.66\\
    \emph{Arrhythmia dataset} & 0.71&	0.75&	0.68&	0.69&	0.68&	0.70&	0.69&	0.75&	0.75&	0.76\\
    \emph{Glass dataset} & 0.72&	0.65&	0.75&	0.83&	0.83&	0.80&	0.82&	0.83&	0.88&	0.89\\
    \emph{Ionosphere dataset} & 0.76&	0.71&	0.78&	0.75&	0.74&	0.85&	0.88&	0.86&	0.88&	0.87\\
    \emph{Letter dataset} & 0.86&	0.85&	0.80&	0.78&	0.84&	0.83&	0.83&	0.74&	0.68&	0.81\\
    \emph{Lympho dataset} & 0.88&	0.94&	0.93&	0.92&	0.94&	0.71&	0.90&	0.90&	0.96&	0.71\\
    \emph{Mnist dataset} & 0.69&	0.70&	0.62&	0.72&	0.70&	0.73&	0.72&	0.67&	0.72&	0.73\\
    \emph{Thyroid dataset} & 0.81&	0.83&	0.82&	0.84&	0.85&	0.84&	0.82&	0.82&	0.87&	0.86\\
    \emph{Vowels dataset} & 0.85&	0.90&	0.90&	0.96&	0.93&	0.90&	0.94&	0.93&	0.88&	0.88\\
    \emph{Wbc dataset} & 0.79&	0.81&	0.67&	0.65&	0.67&	0.66&	0.84&	0.84&	0.84&	0.84\\
    \hline
    \emph{Average accuracy} & 0.772&	0.782&	0.764&	0.780&	0.785&	0.770&	0.811&	0.799&	0.812&	0.801\\
    \hline

  \end{tabular}
\end{table*}

\begin{figure}[h]
  \centering
  \includegraphics[width=3in]{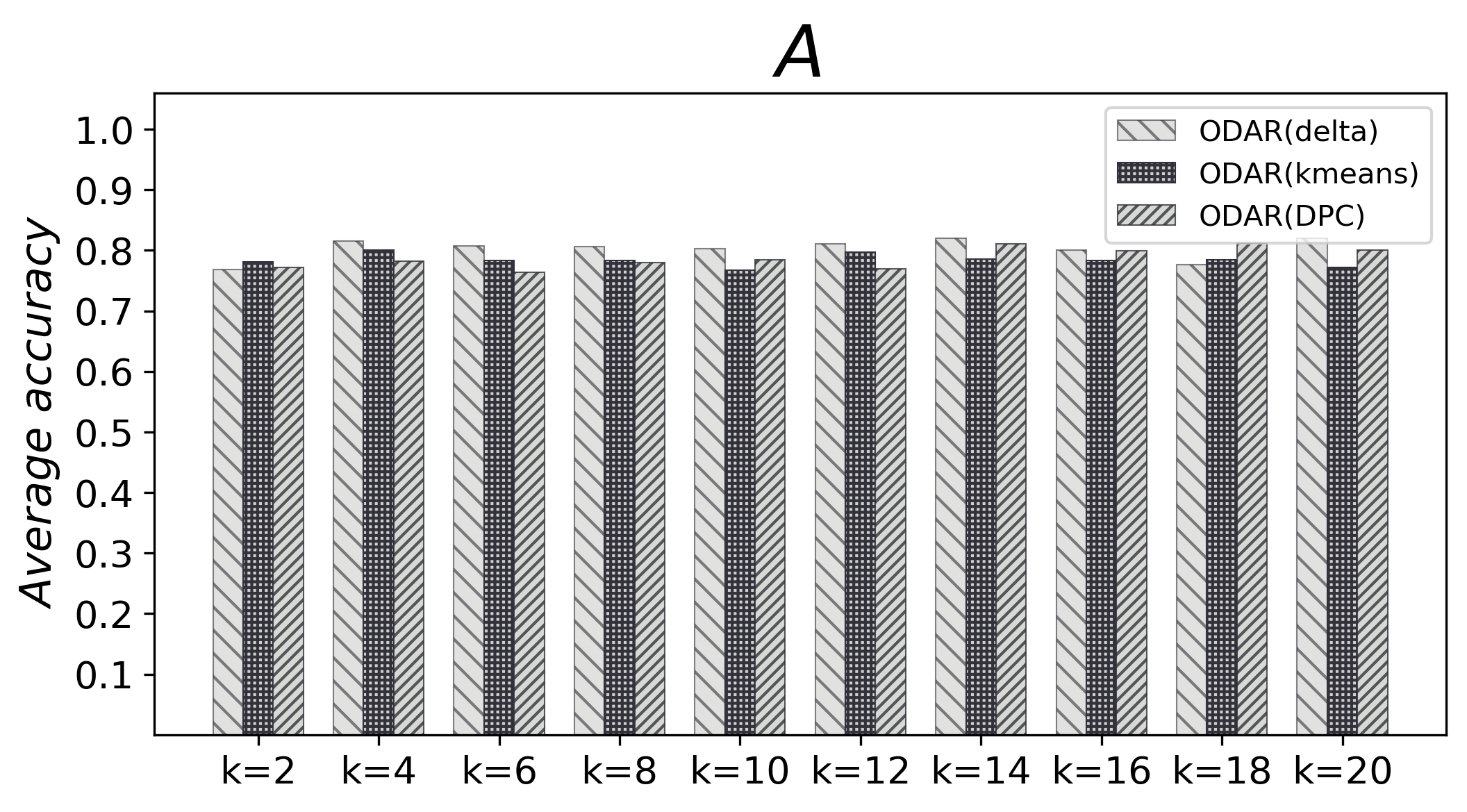}\\
  \includegraphics[width=3in]{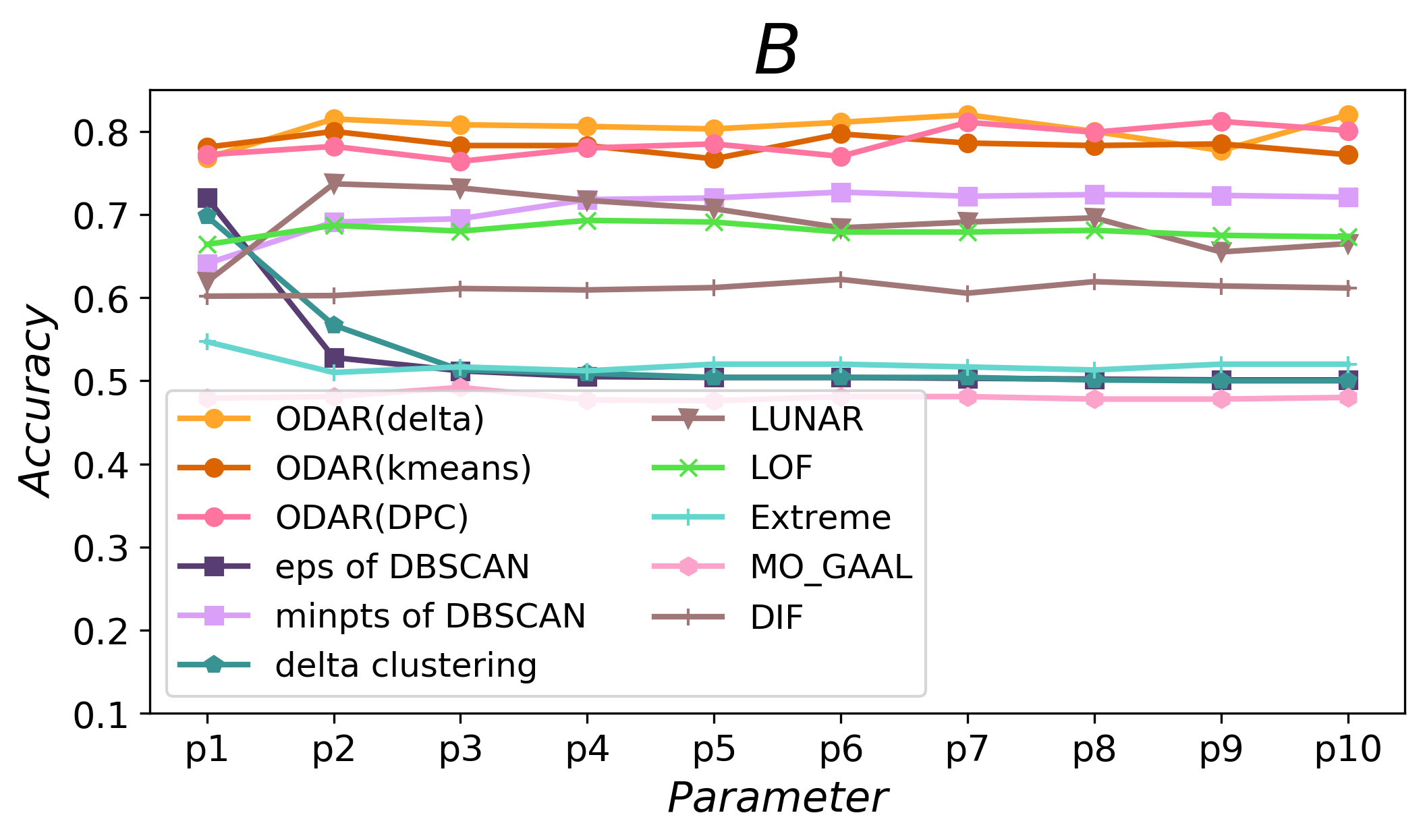}
  \caption{(A) ODAR is insensitive to the parameter $k$. (B) Under different parameter values, ODAR is superior to other methods.}
  \label{fig:parameter_compare}
\end{figure}

\subsubsection{Parameter Analysis}

Generally, we recommend that the parameter $k$ of ODAR is set to less than 20. An excellent approach should be insensitive to the input parameter. Thence, we further analyze the sensitivity of ODAR to input parameter $k$. Specifically, we modify ODAR's parameter $k$ from 2 to 20. Tables \ref{tab:table4}, \ref{tab:table5}, and \ref{tab:table6} record the accuracies of the three clustering algorithms with the help of the ODAR under different $k$. We also calculate the average accuracy of each clustering algorithm under each $k$. Figure \ref{fig:parameter_compare}(A) shows the variation of these average accuracies through a histogram. Experimental results show that the performance of each clustering algorithm is very stable regardless of the parameter $k$ of ODAR, and there are no extremes. Apparently, ODAR is not sensitive to input parameter $k$.

We also test the baseline methods, whose optimal parameters are difficult to determine directly, under different parameters. For each method, 10 parameter values are extracted from the parameter interval (described in Section \ref{sectionparameter}) in an arithmetic sequence. The average accuracies of all methods (including the three clustering algorithms with the help of ODAR) under different parameters are shown in Figure \ref{fig:parameter_compare}(B). We can observe that \emph{DBSCAN} is sensitive to parameter \emph{eps}, and only one parameter value makes it perform well. \emph{Delta clustering} and \emph{LUNAR} are also sensitive to parameters. Under most parameter values, \emph{delta clustering} performs poorly. Although the performance of \emph{MO-GAAL} and \emph{DIF} are stable, their accuracies are low. As for ODAR, no matter which clustering algorithm is matched, it is not only insensitive to parameter, but also superior to baseline methods under each parameter value.

\begin{figure}[h]
  \centering
  \includegraphics[width=3.4in]{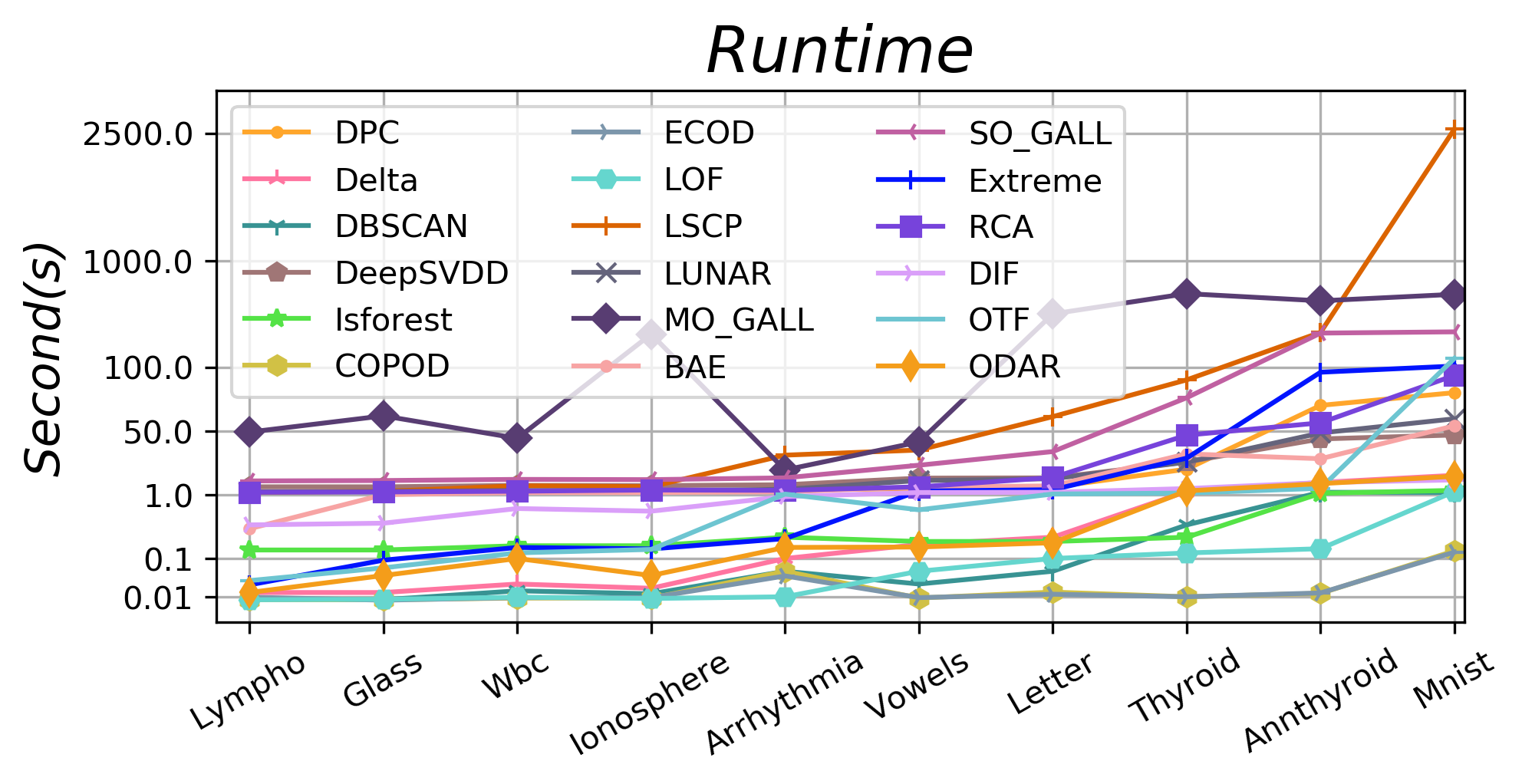}
  \caption{Comparison on running time.}
  \label{fig:runtime}
\end{figure}

\subsubsection{Runtime Analysis}
\label{sec:Runtime}
We further compare running time between \emph{ODAR} and baseline methods on real-world datasets, as shown in Figure \ref{fig:runtime}. The experimental results show that \emph{extreme clustering}, \emph{MO-GALL}, \emph{SO-GALL} and \emph{LSCP} have very long running times, running over 100 seconds on \emph{annthyroid} and \emph{mnist}, and even \emph{LSCP} has a running time of over 2500 seconds. \emph{DPC}, \emph{RCA}, \emph{BAE}, \emph{OTF}, \emph{LUNAR}, and \emph{DeepSVDD} run faster than the above methods, but they still have to run over 50 seconds on some large datasets. \emph{ODAR} and the remaining methods run less than 1 second on most datasets. COPOD and ECOD are the fastest, but their accuracy is much lower than ODAR on most datasets (see Table \ref{tab:table3} for details). Therefore, ODAR's runtime is acceptable.

\begin{figure*}[h]
  \centering
  \includegraphics[width=1.2in]{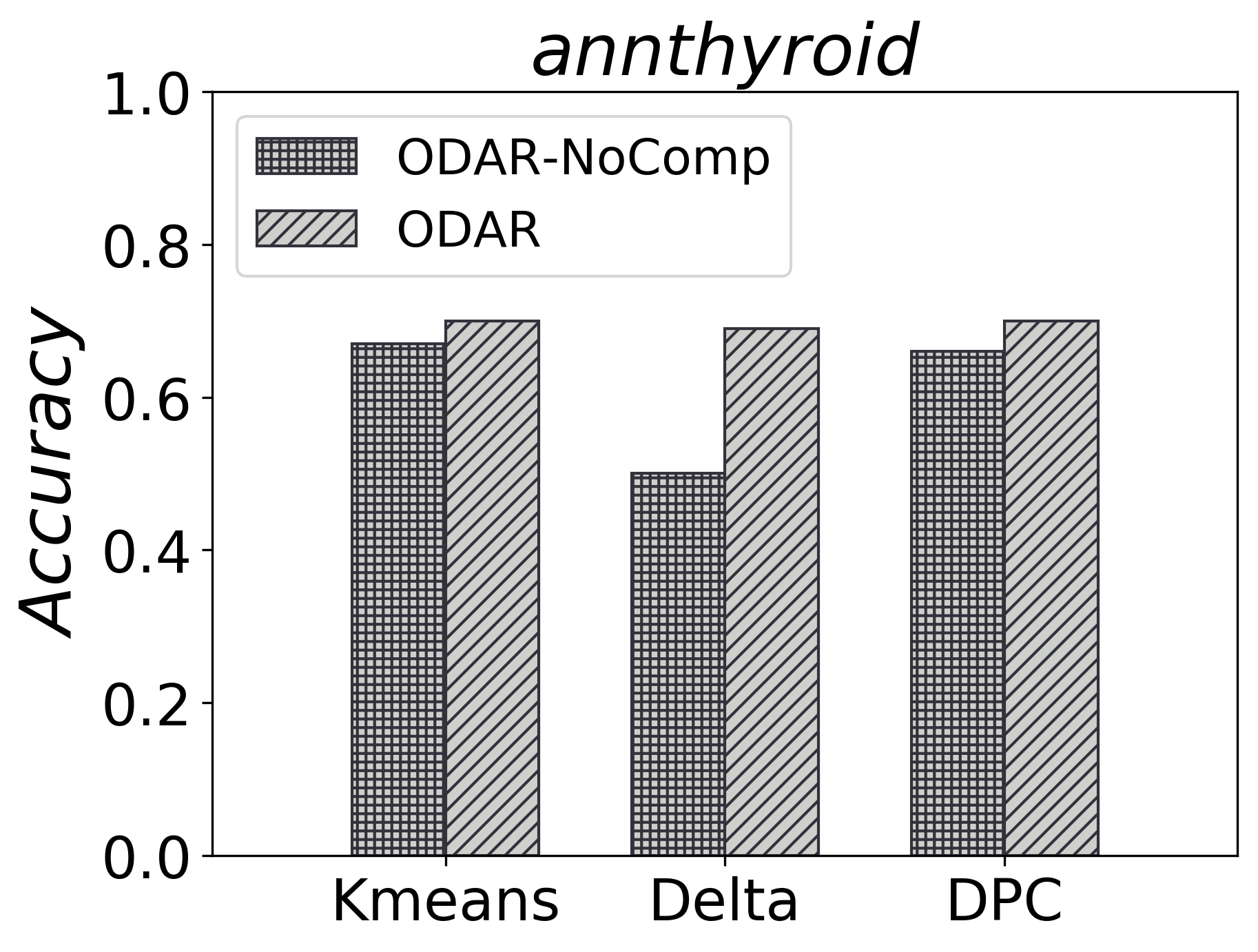}
  \quad
  \includegraphics[width=1.2in]{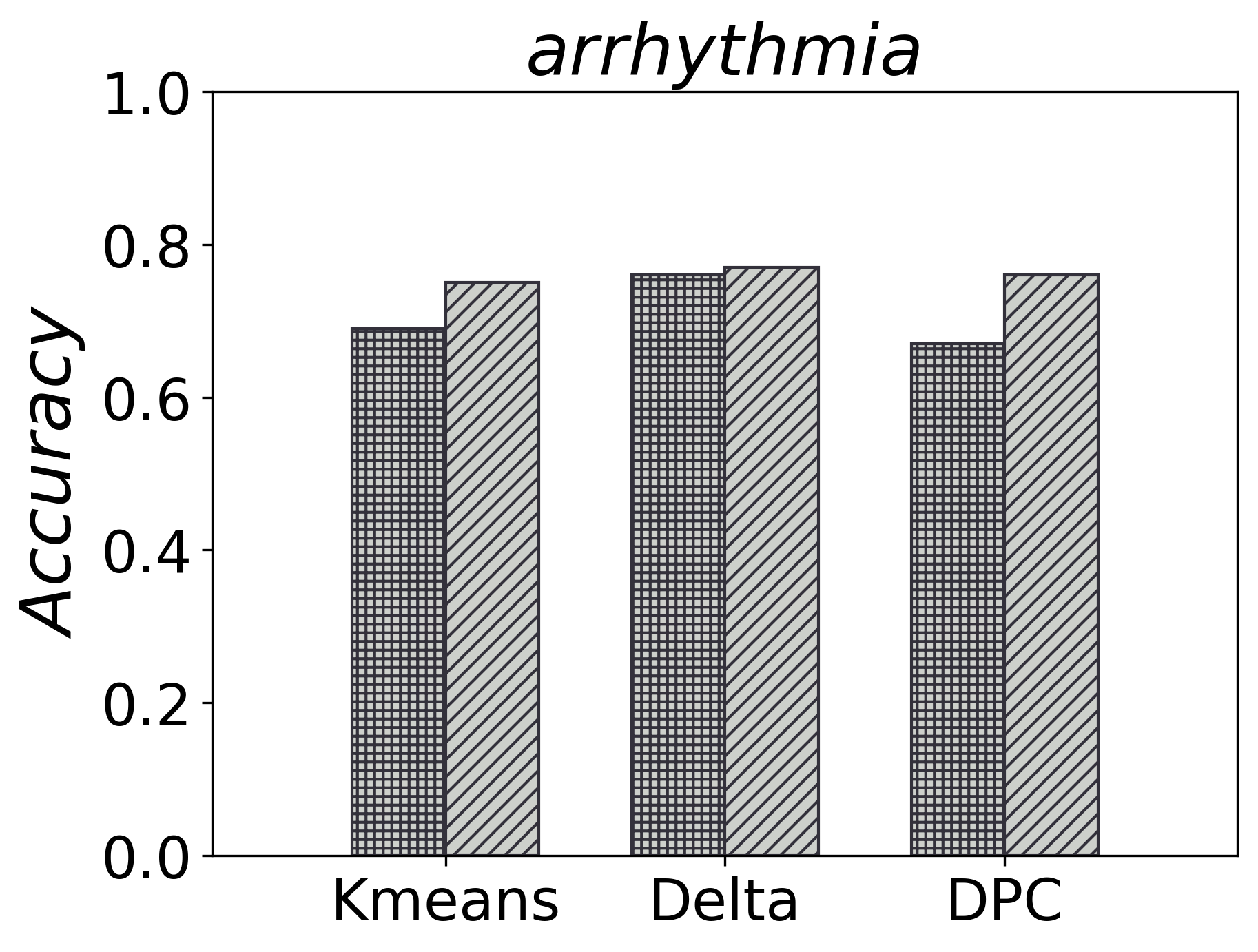}
  \quad
  \includegraphics[width=1.2in]{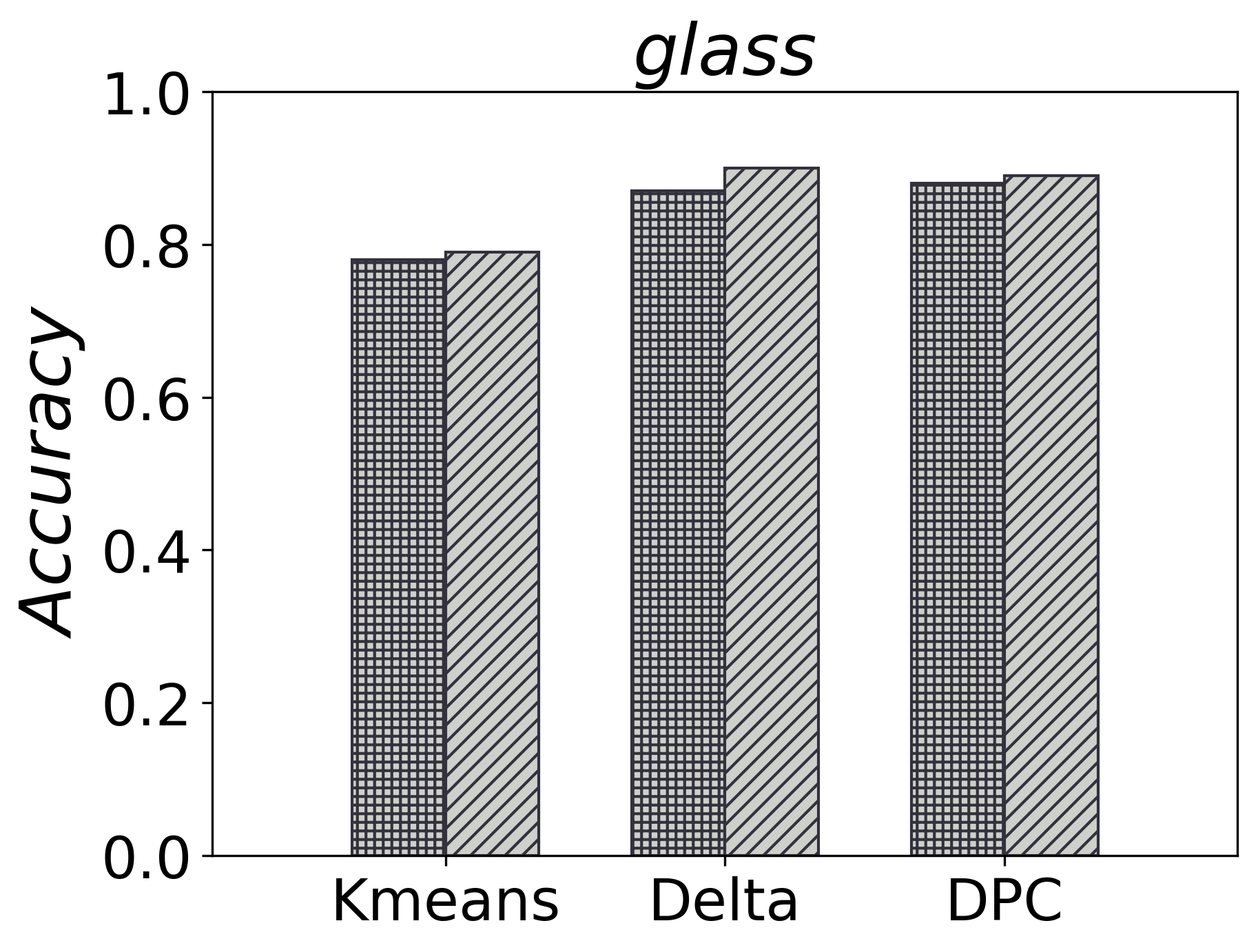}
  \quad
  \includegraphics[width=1.2in]{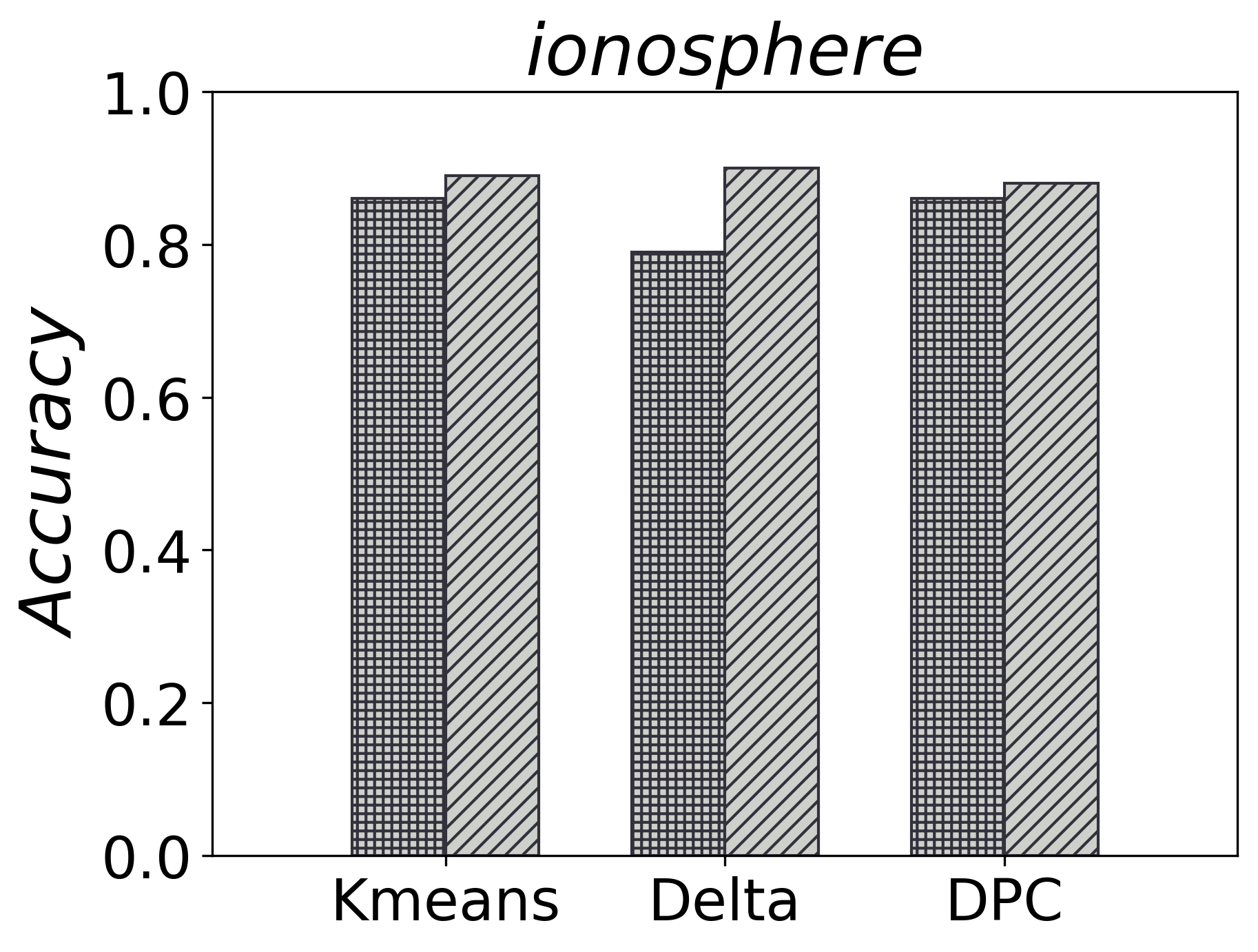}
  \quad
  \includegraphics[width=1.2in]{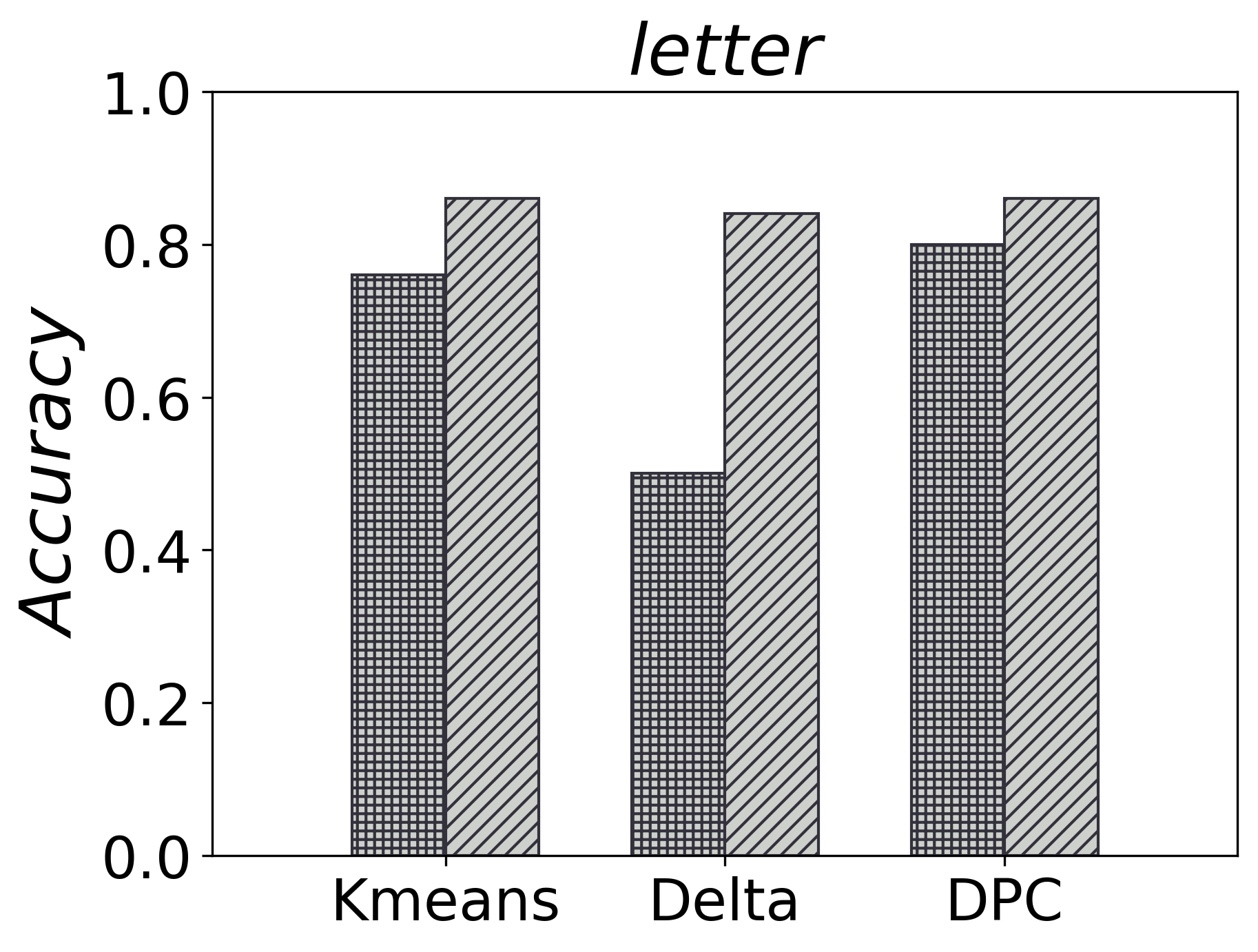}\\
  \includegraphics[width=1.2in]{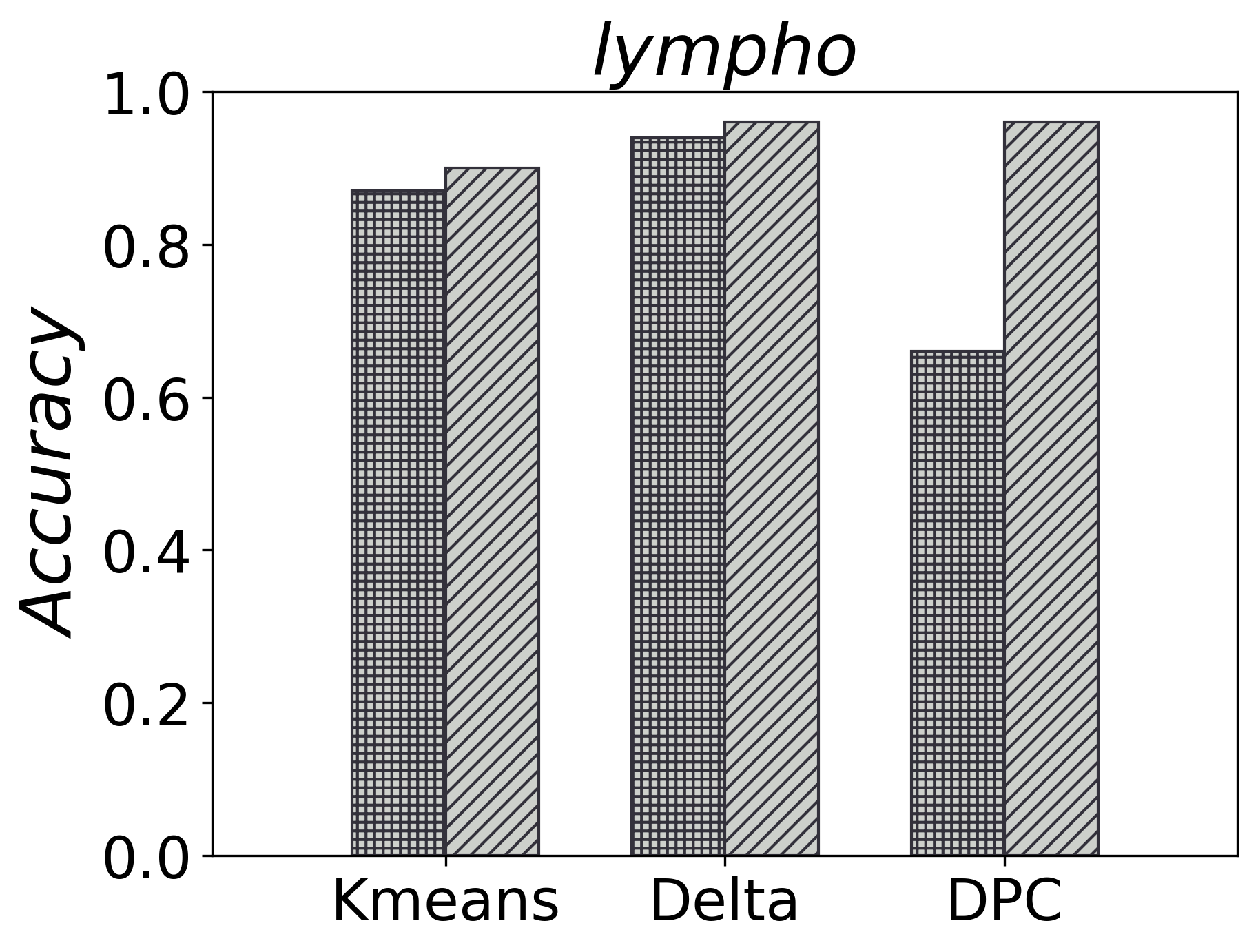}
  \quad
  \includegraphics[width=1.2in]{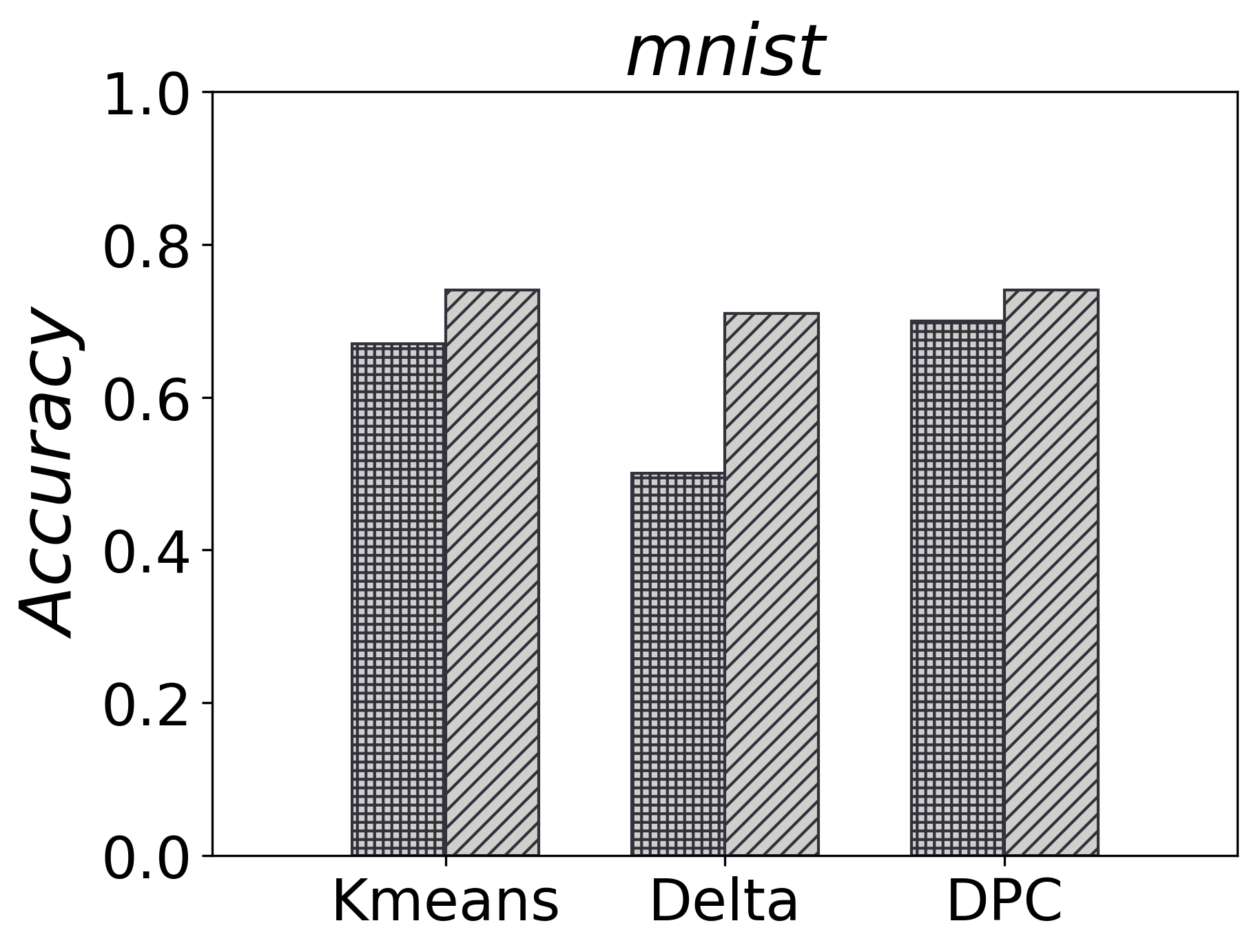}
  \quad
  \includegraphics[width=1.2in]{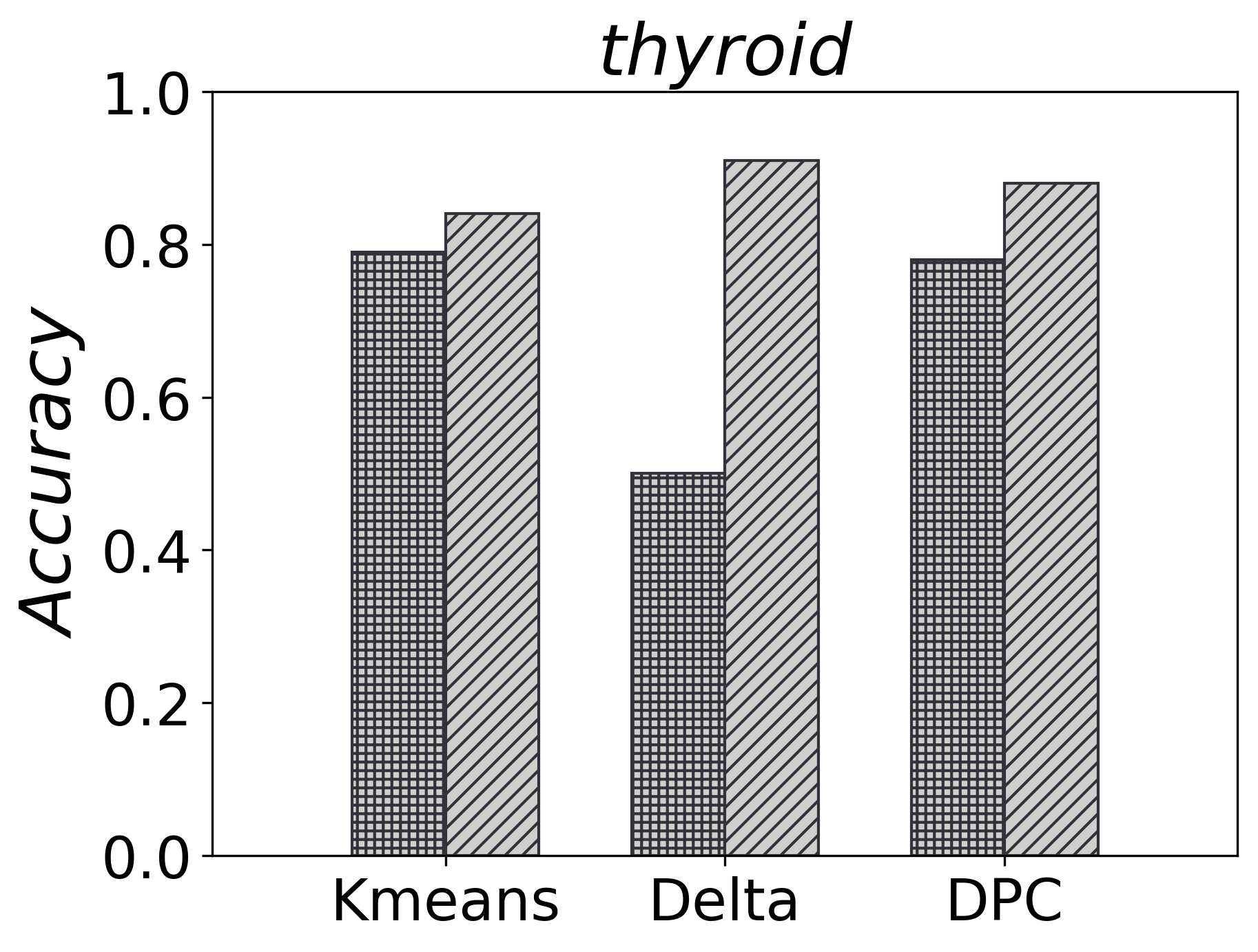}
  \quad
  \includegraphics[width=1.2in]{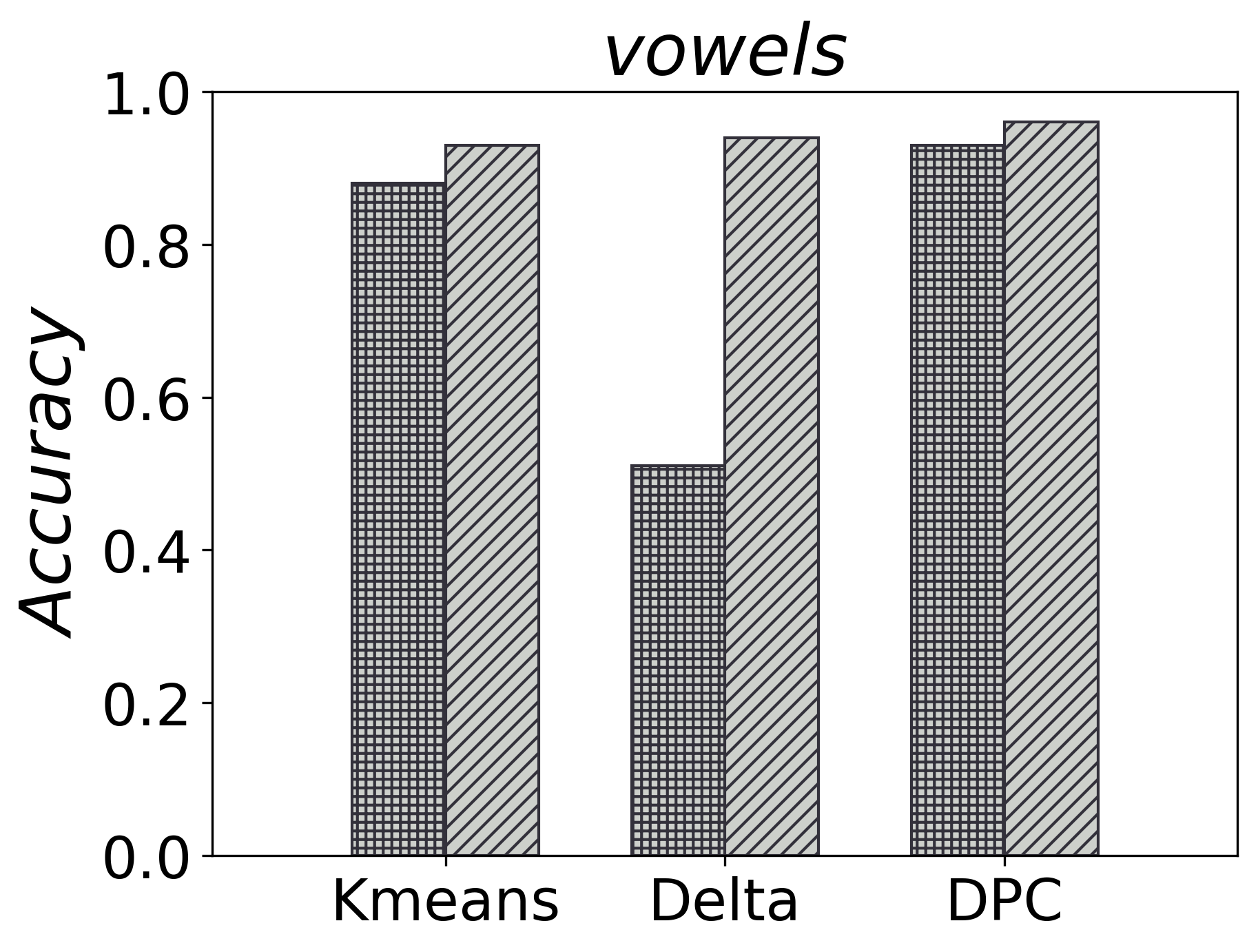}
  \quad
  \includegraphics[width=1.2in]{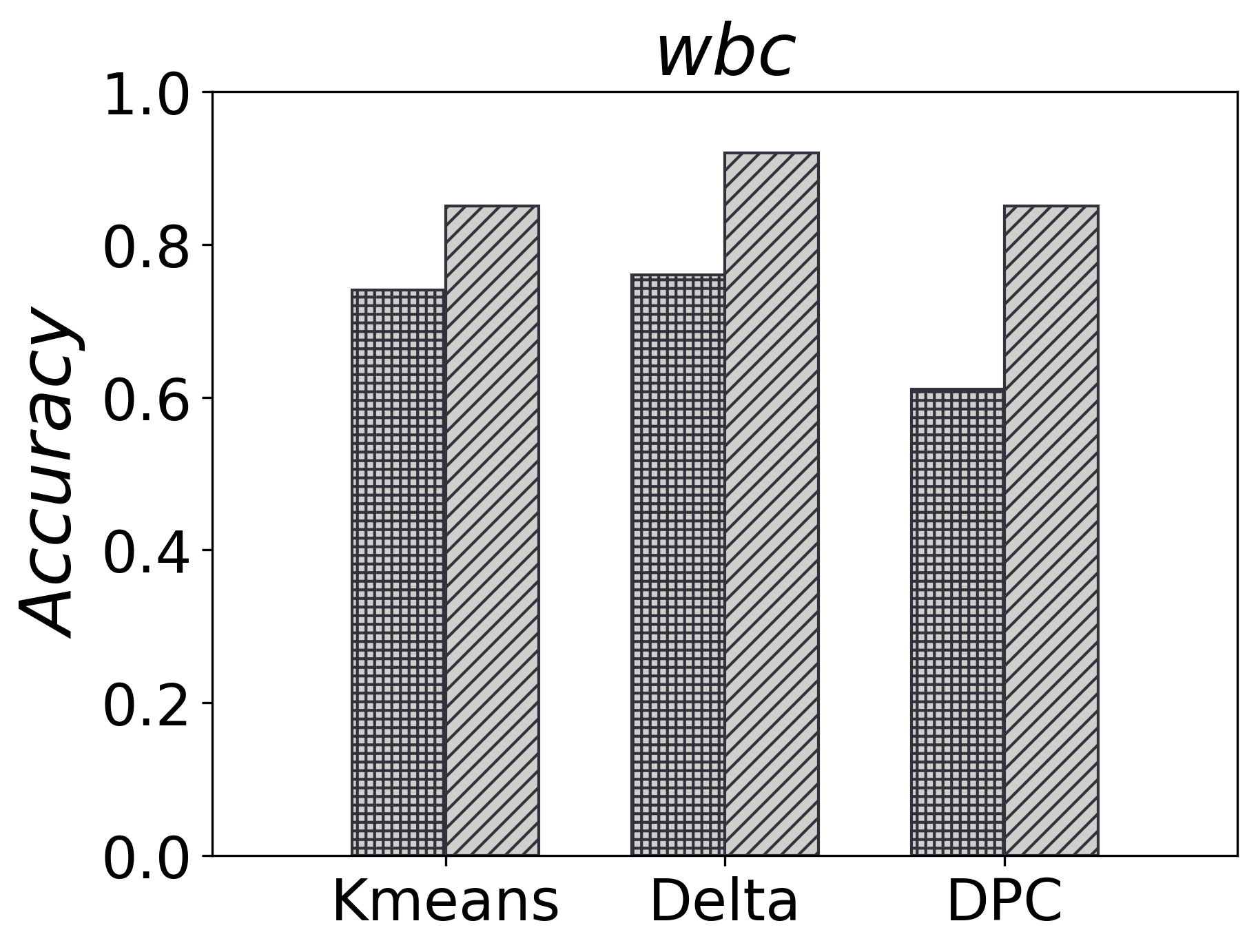}
  \caption{ODAR-NoComp is the ODAR that is not equipped with the component clustering strategy. The accuracies of the clustering algorithms with the help of ODAR are all higher than those of the clustering algorithms with the help of ODAR-NoComp. Therefore, the component clustering strategy is necessary for ODAR.}
  \label{fig:strategy}
\end{figure*}

\subsection{Ablation Experiments}
In this subsection, we tested the necessity of the shrinking method and the component clustering strategy for ODAR.

\subsubsection{Necessity of the Component Clustering  Strategy}
\label{sec:component clystering strategy}
To verify the necessity of the component clustering strategy, we compare ODAR with ODAR-NoComp. ODAR-NoComp is the ODAR that is not equipped with the component clustering strategy. Figure \ref{fig:strategy} shows the accuracies of the three clustering algorithms with the help of ODAR and ODAR-NoComp on 10 real-world datasets. The experimental results show that the accuracies of the clustering algorithms with the help of ODAR are all higher than those of the clustering algorithms with the help of ODAR-NoComp. For example, on \emph{vowels} dataset, the accuracy of the \emph{delta clustering} with the help of ODAR-NoComp is only 0.51, but the accuracy of the \emph{delta clustering} with the help of ODAR is as high as 0.94; on \emph{lympho} dataset, the accuracy of the \emph{DPC} with the help of ODAR-NoComp is only 0.66, but the accuracy of the \emph{DPC} with the help of ODAR is as high as 0.96; on \emph{letter} dataset, the accuracy of the \emph{kmeans} with the help of ODAR-NoComp is 0.76, but the accuracy of the \emph{kmeans} with the help of ODAR is as high as 0.86. In terms of average accuracy, the clustering algorithms with the help of ODAR is 0.84, and the clustering algorithms with the help of ODAR-NoComp is 0.73. The component clustering strategy improves the outlier detection accuracy by 15\%. Therefore, the component clustering strategy is necessary for ODAR.

\begin{figure}[h]
  \centering
  \includegraphics[width=3in]{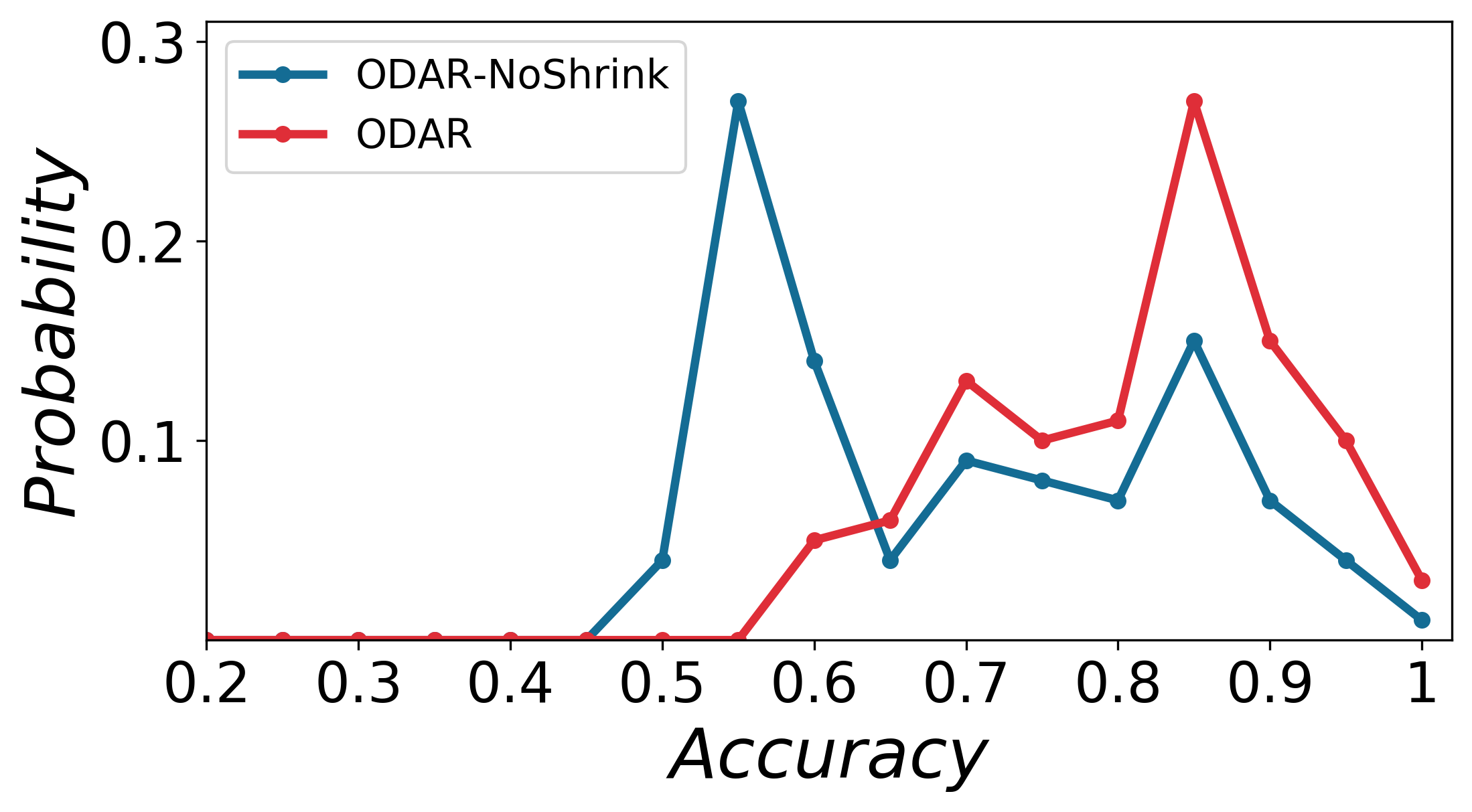}
  \caption{ODAR-NoShrink is the ODAR that is not equipped with the shrinking method. The \emph{delta clustering} with the help of ODAR has a higher probability of obtaining higher accuracy than the \emph{delta clustering} with the help of ODAR-NoShrink. Therefore, the shrinking method is necessary for ODAR.}
  \label{fig:shrink}
\end{figure}

\subsubsection{Necessity of the Shrinking Method}
\label{sec:the shrink method}
\emph{Delta clustering} is a parameter-sensitive clustering algorithm, different input parameter values may produce completely opposite clustering results. Here, we verify whether the shrinking method can improve the adaptability of ODAR to \emph{delta clustering}. We refer to the ODAR that is not equipped with the shrinking method as ODAR-NoShrink. We compare ODAR with ODAR-NoShrink on 10 real-world datasets. Specifically, within the same parameter interval (containing 100 parameter values), we execute the \emph{delta clustering} with the help of ODAR and the \emph{delta clustering} with the help of ODAR-NoShrink, respectively. Next, we count the probability of different accuracies, as shown in Figure \ref{fig:shrink}. The experimental results show that the \emph{delta clustering} with the help of ODAR has a higher probability of obtaining higher accuracy than the \emph{delta clustering} with the help of ODAR-NoShrink. The reason ODAR excels is that the shrinking method facilitates \emph{delta clustering} to identify the outlier-cluster in the ODAR space, thus ensuring the stability of ODAR. Therefore, the shrinking method is necessary for ODAR.

\subsection{Robustness Experiments}
Outlier detection is easily affected by some factors, such as the distribution and number of outliers, or unbalanced density. In this subsection, we tested the robustness of ODAR in different scenarios, and verified whether these robustness will change with different clustering algorithms.

\begin{figure*}[h]
  \centering
  \includegraphics[width=1.3in]{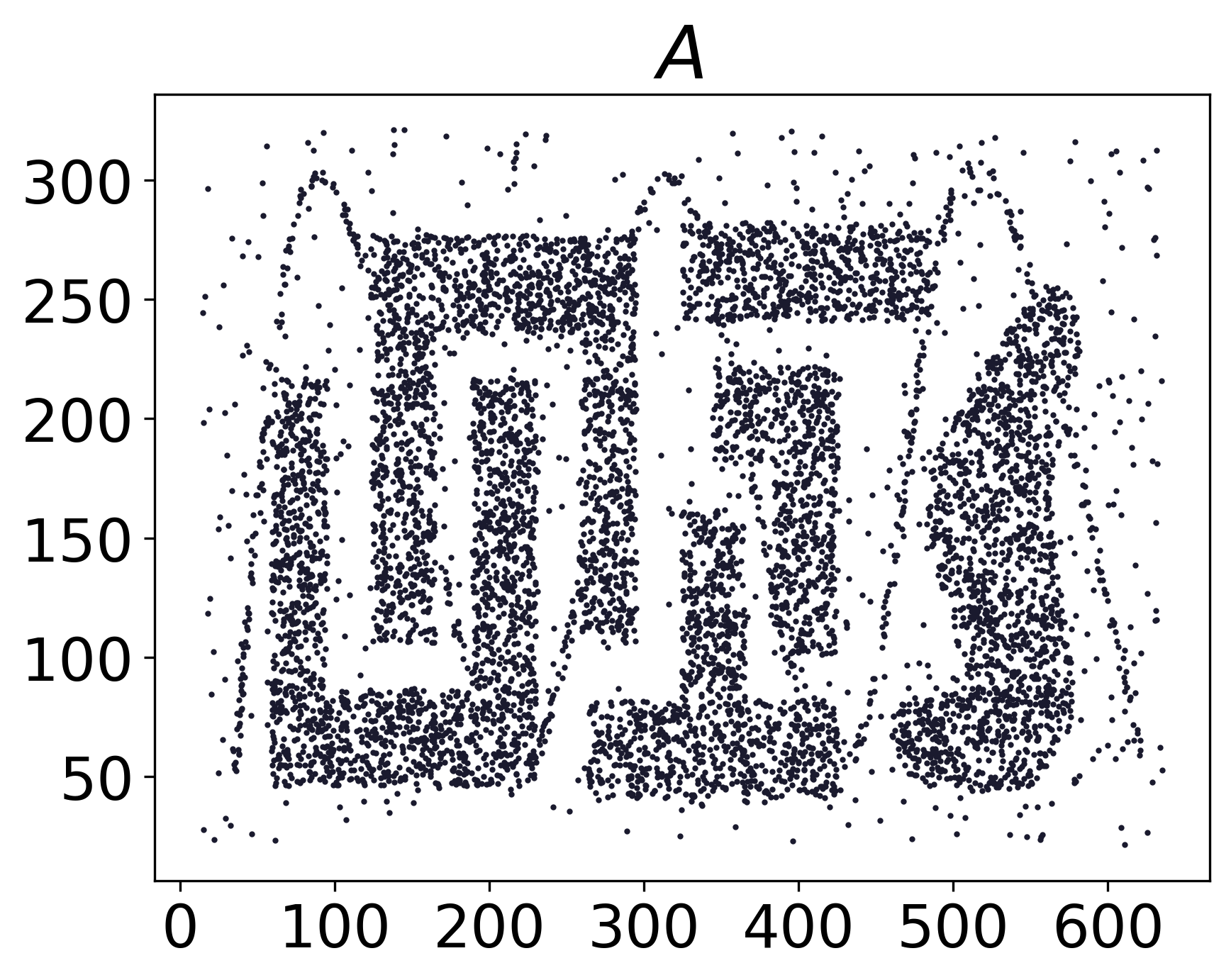}
  \quad\quad
  \includegraphics[width=1.3in]{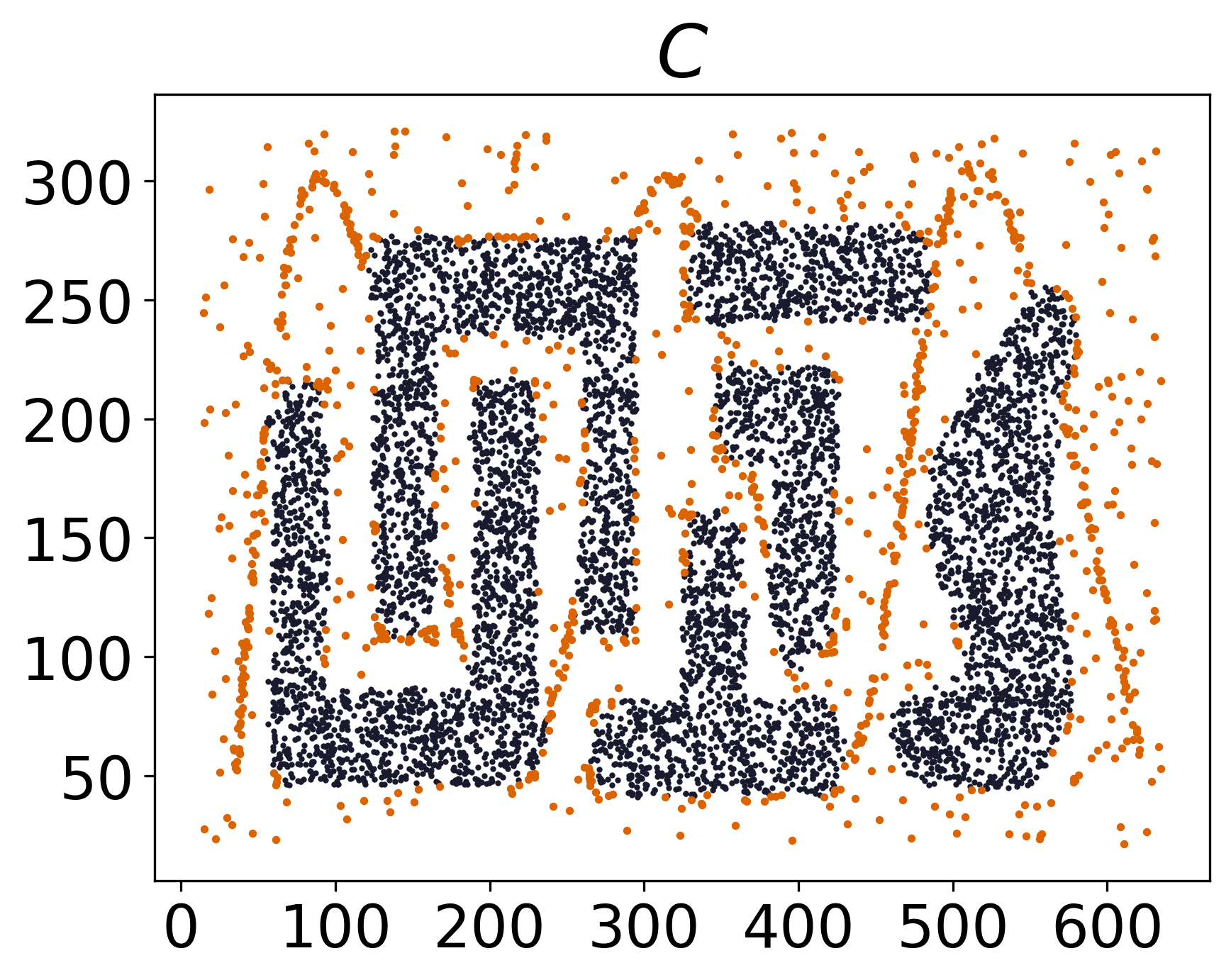}
  \quad\quad
  \includegraphics[width=1.3in]{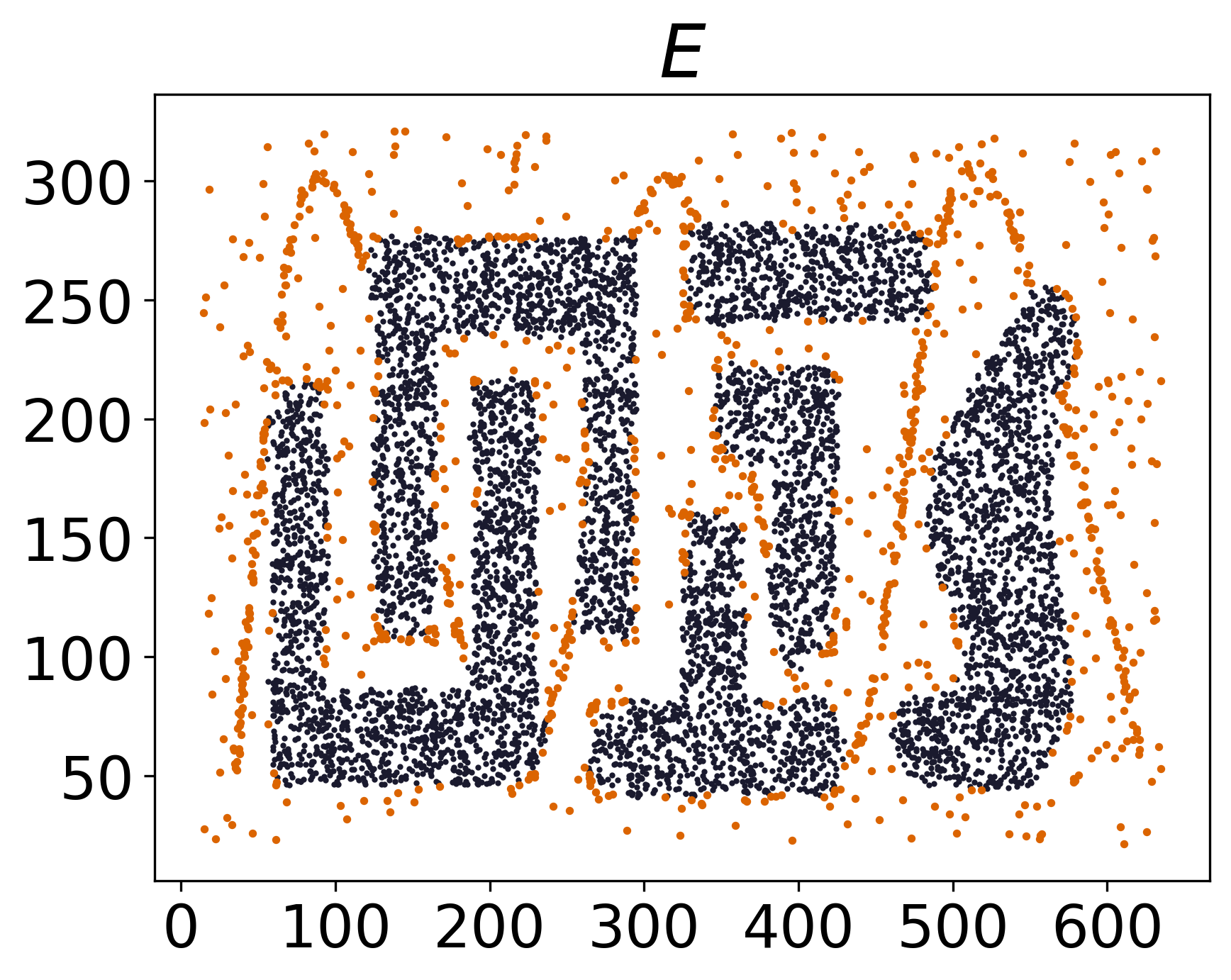}
  \quad\quad
  \includegraphics[width=1.3in]{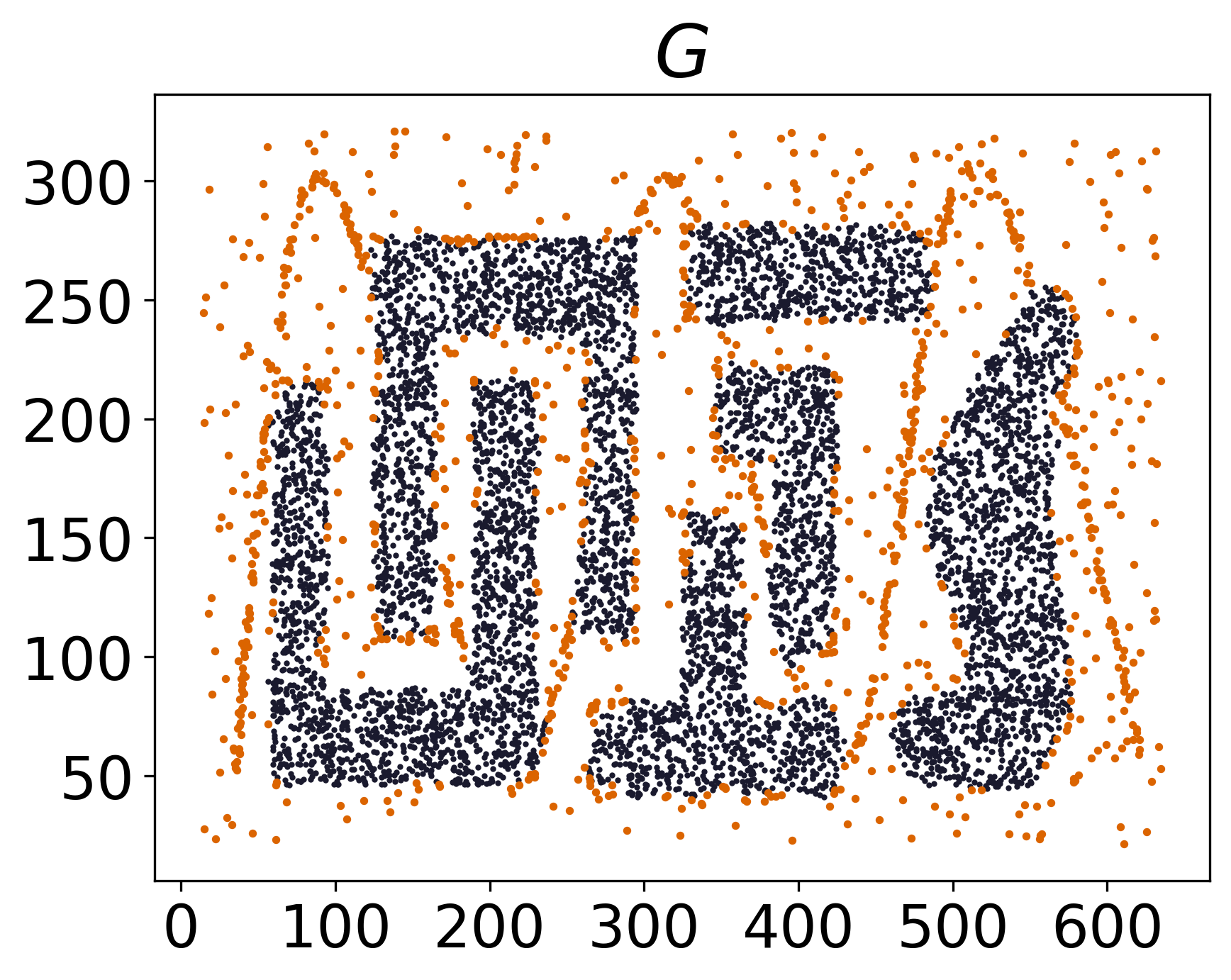}\\
  \includegraphics[width=1.3in]{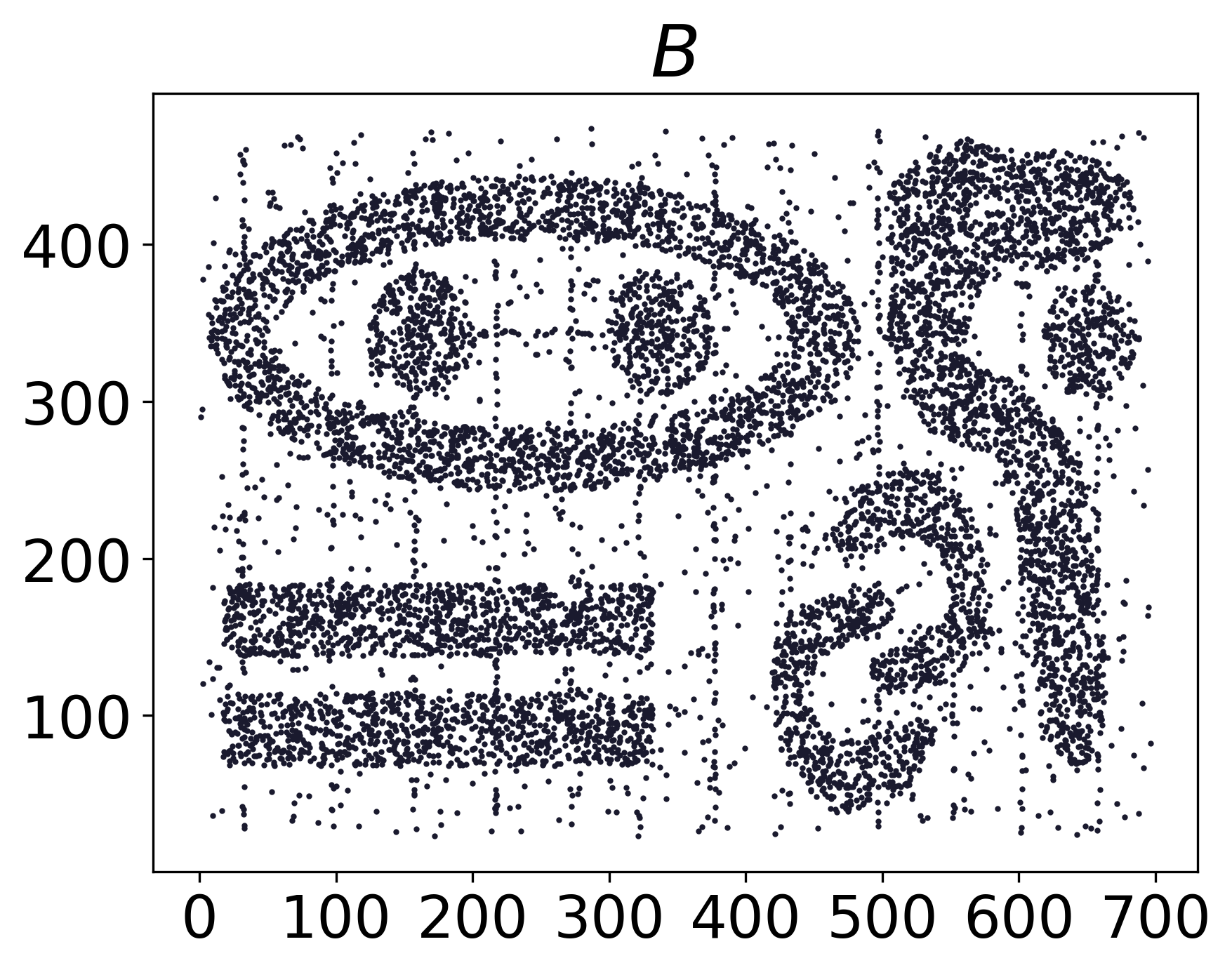}
  \quad\quad
  \includegraphics[width=1.3in]{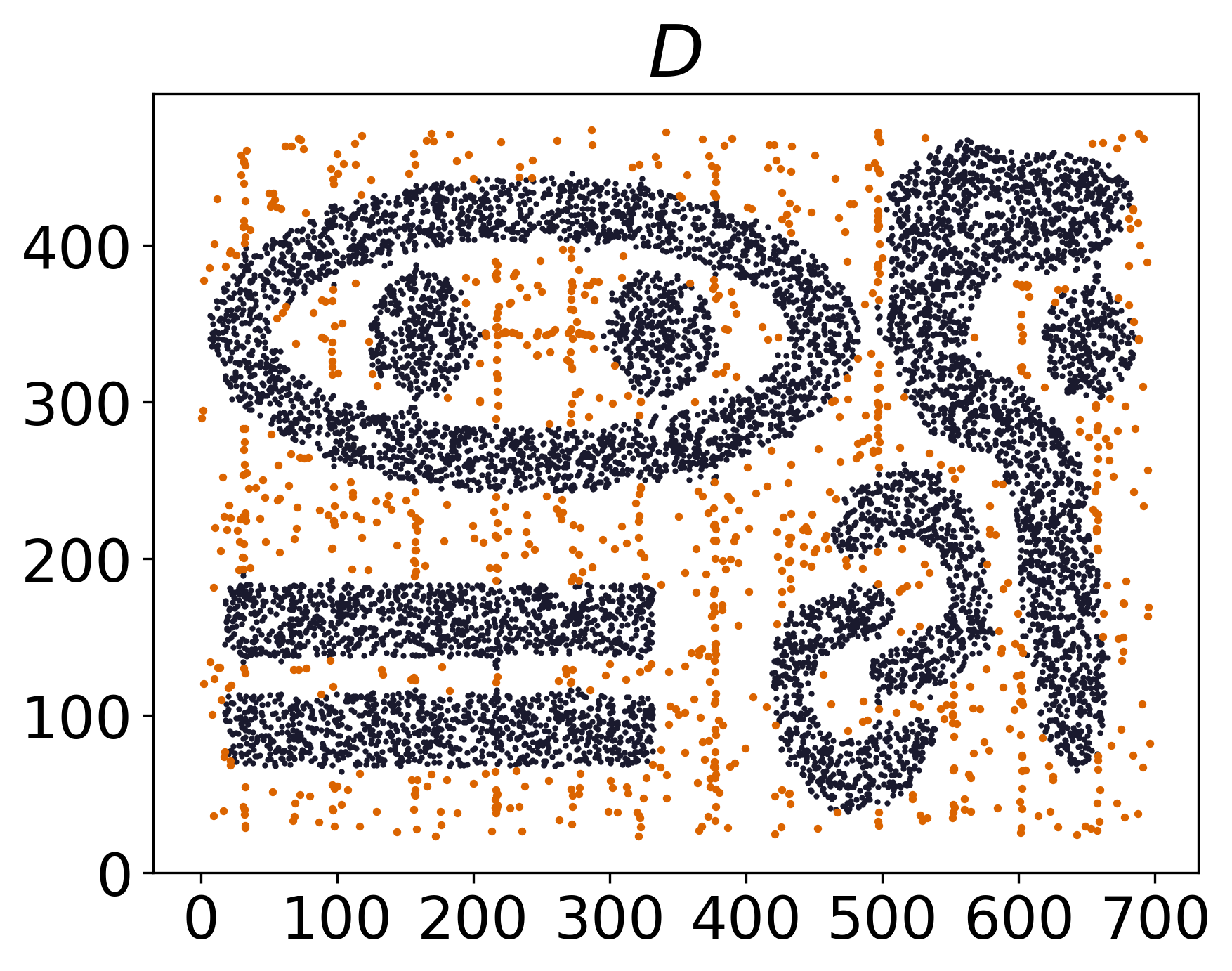}
  \quad\quad
  \includegraphics[width=1.3in]{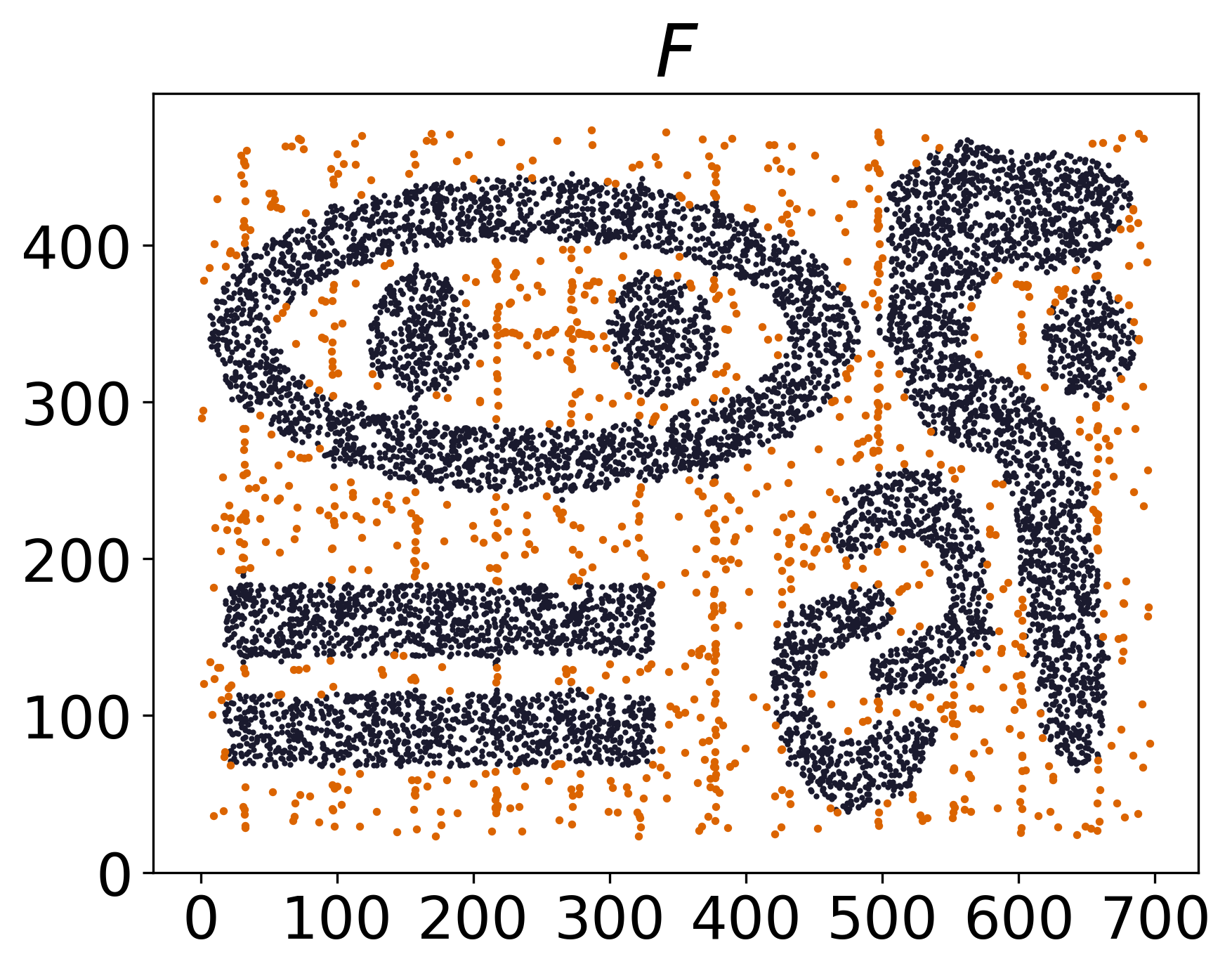}
  \quad\quad
  \includegraphics[width=1.3in]{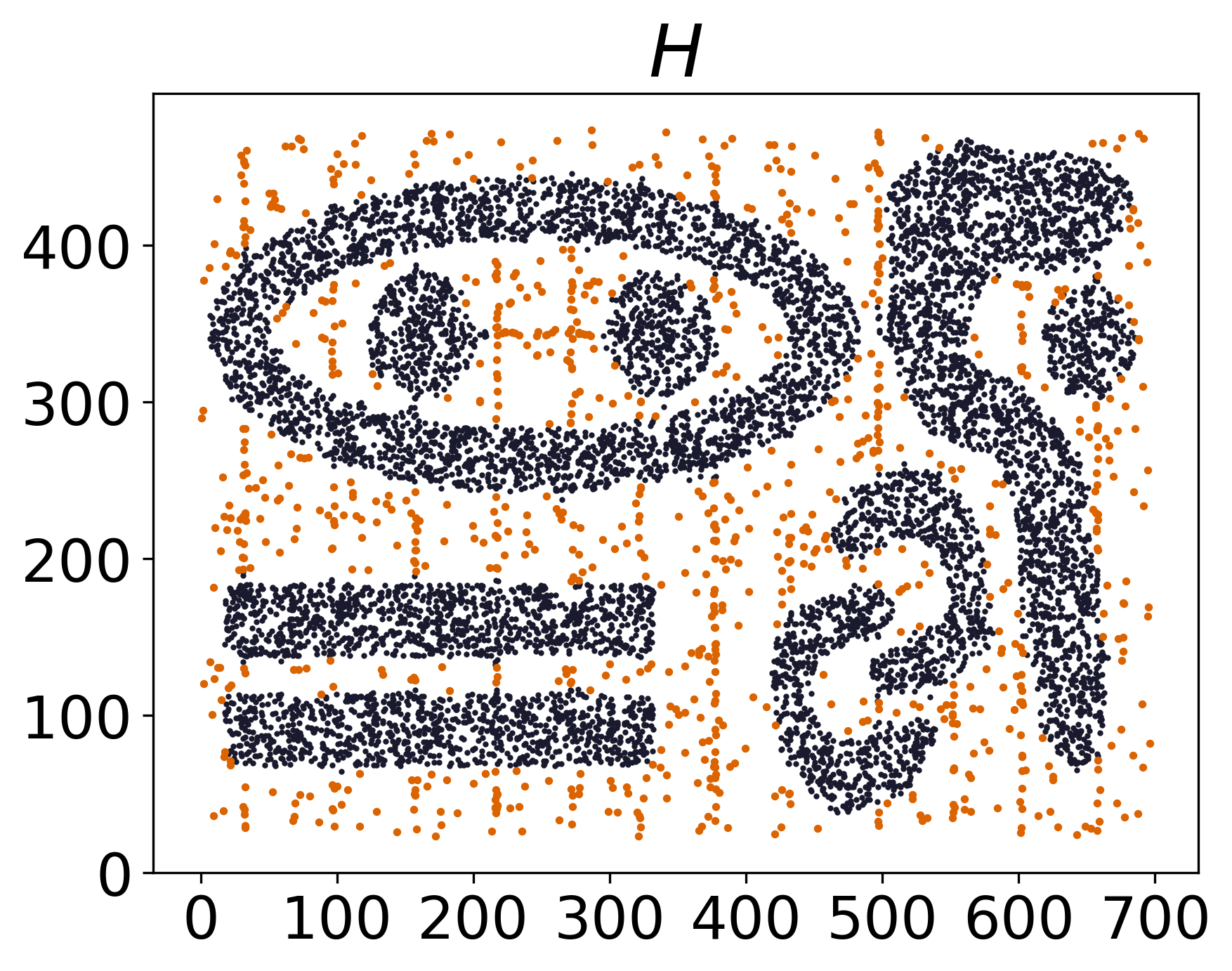}
  \caption{With the help of ODAR, the outliers detected by \emph{delta clustering} ($C,D$), \emph{kmeans} ($E,F$), and \emph{DPC} ($G,H$) for \emph{t4.8k} ($A$) and \emph{t7.10k} ($B$) datasets. The three clustering algorithms successfully detect ‘sin’ outliers and ‘vertical lines’ outliers interspersed between normal objects. Evidently, ODAR's performance is reliable in detecting the outliers with complex distribution.}
  \label{fig:shape}
\end{figure*}

\begin{figure*}[h]
  \centering
  \includegraphics[width=1.2in]{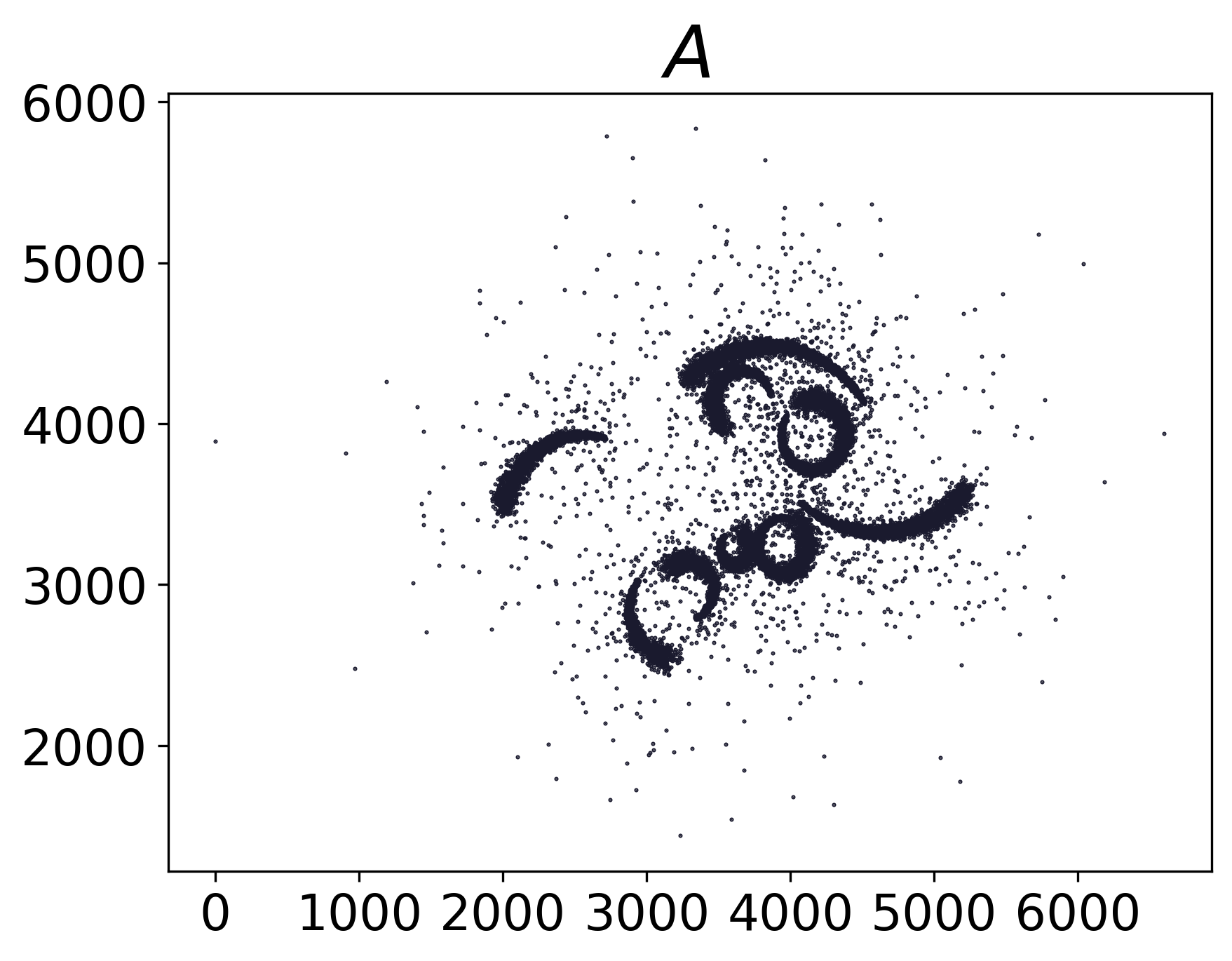}
  \quad\quad
  \includegraphics[width=1.2in]{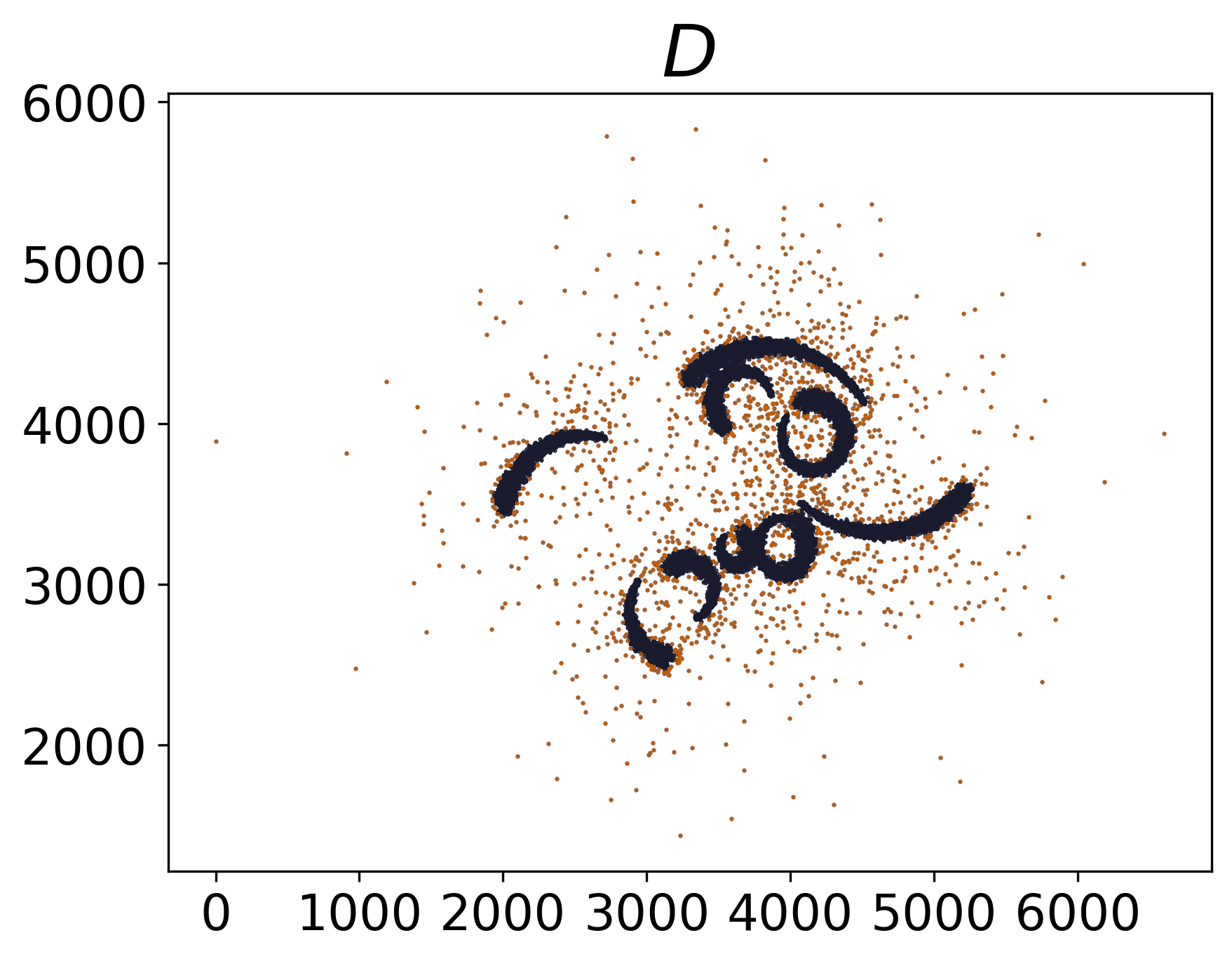}
  \quad\quad
  \includegraphics[width=1.2in]{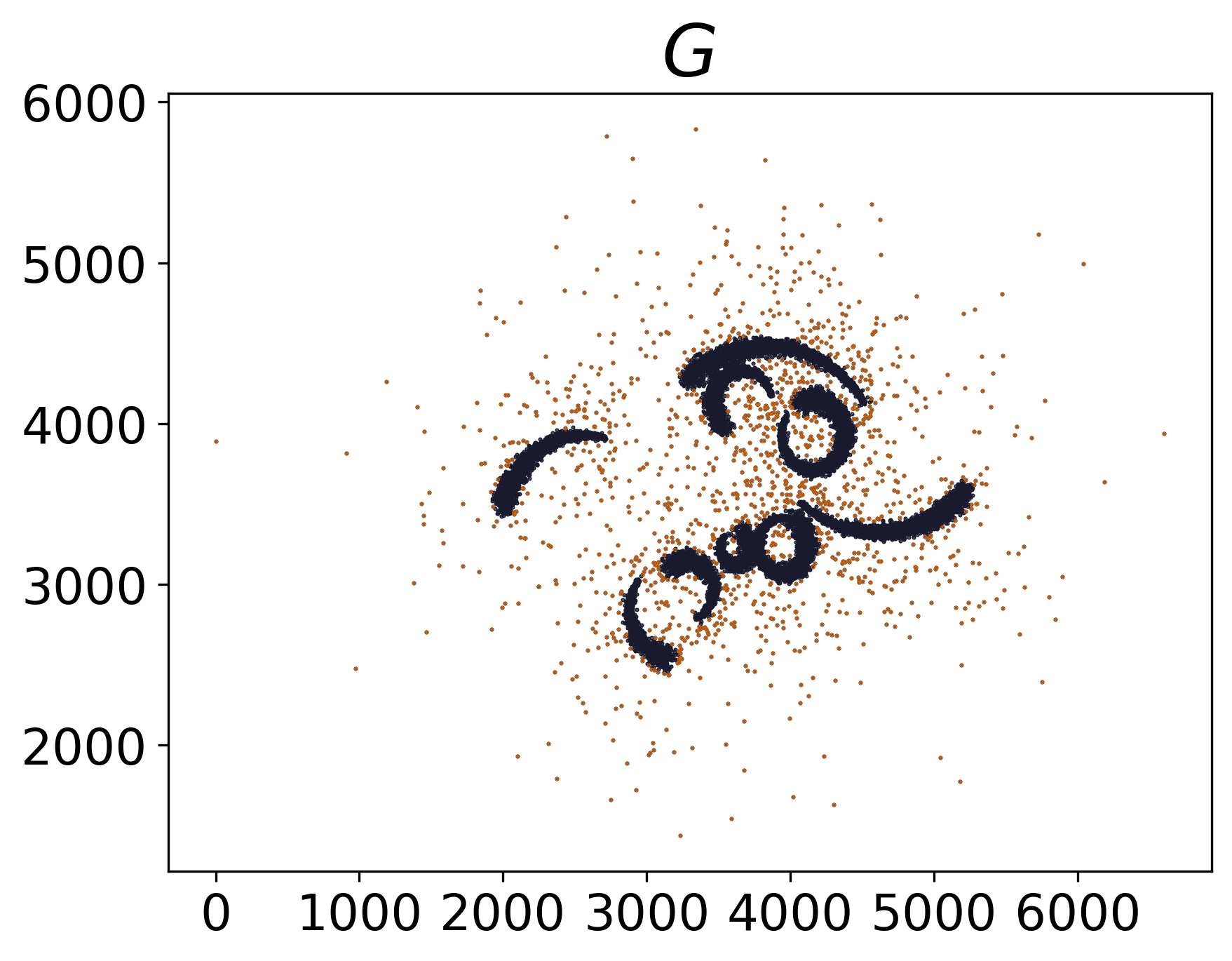}
  \quad\quad
  \includegraphics[width=1.2in]{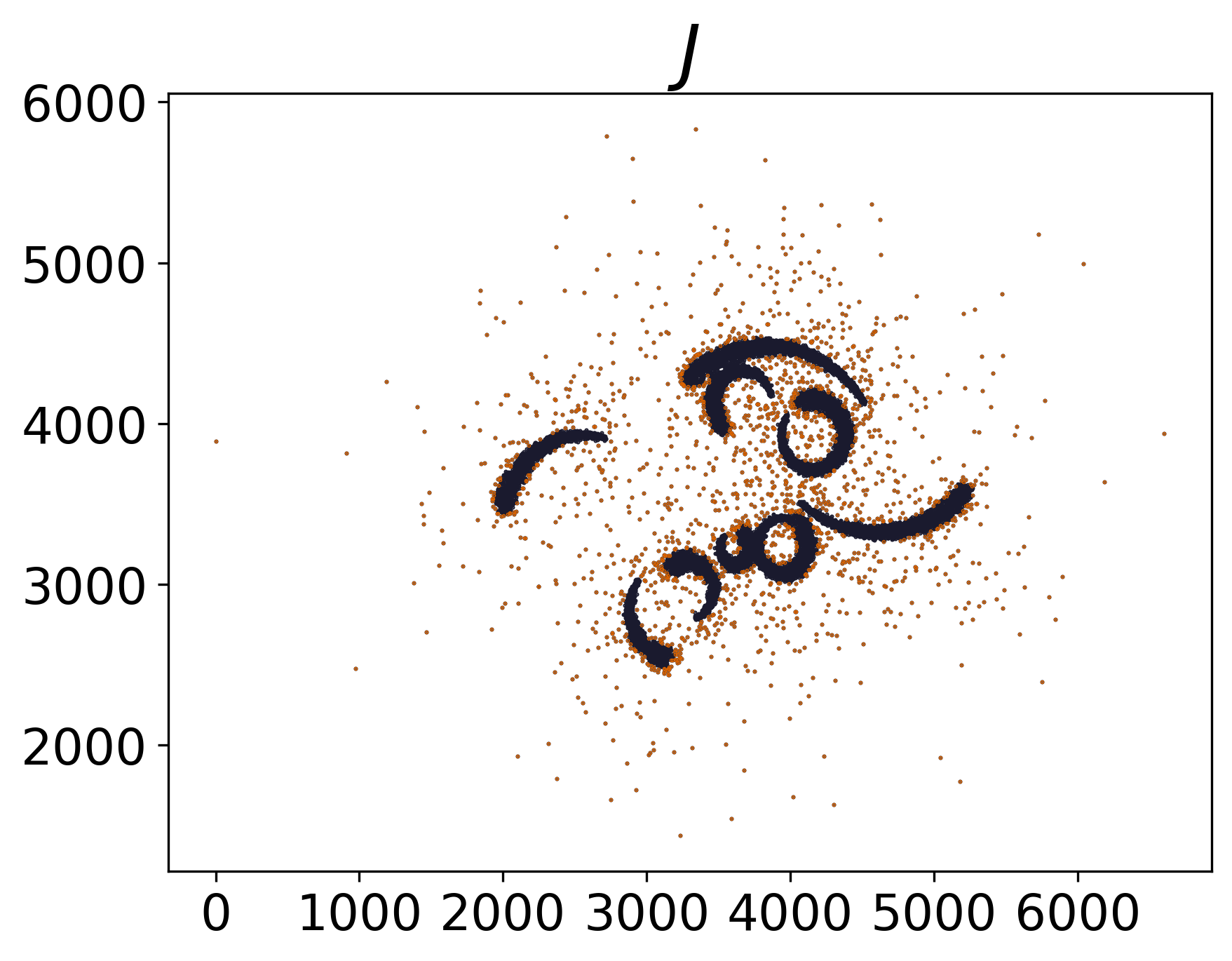}\\
  \includegraphics[width=1.2in]{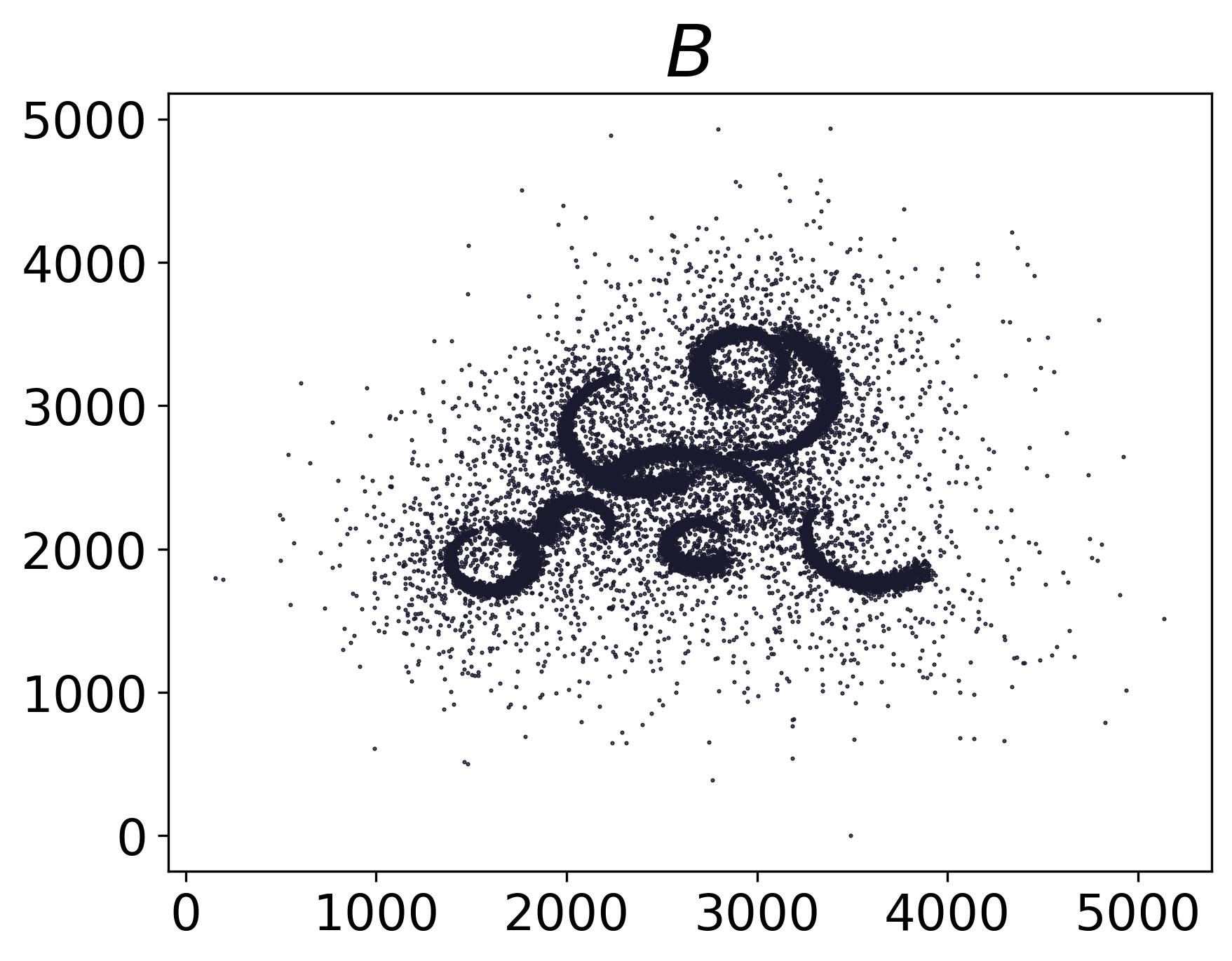}
  \quad\quad
  \includegraphics[width=1.2in]{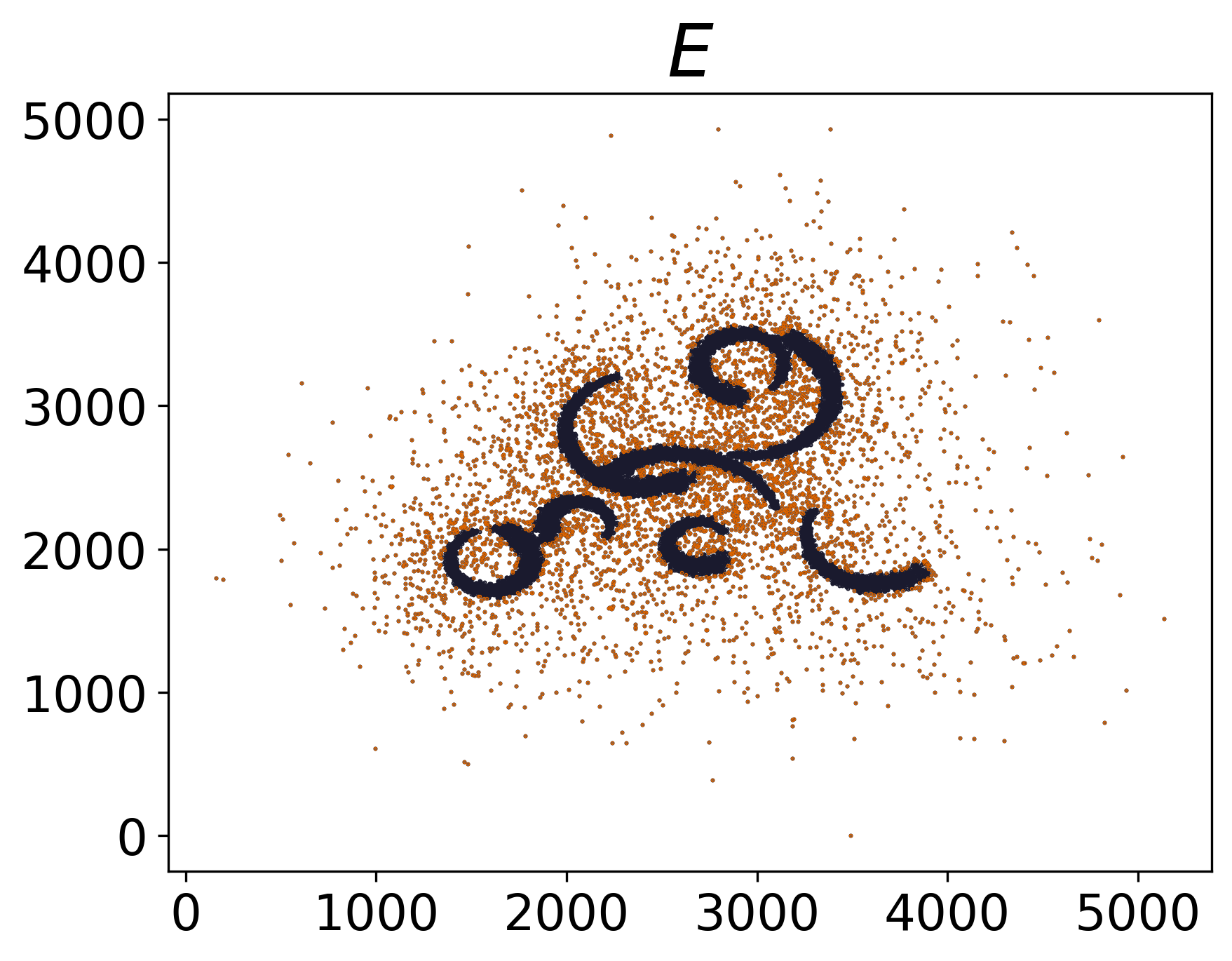}
  \quad\quad
  \includegraphics[width=1.2in]{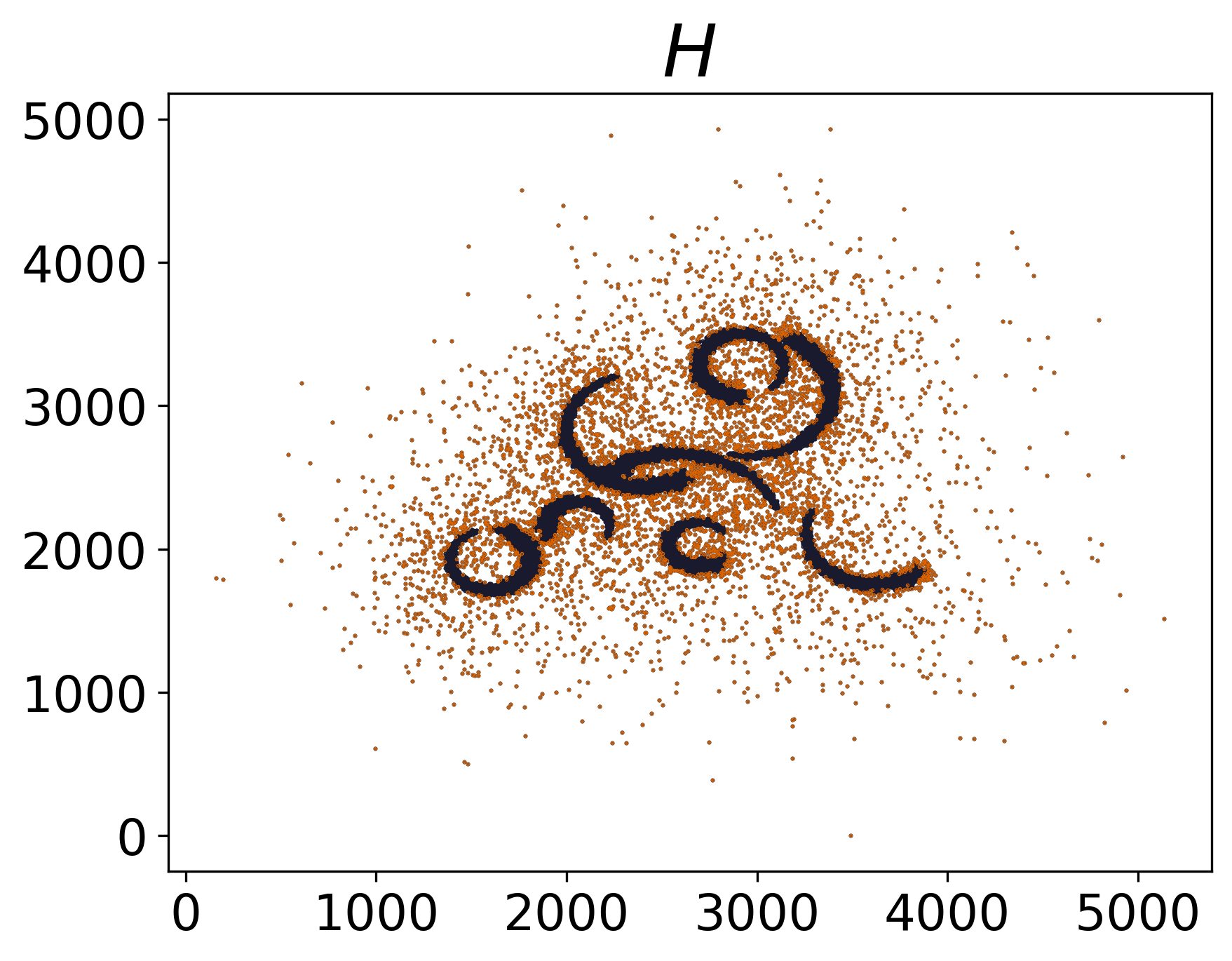}
  \quad\quad
  \includegraphics[width=1.2in]{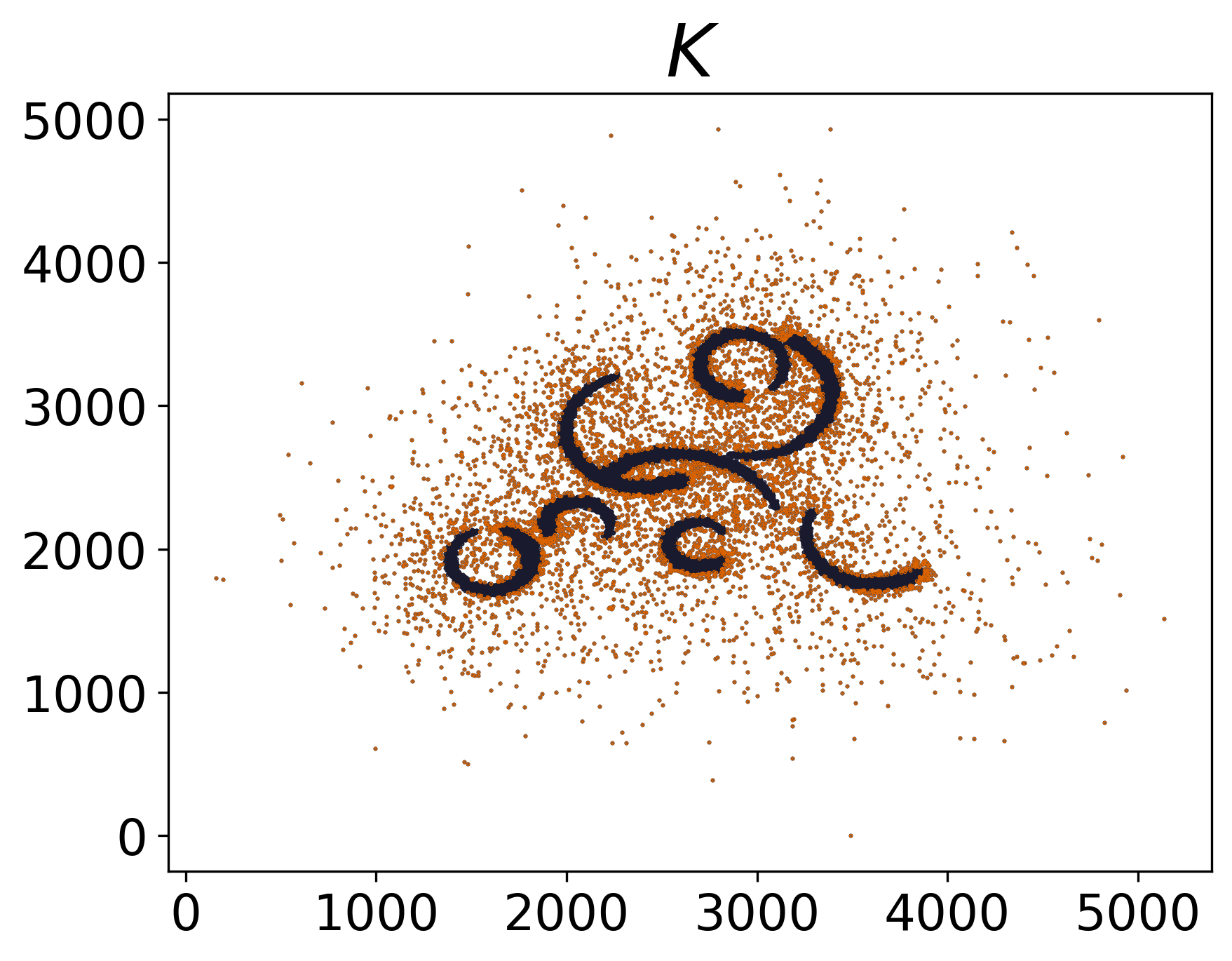}\\
  \includegraphics[width=1.2in]{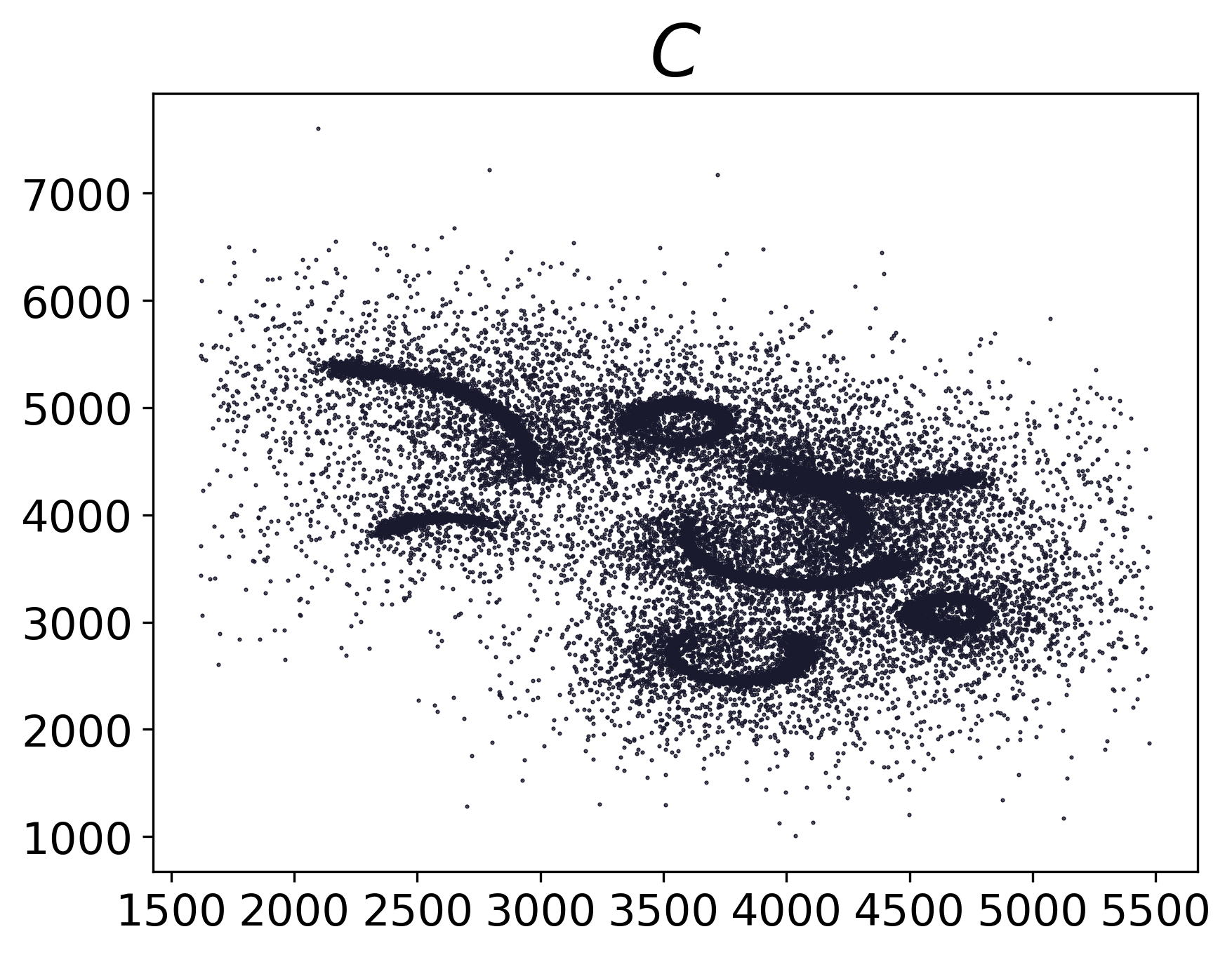}
  \quad\quad
  \includegraphics[width=1.2in]{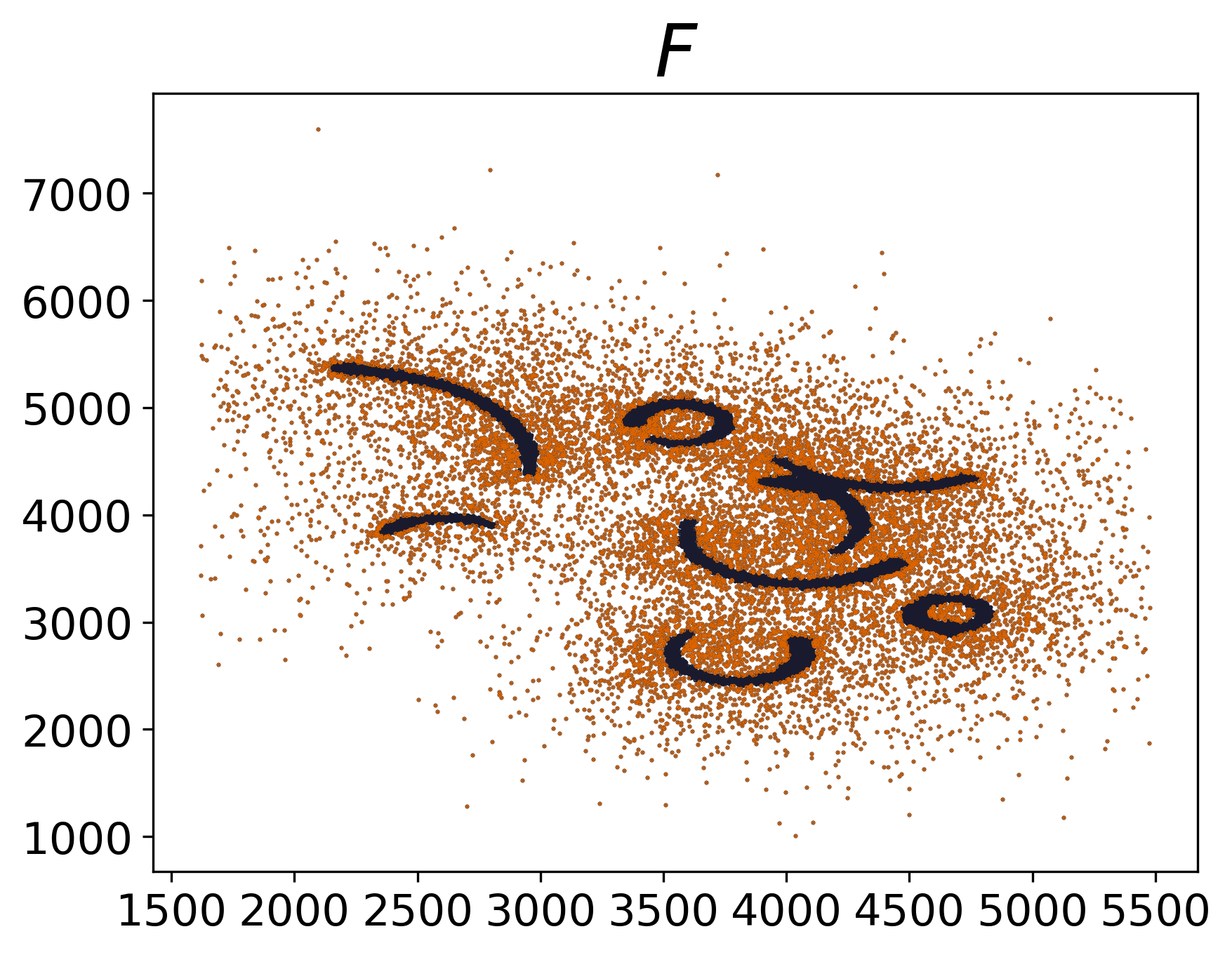}
  \quad\quad
  \includegraphics[width=1.2in]{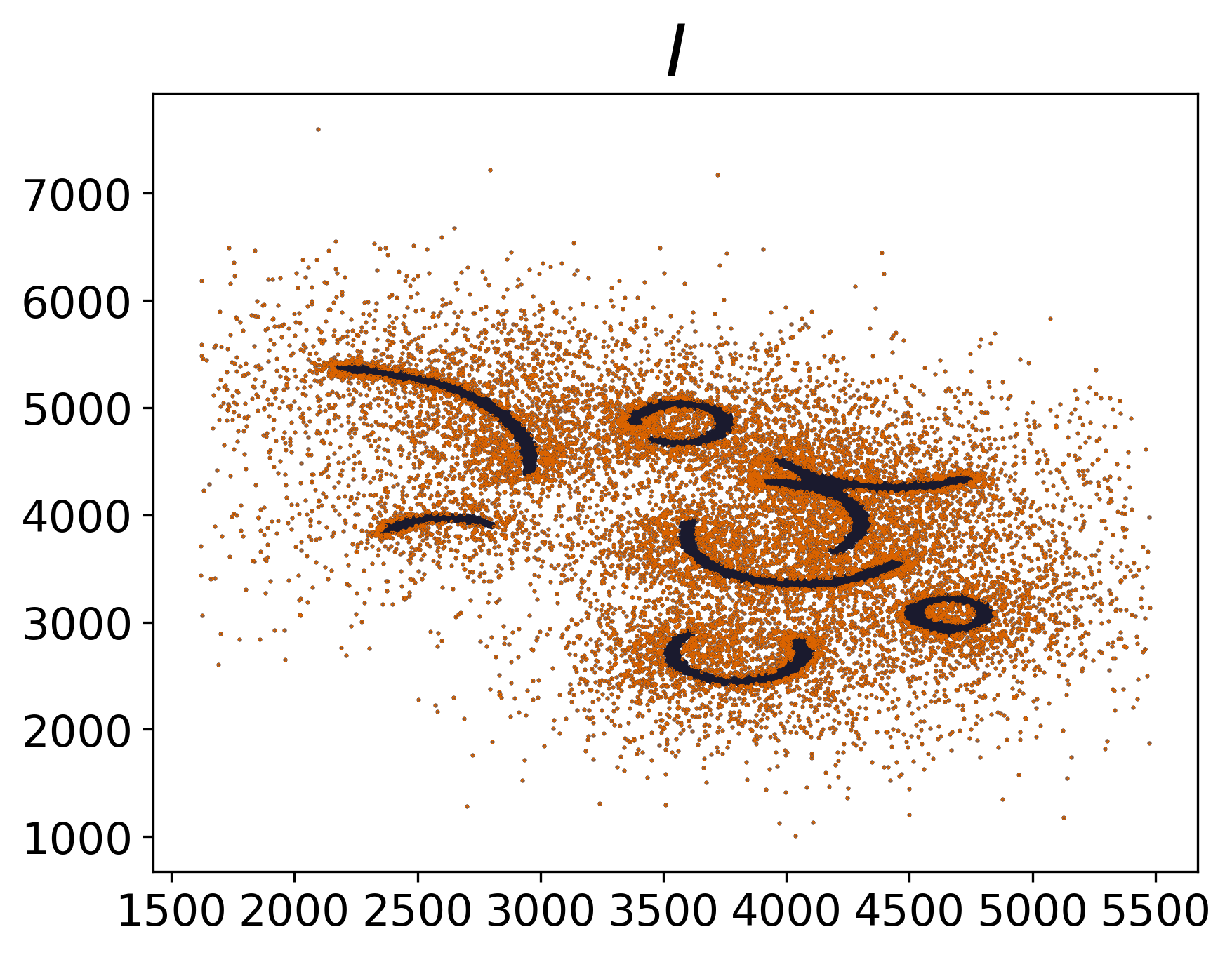}
  \quad\quad
  \includegraphics[width=1.2in]{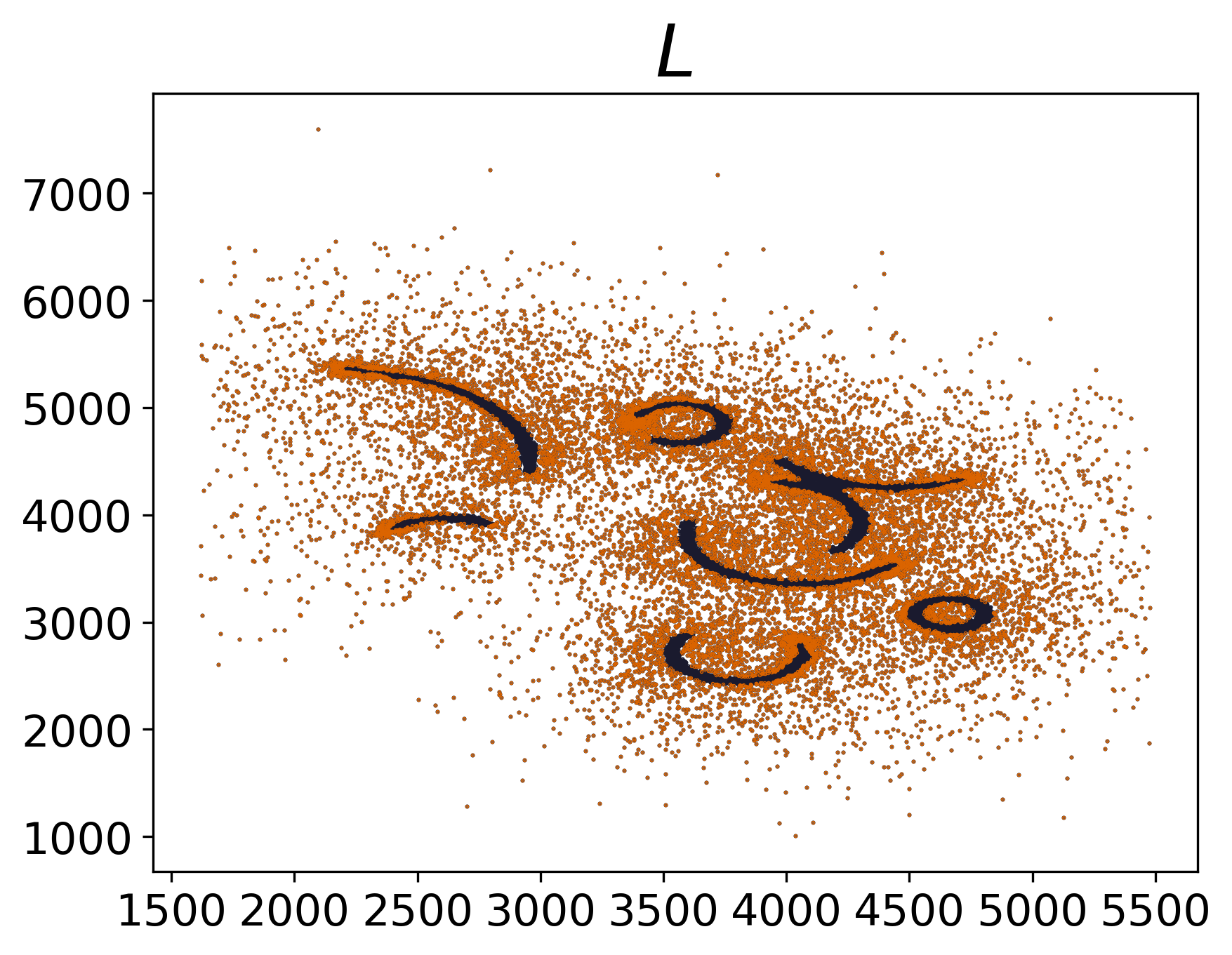}
  
  \caption{With the help of ODAR, the outliers detected by \emph{delta clustering} ($D,E,F$), \emph{kmeans} ($G,H,I$), and \emph{DPC} ($J,K,L$) for \emph{worm-num-least} ($A$), \emph{worm-num-medium} ($B$), and \emph{worm-num-most} ($C$) datasets. From \emph{worm-num-least} to \emph{worm-num-most}, the number of outliers is increasing. The shape of worm clusters in \emph{worm-num-most} cannot be easily recognized with the naked eye. After detecting all outliers, the shape of these worm clusters becomes very clear.}
  \label{fig:num}
\end{figure*}

\subsubsection{Robustness to Distribution}
Sometimes, outliers appear in complex distribution. They may gather somewhere with large density or be nested within normal objects, which makes it easy for detection methods to confuse between outliers and normal objects. Here, we select two datasets, \emph{t}4.7\emph{k} and \emph{t}7.10\emph{k}, to verify the robustness of ODAR to the outliers with complex distribution. In the two datasets, normal objects are the objects inside clusters, and outliers are the objects outside clusters. The first column of Figure \ref{fig:shape} shows their distribution. In \emph{t}4.7\emph{k}, some outliers traverse normal objects in a path like the $sin$ function. In \emph{t}7.10\emph{k}, some outliers traverse normal objects in a path resembling vertical lines. Columns second to fourth of Figure \ref{fig:shape} show detected outliers by \emph{delta clustering}, \emph{kmeans} and \emph{DPC} with the help of ODAR, in which outliers are marked in orange. Although the '$sin$' outliers and 'vertical lines' outliers are connected to clusters and have larger densities than sparse-outliers, these clustering algorithms are not confused and correctly detect them with the help of ODAR. For the objects inside clusters, none are incorrectly marked as outliers. Evidently, ODAR's performance is reliable in detecting the outliers with complex distribution.

\begin{figure*}[h]
  \centering
  \includegraphics[width=1.2in]{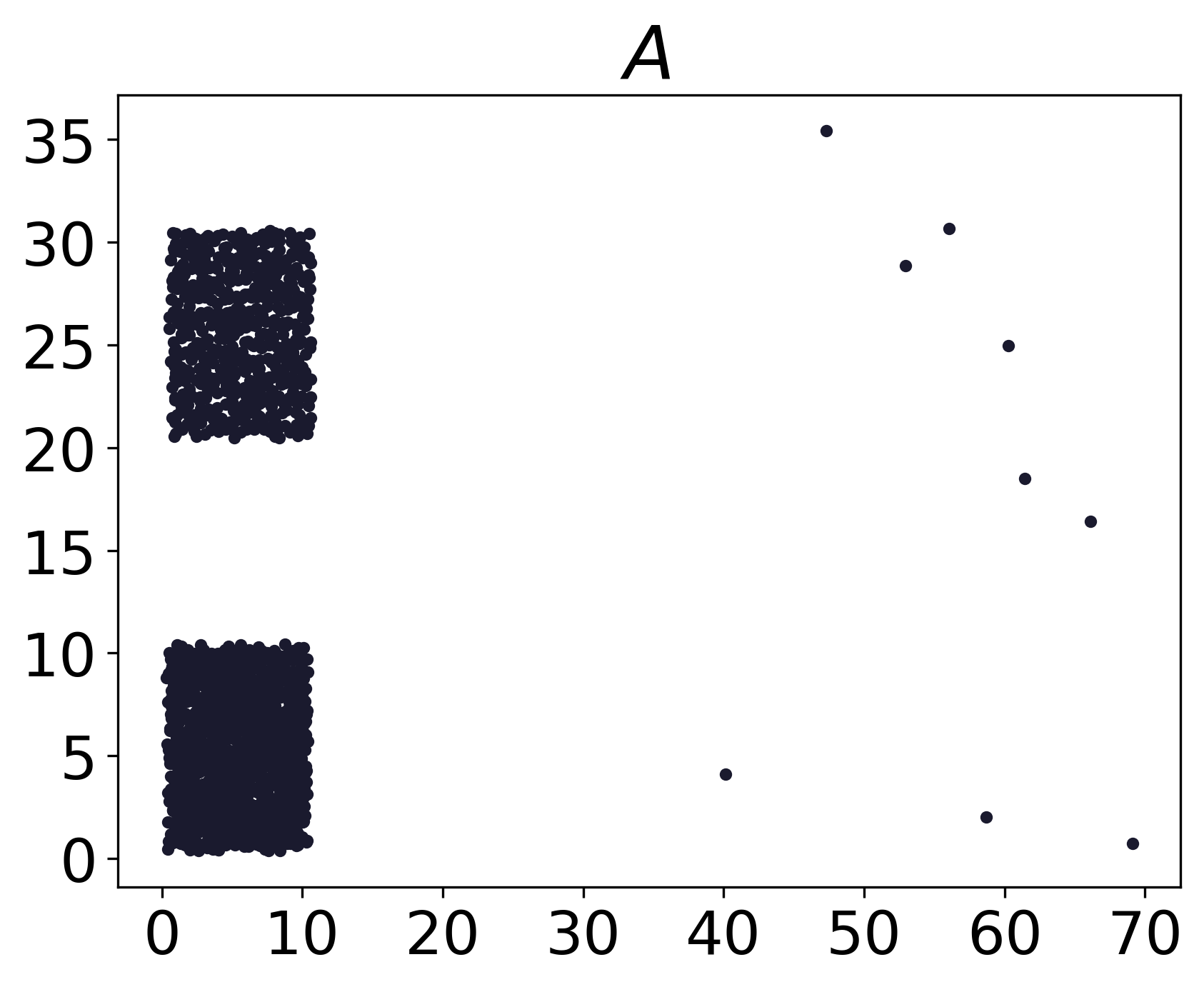}
  \quad\quad
  \includegraphics[width=1.2in]{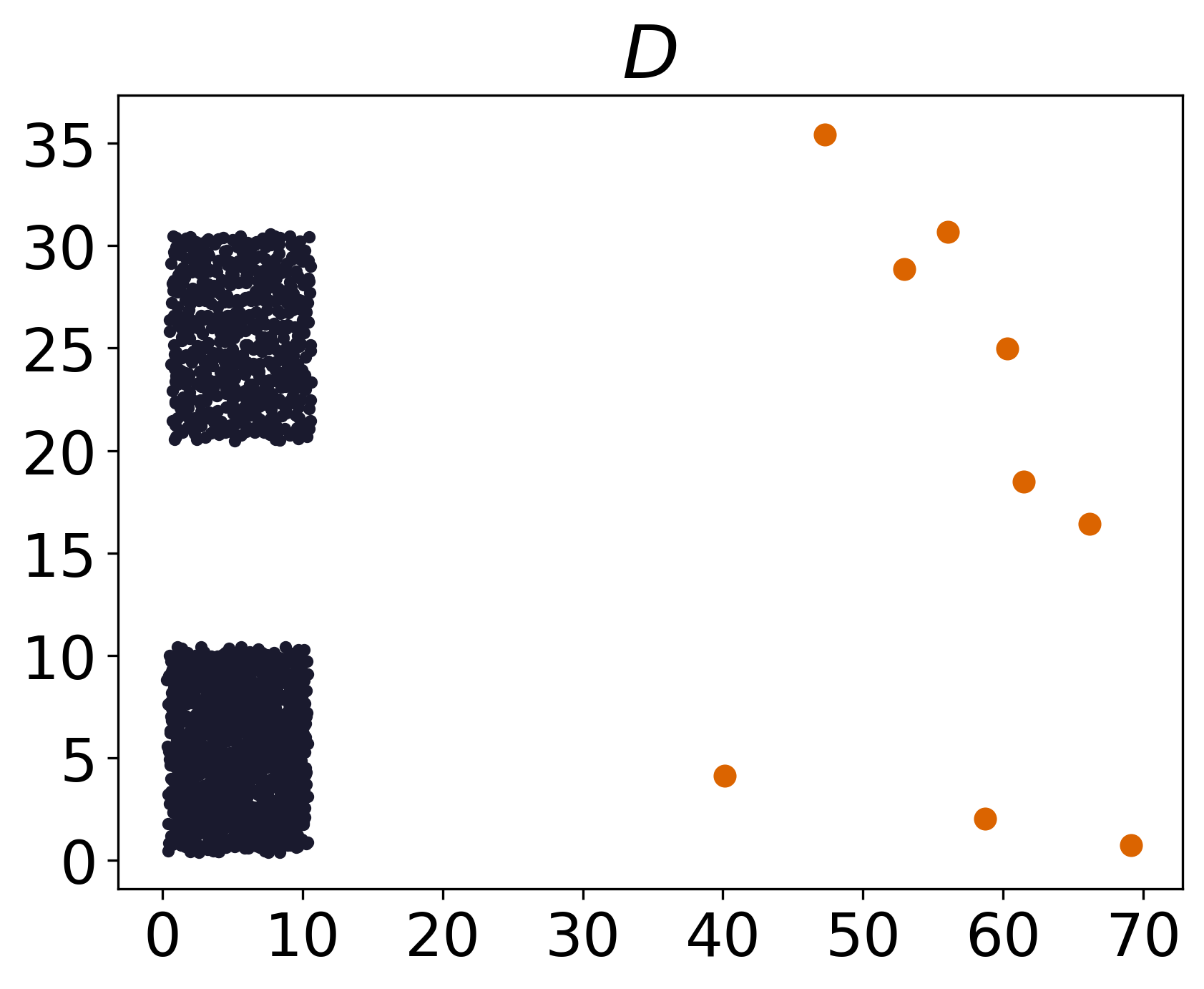}
  \quad\quad
  \includegraphics[width=1.2in]{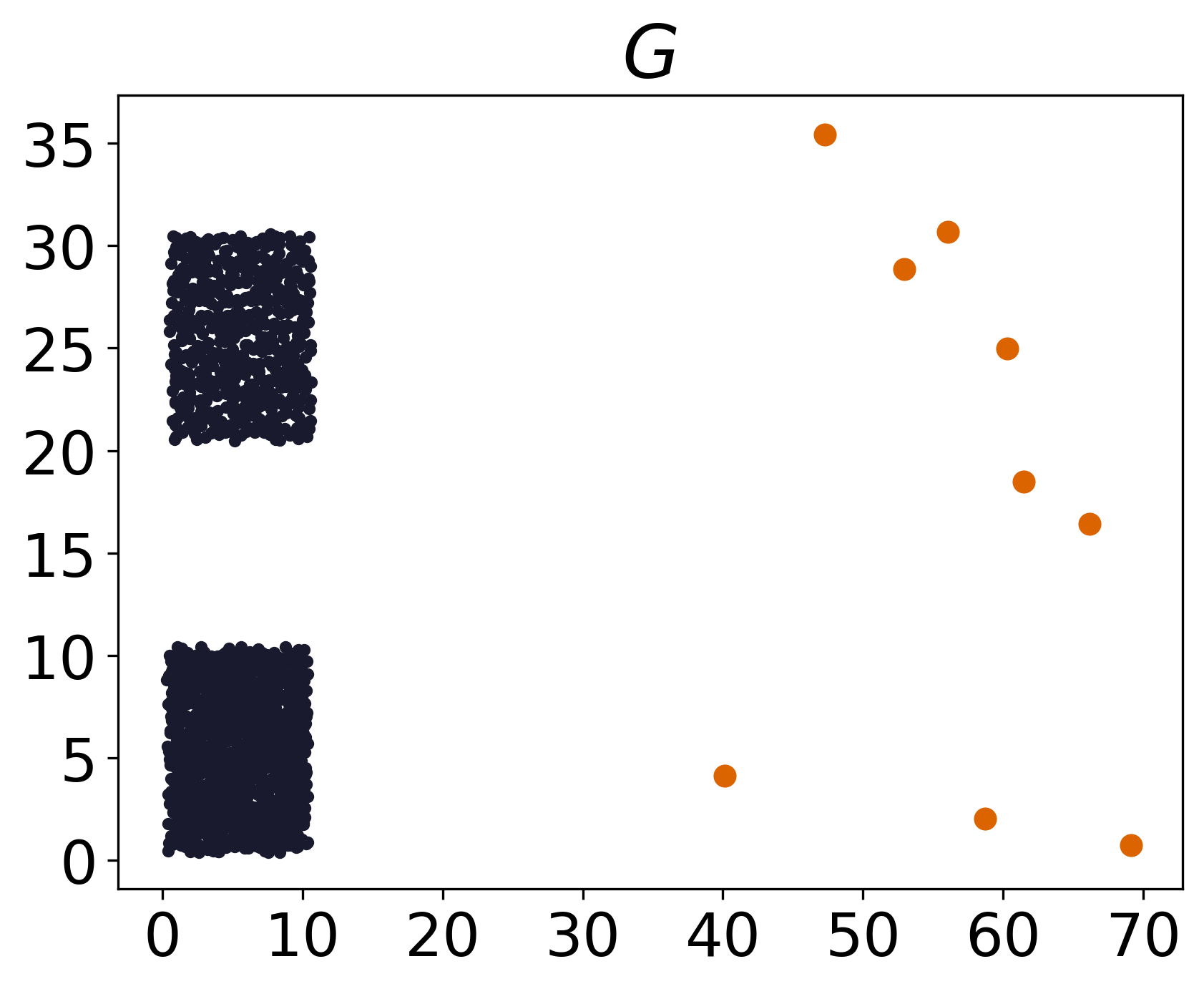}
  \quad\quad
  \includegraphics[width=1.2in]{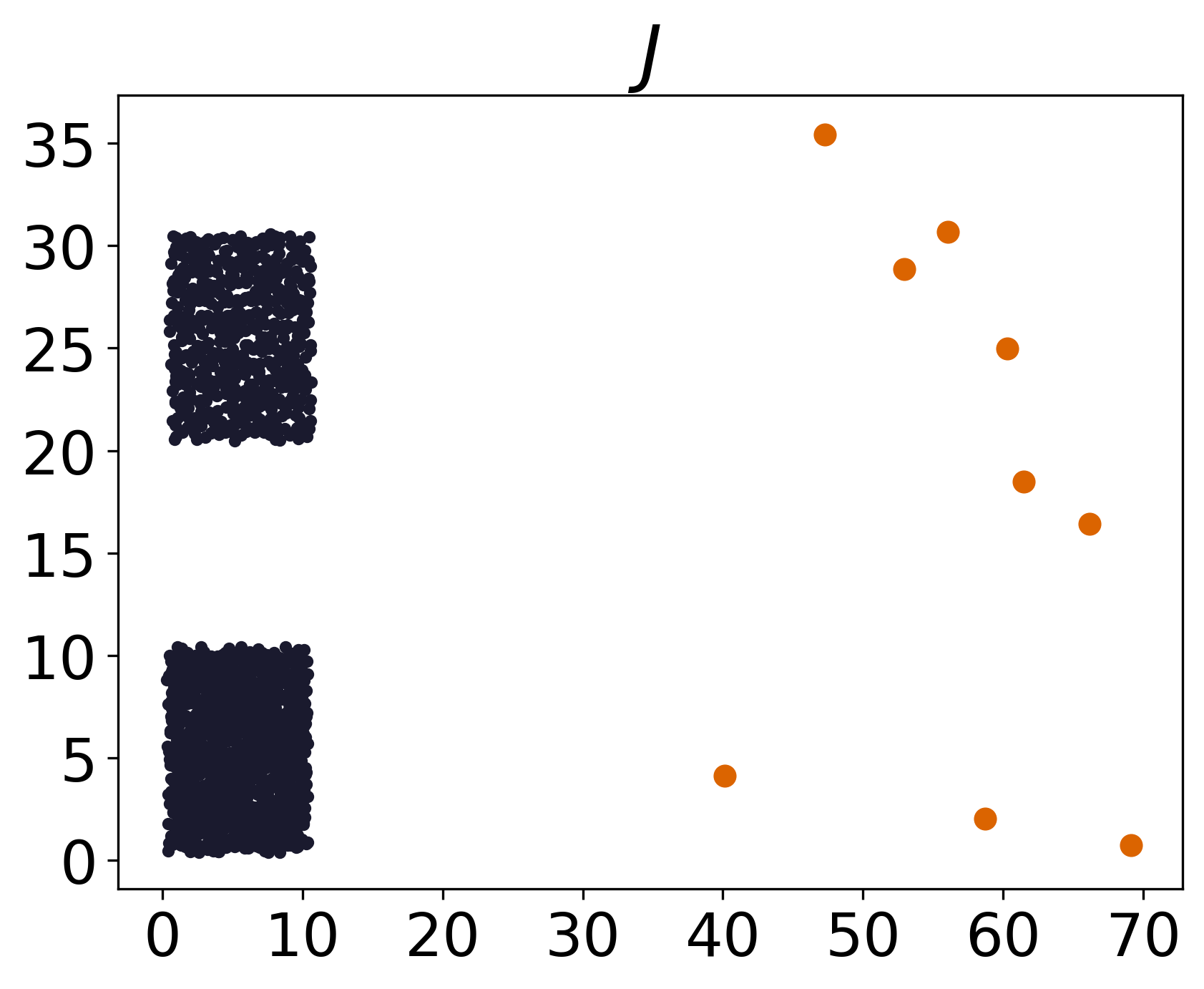}\\
  \includegraphics[width=1.2in]{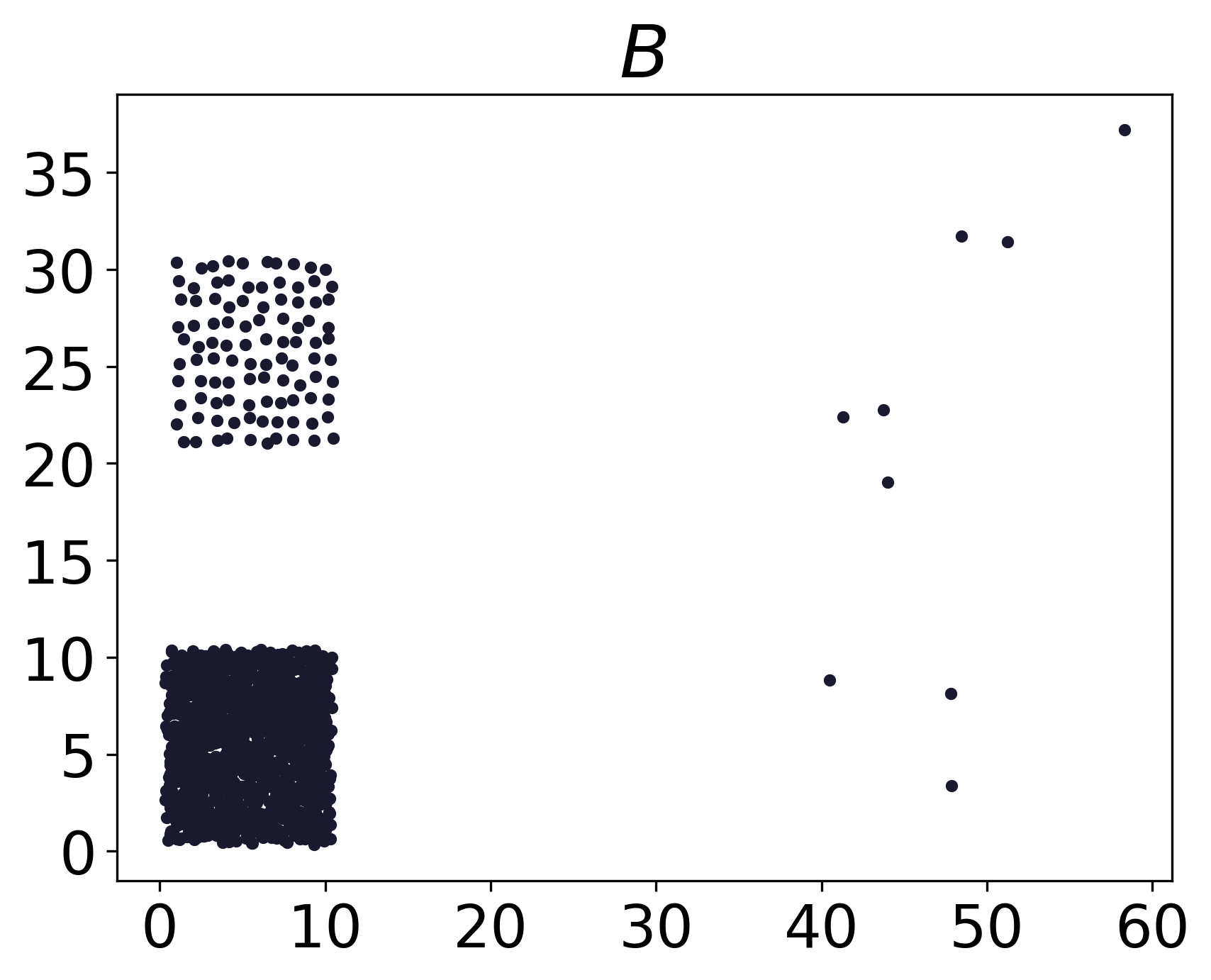}
  \quad\quad
  \includegraphics[width=1.2in]{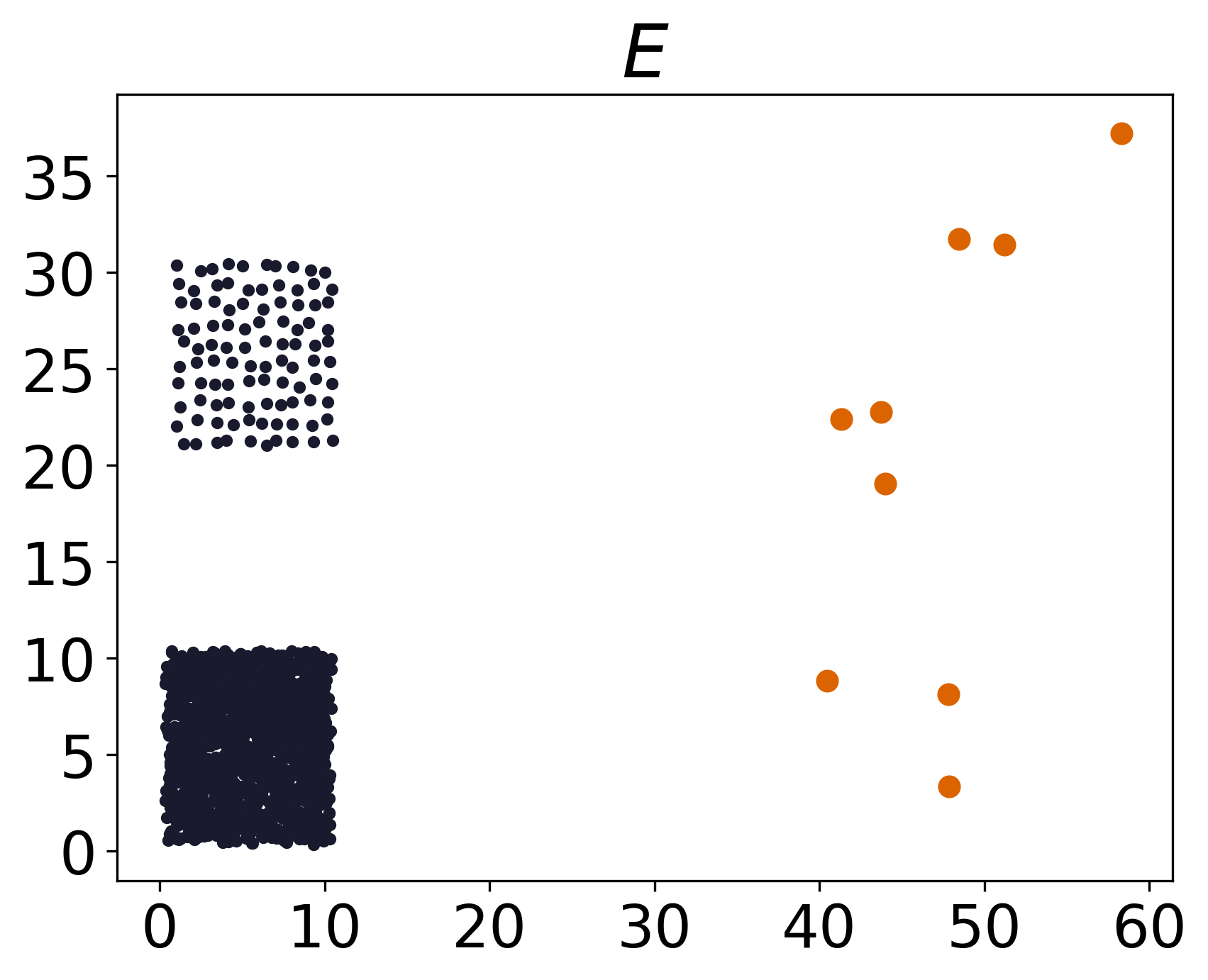}
  \quad\quad
  \includegraphics[width=1.2in]{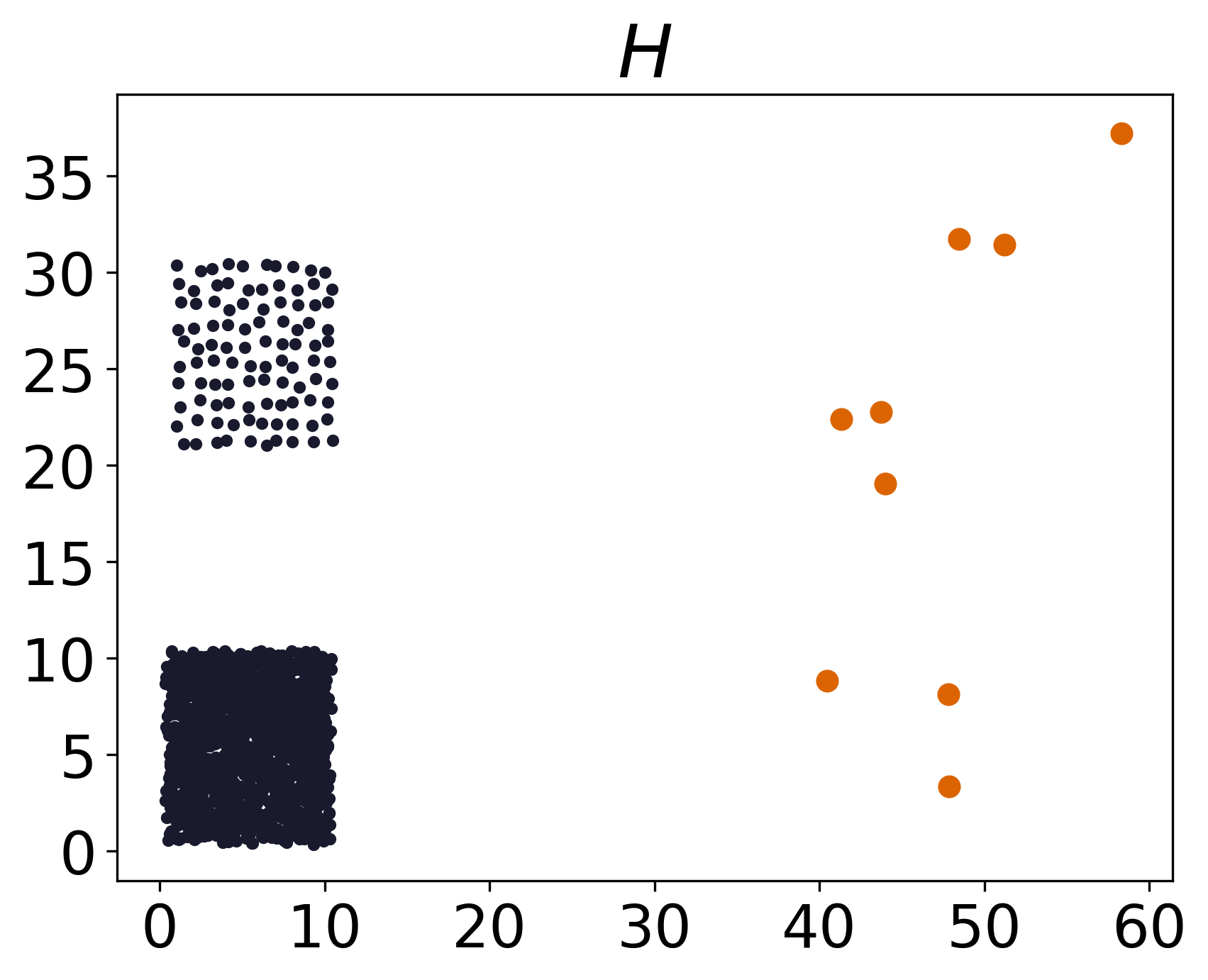}
  \quad\quad
  \includegraphics[width=1.2in]{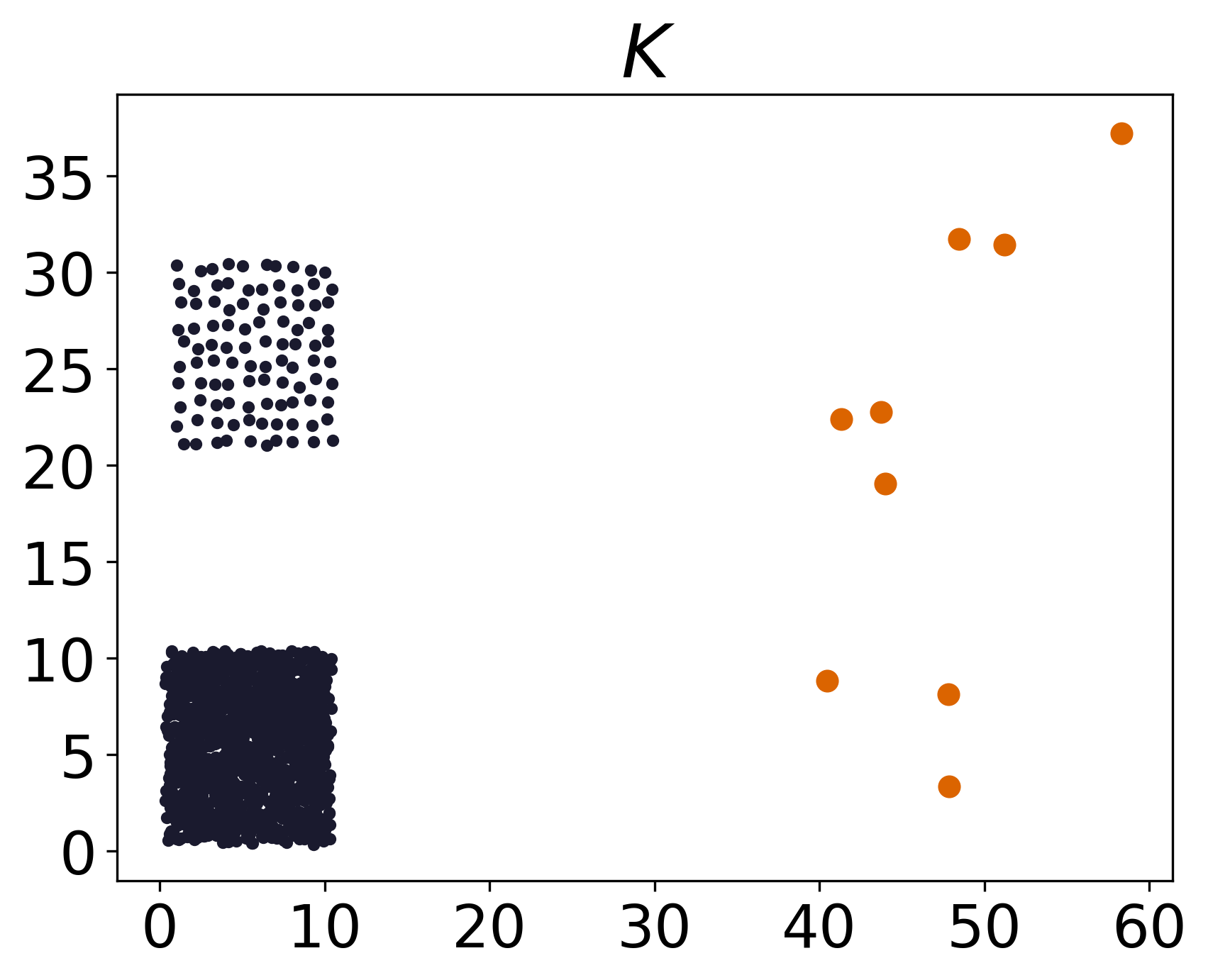}\\
  \includegraphics[width=1.2in]{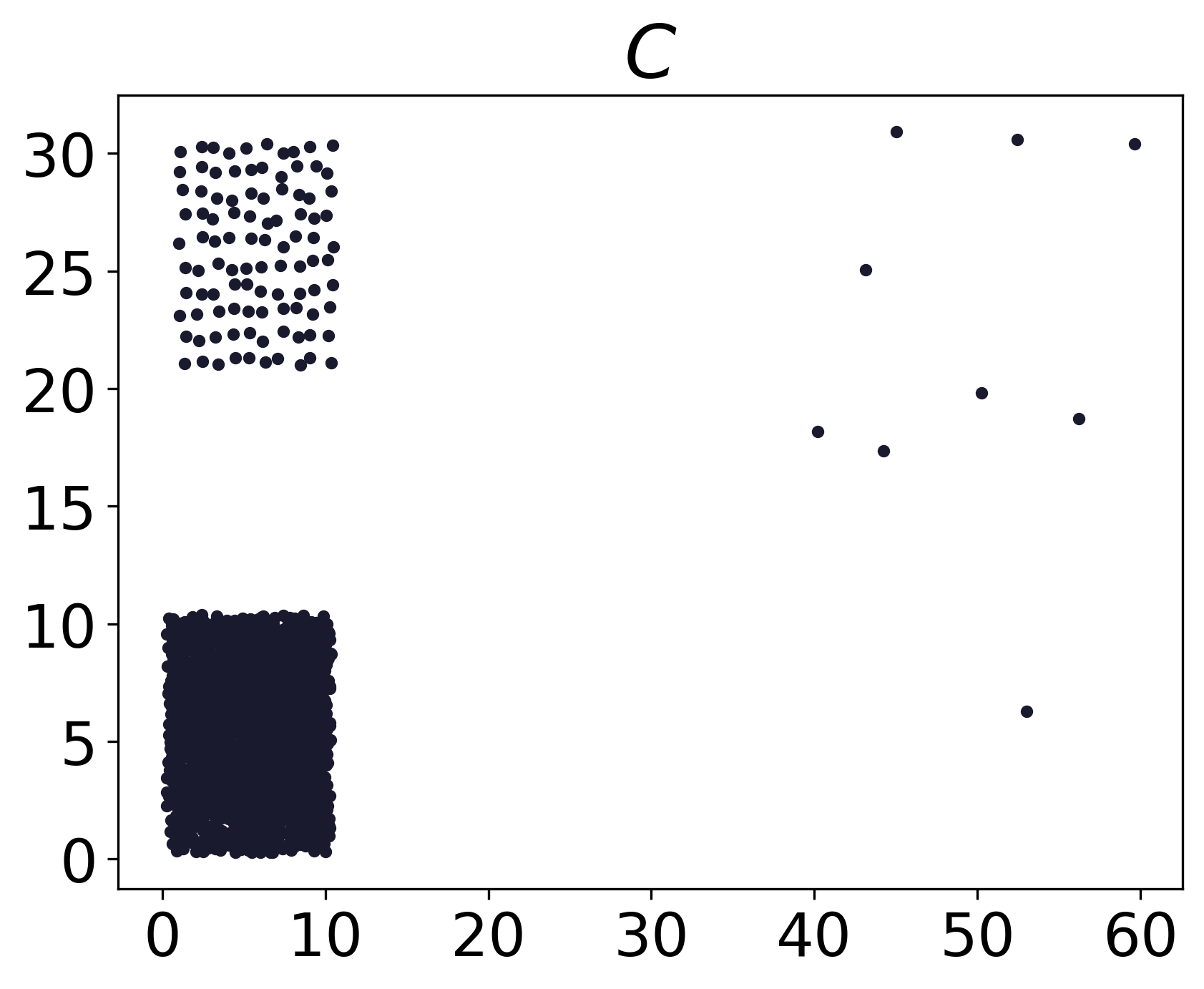}
  \quad\quad
  \includegraphics[width=1.2in]{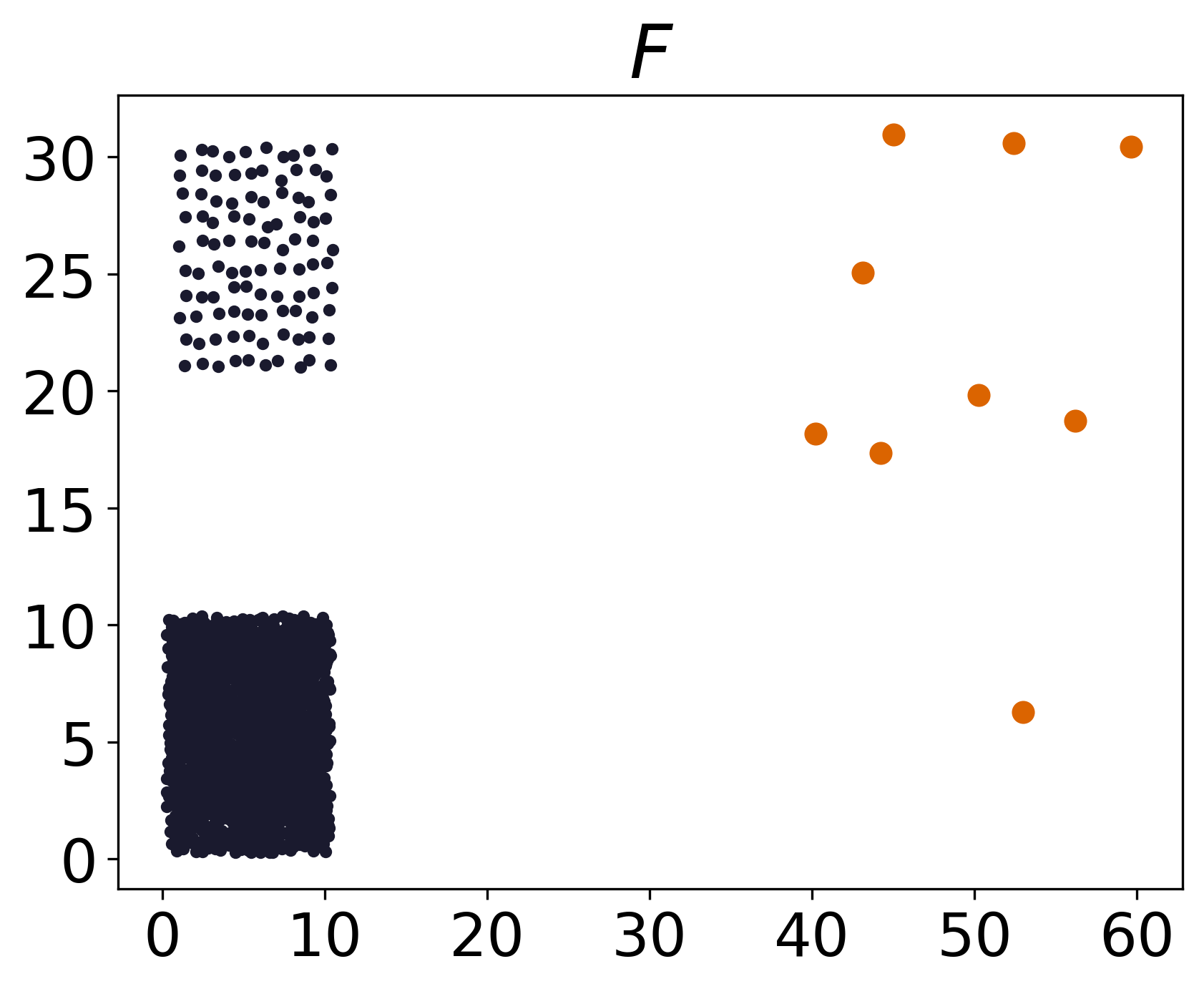}
  \quad\quad
  \includegraphics[width=1.2in]{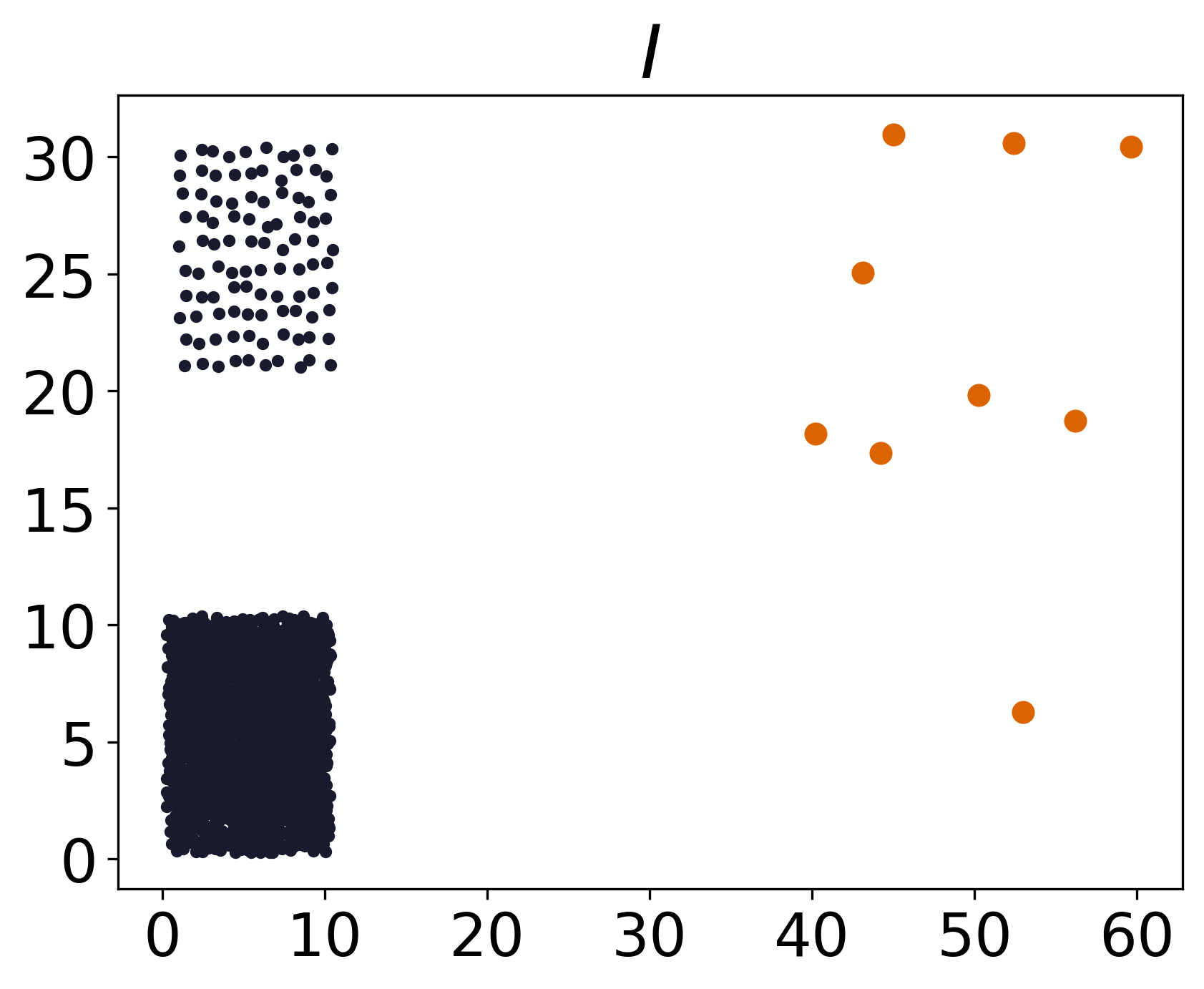}
  \quad\quad
  \includegraphics[width=1.2in]{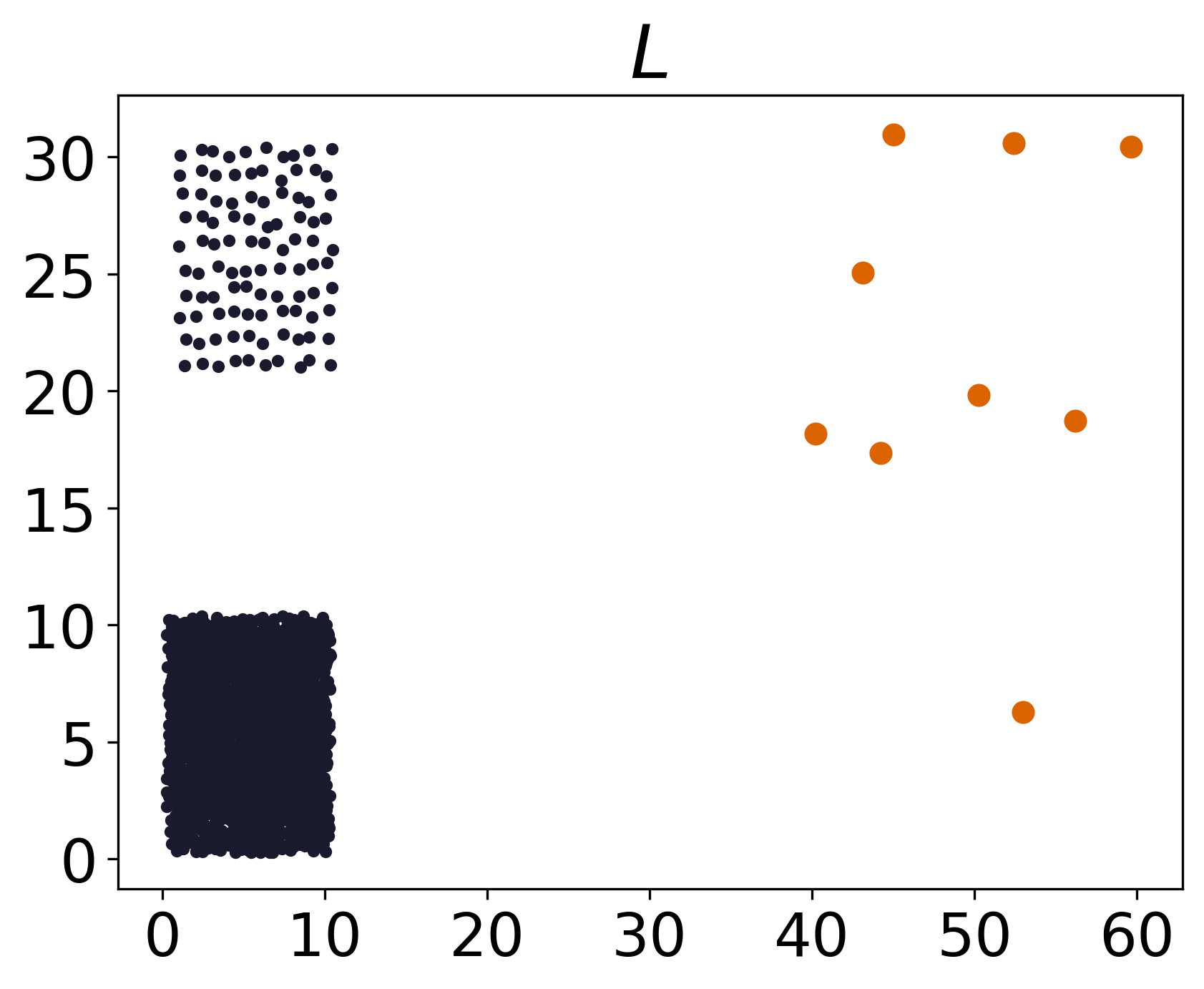}
  \caption{With the help of ODAR, the outliers detected by \emph{delta clustering} ($D,E,F$), \emph{kmeans} ($G,H,I$), and \emph{DPC} ($J,K,L$) for \emph{unbalanced-2} ($A$), \emph{unbalanced-10} ($B$), and \emph{unbalanced-15} ($C$) datasets. From \emph{unbalanced-2} to \emph{unbalanced-15}, the two clusters' density grows more unbalanced. Detection results are not disturbed by unbalanced density, and the objects in low-density clusters are not detected as outliers.}
  \label{fig:unbalance}
\end{figure*}

\subsubsection{Robustness to the Number of outliers}
ODAR assumes that the number of outliers is less than those of normal objects, and the feasibility of high-order density is based on this assumption. Here, we want to verify whether ODAR is effective when the number of outliers is close to those of normal objects. We use the \emph{worm}, a dataset production function  \cite{sieranoja2019fast}, to generate three datasets for verification. These datasets are named \emph{worm-num-least}, \emph{worm-num-medium}, \emph{worm-num-most}, and the differences between them are shown in the first column of Figure \ref{fig:num}. From \emph{worm-num-least} to \emph{worm-num-most}, the proportion of outliers increases from 8\% to 40\%. The number of outliers in \emph{worm-num-most} even exceeds the number of objects in \emph{worm-num-least}. Columns second to fourth of Figure \ref{fig:num} show the detection results, in which outliers are also marked in orange. Experimental results show that, with the help of ODAR, the three clustering algorithms correctly detect all outliers that are free from worm clusters. For the objects inside worm clusters, the three clustering algorithms correctly mark them as normal objects. Specifically, in the \emph{worm-num-most} dataset, as shown in Figure \ref{fig:num}(C), a large number of outliers make it difficult to identify the shape of worm clusters with the naked eye. After detecting outliers, all worm clusters in \emph{worm-num-most} become clearly visible, shown in Figure \ref{fig:num}(F, I, L).

\subsubsection{Robustness to Unbalanced Density}
Obviously, if densities are unbalanced among normal objects, then low-density normal objects may be similar to outliers. Since ODAR relies on density features to detect outliers, it is necessary to consider the robustness of ODAR to unbalanced density. Here, we test ODAR on three unbalanced datasets, namely \emph{unbalanced-2}, \emph{unbalanced-10}, and \emph{unbalanced-15}. The three datasets consist of a high-density cluster, a low-density cluster, and some outliers. In the \emph{unbalanced-2} dataset, the low-density cluster has 484 objects, the high-density cluster has 961 objects, so the density difference between the two clusters is about 2. In the \emph{unbalanced-10} dataset, the low-density cluster has 100 objects, the high-density cluster has 961 objects, so the density difference is about 10. In the \emph{unbalanced-15} dataset, the low-density cluster has 100 objects, the high-density cluster has 1440 objects, so the density difference is about 15. The three datasets are shown in the first column of Figure \ref{fig:unbalance}. Columns second to fourth of Figure \ref{fig:unbalance} show the results detected by different clustering algorithms with the help of ODAR, in which outliers are still marked in orange. Clearly, ODAR is not affected by unbalanced density, as the objects in low-density clusters are not detected as outliers.

\textbf{Summary of robustness.} Based on the above results, it can be concluded that ODAR is robust to the distribution and number of outliers, and unbalanced density. What's more, no matter which clustering algorithm is matched, these robustness are invariable.

\section{CONCLUSION AND FUTURE WORKS}
\label{sec:conlusion}
We propose a method called ODAR, designed to detect outliers for clustering algorithms. Through secondary density transformation, ODAR maps the dataset into a feature space. In this feature space, outliers and normal objects are gathered into two distinct clusters. As a result, any clustering algorithm can detect outliers by identifying clusters. We conduct a series of experiments to verify ODAR's robustness and confirm that it can be applied to complex datasets. To further verify the effectiveness, we test ODAR on ten real-world datasets. Experiments show that the clustering algorithms with the help of ODAR perform almost the best on all datasets, and have an overwhelming advantage over baseline methods. More importantly, ODAR is not sensitive to input parameter.

However, ODAR has a flaw. In addition to large-density objects, marginal objects among normal objects are also a special group. Their local densities are less than those of internal objects, and their number is far less than the number of internal objects. Therefore, the high-order densities of marginal objects are less than those of internal objects. Once the local densities of marginal objects are close to outliers, they will be falsely detected as outliers. In the future, we will try our best to reduce the difference in high-order density between internal objects and marginal objects to further improve ODAR.

\ifCLASSOPTIONcaptionsoff
  \newpage
\fi



\bibliographystyle{IEEEtran}

\bibliography{ODAR}

\end{document}